\documentclass{article}
\PassOptionsToPackage{numbers,compress}{natbib}
\usepackage[preprint]{neurips_2021}
\usepackage[utf8]{inputenc}
\usepackage[T1]{fontenc}
\usepackage{url}
\usepackage{booktabs}
\usepackage{amsfonts}
\usepackage{nicefrac}
\usepackage{microtype}

\usepackage{algorithm}
\usepackage{algorithmicx}
\usepackage[noend]{algpseudocode}
\usepackage{amsmath}
\usepackage{amssymb}
\usepackage{amsthm}
\usepackage{bbm}
\usepackage{bm}
\usepackage{color}
\usepackage{dirtytalk}
\usepackage{dsfont}
\usepackage{enumerate}
\usepackage{graphicx}
\usepackage{listings}
\usepackage{mathtools}
\usepackage{times}
\usepackage{wrapfig}
\usepackage{xspace}
\usepackage{comment}
\usepackage{enumitem}
\usepackage{caption}
\usepackage{subcaption}

\usepackage[usenames,dvipsnames]{xcolor}
\usepackage[bookmarks=false]{hyperref}
\hypersetup{
  pdftex,
  pdffitwindow=true,
  pdfstartview={FitH},
  pdfnewwindow=true,
  colorlinks,
  linktocpage=true,
  linkcolor=Green,
  urlcolor=Green,
  citecolor=Green
}
\usepackage[capitalize,noabbrev]{cleveref}

\usepackage{tikz}
\usetikzlibrary{calc,trees,positioning,arrows}
\usetikzlibrary{automata}
\usetikzlibrary{trees}
\usetikzlibrary{decorations.pathmorphing} % noisy shapes
\usetikzlibrary{fit}         % fitting shapes to coordinates
\usetikzlibrary{backgrounds} % drawing the background after the
\usetikzlibrary{shadows}
\usetikzlibrary{arrows.meta}
\tikzset{
  jump/.style={
     to path={
         let \p1=(\tikztostart),\p2=(\tikztotarget),\n1={atan2(\y2-\y1,\x2-\x1)} in
         (\tikztostart) -- ($($(\tikztostart)!#1!(\tikztotarget)$)!0.15cm!(\tikztostart)$)
         arc[start angle=\n1+180,end angle=\n1,radius=0.15cm] -- (\tikztotarget)}
  },
  jump/.default={0.5}
}
\tikzstyle{mynode}=[circle, thick, minimum size=0.7cm, draw=black!80]
\tikzstyle{observed}=[circle, thick, minimum size=0.7cm, draw=black!100, fill=black!20]
\tikzstyle{latent}=[circle, thick, minimum size=0.7cm, draw=black!80]
\tikzstyle{plate}=[rectangle, thick, inner sep=0.3cm, draw=black!100]
\tikzstyle{shadeplate}=[rectangle, thick, inner sep=0.4cm, draw=black!100]
\tikzstyle{table}=[circle,fill=blue!20,draw=black!100,inner sep=1pt, minimum size=30pt]
\tikzstyle{client}=[rectangle,fill=blue!20,draw=black!100,inner sep=1pt, minimum size=12pt]
\usepackage[backgroundcolor=white,disable]{todonotes}
\newcommand{\todosb}[2][]{\todo[color=Cyan!20,size=\tiny,#1]{SB: #2}} % Soumya's comments
\newcommand{\commentout}[1]{}
\newcommand{\junk}[1]{}

\Crefname{corollary}{Corollary}{Corollaries}
\Crefname{proposition}{Proposition}{Propositions}
\Crefname{theorem}{Theorem}{Theorems}
\Crefname{definition}{Definition}{Definitions}
\Crefname{assumption}{Assumption}{Assumptions}
\Crefname{example}{Example}{Examples}
\Crefname{remark}{Remark}{Remarks}
\Crefname{setting}{Setting}{Settings}
\Crefname{lemma}{Lemma}{Lemmas}

\usepackage{thmtools}
\declaretheorem[name=Theorem,refname={Theorem,Theorems},Refname={Theorem,Theorems}]{theorem}
\declaretheorem[name=Lemma,refname={Lemma,Lemmas},Refname={Lemma,Lemmas},sibling=theorem]{lemma}

\declaretheorem[name=Assumption,refname={Assumption,Assumptions},Refname={Assumption,Assumptions}]{assumption}
\declaretheorem[name=Proposition,refname={Proposition,Propositions},Refname={Proposition,Propositions},sibling=theorem]{proposition}

\newcommand{\cA}{\mathcal{A}}

\newcommand{\cE}{\mathcal{E}}

\newcommand{\cL}{\mathcal{L}}

\newcommand{\cN}{\mathcal{N}}

\newcommand{\realset}{\mathbb{R}}
\newcommand{\R}{\mathbb{R}}

\newcommand{\diag}[1]{\mathrm{diag}\left(#1\right)}

\newcommand{\E}[1]{\mathbb{E} \left[#1\right]}
\newcommand{\condE}[2]{\mathbb{E} \left[#1 \,\middle|\, #2\right]}

\newcommand{\prob}[1]{\mathbb{P} \left(#1\right)}
\newcommand{\condprob}[2]{\mathbb{P} \left(#1 \,\middle|\, #2\right)}

\newcommand{\abs}[1]{\left|#1\right|}

\newcommand*\dif{\mathop{}\!\mathrm{d}}

\newcommand{\I}[1]{\mathds{1} \! \left\{#1\right\}}

\newcommand{\set}[1]{\left\{#1\right\}}

\newcommand{\T}{^\top}

\DeclareMathOperator*{\argmax}{arg\,max\,}

\let\det\relax
\DeclareMathOperator{\det}{det}

\mathchardef\mhyphen="2D

\newcommand{\adats}{\ensuremath{\tt AdaTS}\xspace}
\newcommand{\adatsminus}{\ensuremath{\tt AdaTS^-}\xspace}
\newcommand{\adatsplus}{\ensuremath{\tt AdaTS^+}\xspace}

\newcommand{\metats}{\ensuremath{\tt MetaTS}\xspace}
\newcommand{\oraclets}{\ensuremath{\tt OracleTS}\xspace}

\newcommand{\ts}{\ensuremath{\tt TS}\xspace}

\title{No Regrets for Learning the Prior in Bandits}

\author{
  Soumya Basu \\
  Google \\
  \And
  Branislav Kveton \\
  Google Research \\
  \And
  Manzil Zaheer \\
  Google Research \\
  \And
  Csaba Szepesv\'{a}ri \\
  DeepMind / University of Alberta \\
}

\begin{document}

\maketitle

\begin{abstract}
We propose \adats, a Thompson sampling algorithm that adapts sequentially to bandit tasks that it interacts with. The key idea in \adats is to adapt to an unknown task prior distribution by maintaining a distribution over its parameters. When solving a bandit task, that uncertainty is marginalized out and properly accounted for. \adats is a fully-Bayesian algorithm that can be implemented efficiently in several classes of bandit problems. We derive upper bounds on its Bayes regret that quantify the loss due to not knowing the task prior, and show that it is small. Our theory is supported by experiments, where \adats outperforms prior algorithms and works well even in challenging real-world problems.
\end{abstract}

\section{Introduction}
\label{sec:introduction}

We study the problem of maximizing the total reward, or minimizing the total regret, in a sequence of stochastic bandit instances \citep{lai85asymptotically,auer02finitetime,lattimore19bandit}. We consider a Bayesian version of the problem, where the bandit instances are drawn from some distribution. More specifically, the learning agent interacts with $m$ bandit instances in $m$ tasks, with one instance per task. The interaction with each task is for $n$ rounds and with $K$ arms. The reward distribution of arm $i \in [K]$ in task $s \in [m]$ is $p_i(\cdot; \theta_{s, *})$, where $\theta_{s, *}$ is a shared parameter of all arms in task $s$. When arm $i$ is pulled in task $s$, the agent receives a random reward from $p_i(\cdot; \theta_{s, *})$. The parameters $\theta_{1, *}, \dots, \theta_{m, *}$ are drawn independently of each other from a \emph{task prior} $P(\cdot; \mu_*)$. The task prior is parameterized by an unknown \emph{meta-parameter} $\mu_*$, which is drawn from a \emph{meta-prior} $Q$. The agent does not know $\mu_*$ or $\theta_{1, *}, \dots, \theta_{m, *}$. However, it knows $Q$ and the parametric forms of all distributions, which help it to learn about $\mu_*$. This is a form of \emph{meta-learning} \citep{thrun96explanationbased,thrun98lifelong,baxter98theoretical,baxter00model}, where the agent learns to act from interactions with bandit instances.

A simple approach is to ignore the hierarchical structure of the problem and solve each bandit task independently with some bandit algorithm, such as \emph{Thompson sampling (TS)} \citep{thompson33likelihood,chapelle11empirical,agrawal12analysis,russo18tutorial}. This may be highly suboptimal. To illustrate this, imagine that arm $1$ is optimal for any $\theta$ in the support of $P(\cdot; \mu_*)$. If $\mu_*$ was known, any reasonable algorithm would only pull arm $1$ and have zero regret over any horizon. Likewise, a clever algorithm that learns $\mu_*$ should eventually pull arm $1$ most of the time, and thus have diminishing regret as it interacts with a growing number of tasks. Two challenges arise when designing the clever algorithm. First, can it be computationally efficient? Second, what is the regret due to adapting to $\mu_*$?

We make the following contributions. First, we propose a Thompson sampling algorithm for our problem, which we call \adats. \adats maintains a distribution over the meta-parameter $\mu_*$, which concentrates over time and is marginalized out when interacting with individual bandit instances. Second, we propose computationally-efficient implementations of \adats for multi-armed bandits \citep{lai85asymptotically,auer02finitetime}, linear bandits \citep{dani08stochastic,abbasi-yadkori11improved}, and combinatorial semi-bandits \citep{gai12combinatorial,chen14combinatorial,kveton15tight}. These implementations are for specific reward distributions and conjugate task priors. Third, we bound the $n$-round Bayes regret of \adats in linear bandits and semi-bandits, and multi-armed bandits as a special case. The Bayes regret is defined by taking an expectation over all random quantities, including $\mu_* \sim Q$. Our bounds show that not knowing $\mu_*$ has a minimal impact on the regret as the number of tasks grows, of only $\tilde{O}(\sqrt{m n})$. This is in a sharp contrast to prior work \citep{kveton21metathompson}, where this is $\tilde{O}(\sqrt{m} n^2)$. Finally, our experiments show that \adats quickly adapts to the unknown meta-parameter $\mu_*$, is robust to meta-prior misspecification, and performs well even in challenging classification problems.

We present a general framework for learning to explore from similar past exploration problems. One potential application is cold-start personalization in recommender systems where users are tasks. The users have similar preferences, but neither the individual preferences nor their similarity is known in advance. Another application could be online regression with bandit feedback (\cref{sec:classification experiments}) where individual regression problems are tasks. Similar examples in the tasks have similar mean responses, which are unknown in advance.

\section{Setting}
\label{sec:setting}

We first introduce our notation. The set $\set{1, \dots, n}$ is denoted by $[n]$. The indicator $\I{E}$ denotes that event $E$ occurs. The $i$-th entry of vector $v$ is $v_i$. If the vector or its index are already subindexed, we write $v(i)$. We use $\tilde{O}$ for the big-O notation up to polylogarithmic factors. A diagonal matrix with entries $v$ is denoted $\diag{v}$. We use the terms \say{arm} and \say{action} interchangeably, depending on the context.

\begin{wrapfigure}{r}{2in}
  \centering
  \vspace{-0.2in}
  \begin{tikzpicture}[>=latex,text height=1.5ex,text depth=0.25ex]
  \matrix[row sep=0.9cm,column sep=0.1cm] {
    \node (muq) {$\mu_q$}; &
    \node (sigmaq) {$\Sigma_q$}; &
    \node (sigma0) {$\Sigma_0$}; &&&&&&&
    \node (sigma) {$\sigma^2$};
    \\
    \node (mu) [latent]{$\mu_{*}$}; &&
    \node (theta) [latent]{$\theta_{s,*}$}; &&&&&&&
    \node (y) [observed]{$Y_{s,t}$};
    \\
  };
  \path[->]
  (muq) edge[thick] (mu)
  (sigmaq) edge[thick] (mu)
  (sigma0) edge[thick] (theta)
  (sigma) edge[thick] (y)
  (mu) edge[thick] (theta)
  (theta) edge[thick] (y)
  ;
    \begin{pgfonlayer}{background}
      \node (obs) [plate, fit=(y)] {};
      \node [above = 2mm, left] at (obs.south east) {\tiny $t=1,...,n$};
      \node (tasks) [plate, fit=(theta) (y) (obs)] {\ \\[16mm]\tiny $s=1,...,m$};
    \end{pgfonlayer}
  \end{tikzpicture}
  \caption{Graphical model of our environment.}
  \vspace{-0.05in}
  \label{fig:environment}
\end{wrapfigure}
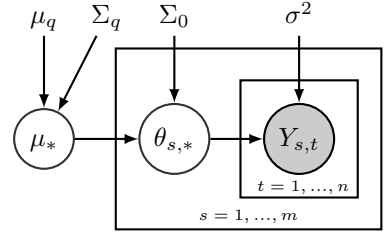

Our setting was proposed in \citet{kveton21metathompson} and is defined as follows. Each bandit \emph{problem instance} has $K$ arms. Each arm $i \in [K]$ is defined by distribution $p_i(\cdot; \theta)$ with parameter $\theta \in \Theta$. The parameter $\theta$ is shared among all arms. The mean of $p_i(\cdot; \theta)$ is denoted by $r(i; \theta)$. The learning agent interacts with $m$ instances, one at each of $m$ \emph{tasks}. At the beginning of task $s \in [m]$, an instance $\theta_{s, *} \in \Theta$ is sampled i.i.d.\ from a \emph{task prior} $P(\cdot; \mu_*)$, which is parameterized by $\mu_*$. The agent interacts with $\theta_{s, *}$ for $n$ rounds. In round $t \in [n]$, it pulls one arm and observes a stochastic realization of its reward. We denote the pulled arm in round $t$ of task $s$ by $A_{s, t} \in [K]$, the realized rewards of all arms in round $t$ of task $s$ by $Y_{s, t} \in \realset^K$, and the reward of arm $i \in [K]$ by $Y_{s, t}(i) \sim p_i(\cdot; \theta_{s, *})$. We assume that the realized rewards $Y_{s, t}$ are i.i.d.\ with respect to both $s$ and $t$. A graphical model of our environment is drawn in \cref{fig:environment}. We define the distribution-specific parameters $\mu_q$, $\Sigma_q$, $\Sigma_0$, and $\sigma^2$ when we instantiate our framework. Our terminology is summarized in \cref{sec:appendix}.

The \emph{$n$-round regret} of an agent or algorithm over $m$ tasks with task prior $P(\cdot; \mu_*)$ is defined as
\begin{align}
  R(m, n; \mu_*)
  = \sum_{s = 1}^m \condE{\sum_{t = 1}^n r(A_{s, *}; \theta_{s, *}) -
  r(A_{s, t}; \theta_{s, *})}{\mu_*}\,,
  \label{eq:regret}
\end{align}
where $A_{s,*} = \argmax_{i \in [K]} r(i; \theta_{s, *})$ is the \emph{optimal arm} in the random problem instance $\theta_{s, *}$ in task $s \in [m]$. The above expectation is over problem instances $\theta_{s, *} \sim P(\cdot; \mu_*)$, their realized rewards, and also pulled arms. Note that $\mu_*$ is fixed. \citet{russo14learning} showed that the Bayes regret, which matches the definition in \eqref{eq:regret} in any task, of Thompson sampling in a $K$-armed bandit with $n$ rounds is $\tilde{O}(\sqrt{K n})$. So, when TS is applied independently in each task, $R(m, n; \mu_*) = \tilde{O}(m \sqrt{K n})$.

Our goal is to attain a comparable regret without knowing $\mu_*$. We frame this problem in a Bayesian fashion, where $\mu_* \sim Q$ before the learning agent interacts with the first task. The agent knows $Q$ and we call it a \emph{meta-prior}. Accordingly, we consider $R(m,n) = \E{R(m, n; \mu_*)}$ as a metric and call it the \emph{Bayes regret}. Our approach is motivated by hierarchical Bayesian models \citep{gelman13bayesian}, where the uncertainty in prior parameters, such as $\mu_*$, is represented by another distribution, such as $Q$. In these models, $Q$ is called a \emph{hyper-prior} and $\mu_*$ is called a \emph{hyper-parameter}. We attempt to learn $\mu_*$ from sequential interactions with instances $\theta_{s, *} \sim P(\cdot; \mu_*)$, which are also unknown. The agent can only observe their noisy realizations $Y_{s, t}$.

\section{Algorithm}
\label{sec:algorithm}

Our algorithm is presented in this section. To describe it, we need to introduce several notions of \emph{history}, the past interactions of the agent. We denote by $H_s = (A_{s, t}, Y_{s, t}(A_{s, t}))_{t = 1}^n$ the history in task $s$ and by $H_{1 : s} = H_1 \oplus \dots \oplus H_s$ a concatenated vector of all histories in the first $s$ tasks. The history up to round $t$ in task $s$ is $H_{s, t} = (A_{s, \ell}, Y_{s, \ell}(A_{s, \ell}))_{\ell = 1}^{t - 1}$ and all history up to round $t$ in task $s$ is $H_{1 : s, t} = H_{1 : s - 1} \oplus H_{s, t}$. We denote the conditional probability distribution given history $H_{1 : s, t}$ by $\mathbb{P}_{s, t}(\cdot) = \condprob{\cdot}{H_{1 : s, t}}$ and the corresponding conditional expectation by $\mathbb{E}_{s, t}[\cdot] = \condE{\cdot}{H_{1 : s, t}}$.

Our algorithm is a form of Thompson sampling \cite{thompson33likelihood,chapelle11empirical,agrawal12analysis,russo18tutorial}. TS pulls arms proportionally to being optimal with respect to the posterior. In particular, let $\cL_{s, t}(\theta) = \prod_{\ell = 1}^{t - 1} p_{A_{s, \ell}}(Y_{s, \ell}(A_{s, \ell}); \theta)$ be the \emph{likelihood of observations} in task $s$ up to round $t$. If the prior $P(\cdot; \mu_*)$ was known, the posterior of instance $\theta$ in round $t$ would be $P_{s, t}^\textsc{ts}(\theta) \propto \cL_{s, t}(\theta) \, P(\theta; \mu_*)$. TS would sample $\tilde{\theta}_t \sim P_{s, t}^\textsc{ts}$ and pull arm $A_t = \argmax_{i \in [K]} r(i; \tilde{\theta}_t)$.

We address the case of unknown $\mu_*$. The key idea in our method is to maintain a posterior density of $\mu_*$, which we call a \emph{meta-posterior}. This density represents uncertainty in $\mu_*$ given history. In task $s$, we denote it by $Q_s$ and define it such that $\condprob{\mu_* \in B}{H_{1 : s - 1}} = \int_{\mu \in B} Q_s(\mu) \, d \kappa_1(\mu)$ holds for any set $B$, where $\kappa_1$ is the reference measure for $\mu$. We use this more general notation, as opposing to $d \mu$, because $\mu$ can be both continuous and discrete. When solving task $s$, $Q_s$ is used to compute an \emph{uncertainty-adjusted task prior} $P_s$, which is a posterior density of $\theta_{s, *}$ given history. Formally, $P_s$ is a density such that $\condprob{\theta_{s, *} \in B}{H_{1 : s - 1}} = \int_{\theta \in B} P_s(\theta) \, d \kappa_2(\theta)$ holds for any set $B$, where $\kappa_2$ is the reference measure for $\theta$. After computing $P_s$, we run TS with prior $P_s$ to solve task $s$. To maintain $Q_s$ and $P_s$, we find it useful expressing them using a recursive update rule below.

\begin{restatable}[]{proposition}{posterior}
\label{prop:posterior} Let $\cL_s(\theta) = \prod_{\ell = 1}^n p_{A_{s, \ell}}(Y_{s, \ell}(A_{s, \ell}); \theta)$ be the likelihood of observations in task $s$. Then for any task $s \in [m]$,
\begin{align*}
  P_s(\theta)
  = \int_\mu P(\theta; \mu) \, Q_s(\mu) \, d \kappa_1(\mu)\,, \quad
  Q_s(\mu) 
  = \int_\theta \cL_{s - 1}(\theta) \, P(\theta; \mu) \, d \kappa_2(\theta) \, Q_{s - 1}(\mu)\,.
\end{align*}
\end{restatable}

The claim is proved in \cref{sec:appendix}. The proof uses the Bayes rule, where we carefully account for the fact that the observations are collected adaptively, the pulled arm in round $t$ of task $s$ depends on history $H_{1 : s, t}$. The pseudocode of our algorithm is in \cref{alg:ada-ts}. Since the algorithm adapts to the unknown task prior $P(\cdot; \mu_*)$, we call it \adats. \adats can be implemented efficiently when $P_s$ is a conjugate prior for rewards, or a mixture of conjugate priors. We discuss several exact and efficient implementations starting from \cref{sec:gaussian bandit}.

\begin{algorithm}[t]
  \caption{\adats: Instance-adaptive exploration in Thompson sampling.}
  \label{alg:ada-ts}
  \begin{algorithmic}[1]
    \State Initialize meta-prior $Q_0 \gets Q$
    \For{$s = 1, \dots, m$}
      \State Compute meta-posterior $Q_s$ (\cref{prop:posterior})
      \State Compute uncertainty-adjusted task prior $P_s$ (\cref{prop:posterior})
      \For{$t = 1, \dots, n$}
        \State Compute posterior of $\theta$ in task $s$, $P_{s, t}(\theta) \propto \cL_{s, t}(\theta) P_s(\theta)$
        \State Sample $\tilde{\theta}_{s, t} \sim P_{s, t}$, pull arm $A_{s, t} \gets \argmax_{i \in [K]} r(i; \tilde{\theta}_{s, t})$, and observe $Y_{s, t}(A_{s, t})$
      \EndFor
    \EndFor
  \end{algorithmic}
\end{algorithm}

The design of \adats is motivated by \metats \cite{kveton21metathompson}, which also maintains a meta-posterior $Q_s$. The difference is that \metats samples $\tilde{\mu}_s \sim Q_s$ in task $s$ to be optimistic with respect to the unknown $\mu_*$. Then it runs TS with prior $P(\cdot; \tilde{\mu}_s)$. While simple and intuitive, the sampling of $\tilde{\mu}_s$ induces a high variance and leads to a conservative worst-case analysis. We improve \metats by avoiding the sampling step. This leads to tighter and more general regret bounds (\cref{sec:regret bounds}), beyond multi-armed bandits; while the practical performance also improves significantly (\cref{sec:experiments}).

\subsection{Gaussian Bandit}
\label{sec:gaussian bandit}

We start with a $K$-armed Gaussian bandit with mean arm rewards $\theta \in \realset^K$. The reward distribution of arm $i$ is $p_i(\cdot; \theta) = \cN(\cdot; \theta_i, \sigma^2)$, where $\sigma > 0$ is reward noise and $\theta_i$ is the mean reward of arm $i$. A natural conjugate prior for this problem class is $P(\cdot; \mu) = \cN(\cdot; \mu, \Sigma_0)$, where $\Sigma_0 = \diag{(\sigma_{0, i}^2)_{i = 1}^K}$ is known and we learn $\mu \in \realset^K$.

Because the prior is a multivariate Gaussian, \adats can be implemented efficiently with a Gaussian meta-prior $Q(\cdot) = \cN(\cdot; \mu_q, \Sigma_q)$, where $\mu_q = (\mu_{q, i})_{i = 1}^K$ and $\Sigma_q = \diag{(\sigma_{q, i}^2)_{i = 1}^K}$ are known mean parameter vector and covariance matrix, respectively. In this case, the meta-posterior in task $s$ is also a Gaussian $Q_s(\cdot) = \cN(\cdot; \hat{\mu}_s, \hat{\Sigma}_s)$, where $\hat{\mu}_s = (\hat{\mu}_{s, i})_{i = 1}^K$ and $\hat{\Sigma}_s = \diag{(\hat{\sigma}_{s, i}^2)_{i = 1}^K}$ are defined as
\begin{align}
  \hat{\mu}_{s, i}
  = \hat{\sigma}_{s, i}^2 \left(\frac{\mu_{q, i}}{\sigma_{q, i}^2} +
  \sum_{\ell = 1}^{s - 1} \frac{T_{\ell, i}}{T_{\ell, i} \, \sigma_{0, i}^2 + \sigma^2}
  \frac{B_{\ell, i}}{T_{\ell, i}}\right)\,, \quad
  \hat{\sigma}_{s, i}^{-2}
  = \sigma_{q, i}^{-2} +
  \sum_{\ell = 1}^{s - 1} \frac{T_{\ell, i}}{T_{\ell, i} \, \sigma_{0, i}^2 + \sigma^2}\,.
  \label{eq:mab meta-posterior}
\end{align}
Here $T_{\ell, i} = \sum_{t = 1}^n \I{A_{\ell, t} = i}$ is the number of pulls of arm $i$ in task $\ell$ and the total reward from these pulls is $B_{\ell, i} = \sum_{t = 1}^n \I{A_{\ell, t} = i} Y_{\ell, t}(i)$. The above formula has a very nice interpretation. The posterior mean $\hat{\mu}_{s, i}$ of the meta-parameter of arm $i$ is a weighted sum of the noisy estimates of the means of arm $i$ from the past tasks $B_{\ell, i} / T_{\ell, i}$ and the prior. In this sum, each bandit task is essentially a single observation. The weights are proportional to the number of pulls in a task, giving the task with more pulls a higher weight. They vary from $(\sigma_{0, i}^2 + \sigma^2)^{-1}$, when the arm is pulled only once, up to $\sigma_{0, i}^{-2}$. This is the minimum amount of uncertainty that cannot be reduced by more pulls.

The update in \eqref{eq:mab meta-posterior} is by \cref{lem:multi-task bayesian regression} in \cref{sec:appendix}, which we borrow from \citet{kveton21metathompson}. From \cref{prop:posterior}, we have that the uncertainty-adjusted prior for task $s$ is $P_s(\cdot) = \cN(\cdot; \hat{\mu}_s, \hat{\Sigma}_s + \Sigma_0)$.

\subsection{Linear Bandit with Gaussian Rewards}
\label{sec:linear bandit}

Now we generalize \cref{sec:gaussian bandit} and consider a \emph{linear bandit} \cite{dani08stochastic,abbasi-yadkori11improved} with $K$ arms and $d$ dimensions. Let $\cA \subset \realset^d$ be an \emph{action set} such that $\abs{\cA} = K$. We refer to each $a \in \cA$ as an \emph{arm}. Then, with a slight abuse of notation from \cref{sec:setting}, the reward distribution of arm $a$ is $p_a(\cdot; \theta) = \cN(\cdot; a\T \theta, \sigma^2)$, where $\theta \in \realset^d$ is shared by all arms and $\sigma > 0$ is reward noise. A conjugate prior for this problem class is $P(\cdot; \mu) = \cN(\cdot; \mu, \Sigma_0)$, where $\Sigma_0 \in \realset^{d \times d}$ is known and we learn $\mu \in \realset^d$.

As in \cref{sec:gaussian bandit}, \adats can be implemented efficiently with a meta-prior $Q(\cdot) = \cN(\cdot; \mu_q, \Sigma_q)$, where $\mu_q \in \realset^d$ is a known mean parameter vector and $\Sigma_q \in \realset^{d \times d}$ is a known covariance matrix. In this case, $Q_s(\cdot) = \cN(\cdot; \hat{\mu}_s, \hat{\Sigma}_s)$, where
\begin{align*}
  \hat{\mu}_s
  & = \hat{\Sigma}_s \left(\Sigma_q^{-1} \mu_q +
  \sum_{\ell = 1}^{s - 1} \frac{B_\ell}{\sigma^2} - \frac{G_\ell}{\sigma^2}
  \left(\Sigma_0^{-1} + \frac{G_\ell}{\sigma^2}\right)^{-1} \frac{B_\ell}{\sigma^2}\right)\,, \\
  \hat{\Sigma}_s^{-1}
  & = \Sigma_q^{-1} +
  \sum_{\ell = 1}^{s - 1} \frac{G_\ell}{\sigma^2} - \frac{G_\ell}{\sigma^2}
  \left(\Sigma_0^{-1} + \frac{G_\ell}{\sigma^2}\right)^{-1} \frac{G_\ell}{\sigma^2}\,.
\end{align*}
Here $G_\ell = \sum_{t = 1}^n A_{\ell, t} A_{\ell, t}\T$ is the outer product of the feature vectors of the pulled arms in task $\ell$ and $B_\ell = \sum_{t = 1}^n A_{\ell, t} Y_{\ell, t}(A_{\ell, t})$ is their sum weighted by their rewards. The above update follows from \cref{lem:multi-task bayesian regression} in \cref{sec:appendix}, which is due to \citet{kveton21metathompson}. From \cref{prop:posterior}, the uncertainty-adjusted prior for task $s$ is $P_s(\cdot) = \cN(\cdot; \hat{\mu}_s, \hat{\Sigma}_s + \Sigma_0)$.

We note in passing that when $K = d$ and $\cA$ is the standard Euclidean basis of $\realset^d$, the linear bandit reduces to a $K$-armed bandit. Since the covariance matrices are unrestricted here, the formulation in this section also shows how to generalize \cref{sec:gaussian bandit} to arbitrary covariance matrices.

\subsection{Semi-Bandit with Gaussian Rewards}
\label{sec:semi-bandit}

A \emph{stochastic combinatorial semi-bandit} \cite{gai12combinatorial,chen13combinatorial,kveton14matroid,kveton15tight,wen15efficient}, or \emph{semi-bandit} for short, is a $K$-armed bandit where at most $L \leq K$ arms are pulled in each round. After the arms are pulled, the agent observes their individual rewards and its reward is the sum of the individual rewards. Semi-bandits can be used to solve online combinatorial problems, such as learning to route.

We consider a Gaussian reward distribution for each arm, as in \cref{sec:gaussian bandit}. The difference in the semi-bandit formulation is that the \emph{action set} is $\cA \subseteq \Pi_L(K)$, where $\Pi_L(K)$ is the set of all subsets of $[K]$ of size at most $L$. In round $t$ of task $s$, the agents pulls arms $A_{s, t} \in \cA$. The meta-posterior is updated analogously to \cref{sec:gaussian bandit}. The only difference is that $\I{A_{\ell, t} = i}$ becomes $\I{i \in A_{\ell, t}}$.

\subsection{Exponential-Family Bandit with Mixture Priors}
\label{sec:exponential-family bandit}

We consider a general $K$-armed bandit with mean arm rewards $\theta \in \realset^K$. The reward distribution of arm $i$ is any one-dimensional exponential-family distribution parameterized by $\theta_i$. In a Bernoulli bandit, this would be $p_i(\cdot; \theta) = \mathrm{Ber}(\cdot; \theta_i)$. A natural prior for this reward model would be a product of per-arm conjugate priors, such as the product of betas for Bernoulli rewards.

It is challenging to generalize our approach beyond Gaussian models because we require more than the standard notion of conjugacy. Specifically, to apply \adats to an exponentially-family prior, such as the product of betas, we need a computationally tractable prior for that prior. In this case, it does not exist. We circumvent this issue by discretization. More specifically, let $\set{P(\cdot; j)}_{j = 1}^L$ be a set of $L$ potential conjugate priors, where each $P(\cdot; j)$ is a product of one-dimensional exponential-family priors. Then a suitable meta-prior is a vector of initial beliefs into each potential prior. In particular, it is $Q(\cdot) = \mathrm{Cat}(\cdot; w_q)$, where $w_q \in \Delta_{L - 1}$ is the belief and $\Delta_L$ is the $L$-simplex.

In this case, $Q_s(j) = \int_\theta \cL_{s - 1}(\theta) P(\theta; j) \, d \kappa_2(\theta) \, Q_{s - 1}(j)$ in \cref{prop:posterior} has a closed form, since it is a standard conjugate posterior update for a distribution over $\theta$ followed by integrating out $\theta$. In addition, $P_s(\theta) = \sum_{j = 1}^L Q_s(j) P(\theta; j)$ is a mixture of exponential-family priors over $\theta$. This is an instance of latent bandits \cite{hong20latent}. For these problems, Thompson sampling can be implemented exactly and efficiently. We do not analyze this setting because prior-dependent Bayes regret bounds for this problem class do not exist yet.

\section{Regret Bounds}
\label{sec:regret bounds}

We first introduce common notation used in our proofs. The action set $\cA \subseteq \mathbb{R}^d$ is fixed. Recall that a matrix $X \in \R^{d \times d}$ is \emph{positive semi-definite (PSD)} if it is symmetric and its smallest eigenvalue is non-negative. For such $X$, we define $\sigma^2_{\max}(X) = \max_{a \in \cA} a\T X a$. Although $\sigma^2_{\max}(X)$ depends on $\cA$, we suppress this dependence because $\cA$ is fixed. We denote by $\lambda_1(X)$ the maximum eigenvalue of $X$ and by $\lambda_d(X)$ the minimum eigenvalue of $X$.

We also need basic quantities from information theory. For two probability measures $P$ and $Q$ over a common measurable space, we use $D(P || Q)$ to denote the relative entropy of $P$ with respect to $Q$. It is defined as $D(P || Q) = \int \log(\frac{d P}{d Q}) \, d P $, where $d P / d Q$ is the Radon-Nikodym derivative of $P$ with respect to $Q$; and is infinite when $P$ is not absolutely continuous with respect to $Q$. We slightly abuse our notation and let $P(X)$ denote the probability distribution of random variable $X$, $P(X \in \cdot)$. For jointly distributed random variables $X$ and $Y$, we let $P(X \mid Y)$ be the conditional distribution of $X$ given $Y$, $P(X \in \cdot \mid Y)$, which is $Y$-measurable and depends on random $Y$. The \emph{mutual information} between $X$ and $Y$ is $I(X; Y) = D(P(X, Y) || P(X) P(Y))$, where $P(X) P(Y)$ is the distribution of the product of $P(X)$ and $P(Y)$. Intuitively, $I(X; Y)$ measures the amount of information that either $X$ or $Y$ provides about the other variable. For jointly distributed $X$, $Y$, and $Z$, we also need the \emph{conditional mutual information} between $X$ and $Y$ conditioned on $Z$. We define this quantity as $I(X; Y \mid Z) = \mathbb{E}[\hat{I}(X; Y \mid Z)]$, where $\hat{I}(X; Y \mid Z) = D(P(X, Y \mid Z) || P(X \mid Z) P(Y \mid Z))$ is the \emph{random conditional mutual information} between $X$ and $Y$ given $Z$. Note that $\hat{I}(X; Y \mid Z)$ is a function of $Z$. By the chain rule for the random conditional mutual information, $\hat{I}(X; Y_1, Y_2 \mid Z) = \mathbb{E}[\hat{I}(X; Y_1 \mid Y_2, Z) \mid Z] + \hat{I}(X; Y_2 \mid Z)$, where expectation is over $Y_2 \mid Z$. We would get the usual chain rule $I(X; Y_1, Y_2) = I(X; Y_1 \mid Y_2) + I(X; Y_2)$ without $Z$.

\subsection{Generic Regret Bound}
\label{sec:generic}

We start with a generic adaptation of the analysis of \citet{lu2019information} to our setting. In round $t$ of task $s$, we denote the pulled arm by $A_{s, t}$, its observed reward by $Y_{s, t} \sim p_{A_{s, t}}(\cdot; \theta_{s, *})$, and the suboptimality gap by $\Delta_{s, t} = r(A_{s, *}; \theta_{s, *}) - r(A_{s, t}; \theta_{s, *})$. For random variables $X$ and $Y$, we denote by $I_{s, t}(X; Y) = \hat{I}(X; Y \mid H_{1 : s, t})$ the random mutual information between $X$ and $Y$ given history $H_{1 : s, t}$ of all observations from the first $s - 1$ tasks and the first $t - 1$ rounds of task $s$. Similarly, for random variables $X$, $Y$, and $Z$, we denote by $I_{s, t}(X; Y \mid Z) = \mathbb{E}[\hat{I}(X; Y \mid Z, H_{1 : s, t}) \mid H_{1 : s, t}]$ the random mutual information between $X$ and $Y$ conditioned on $Z$, given history $H_{1 : s, t}$. It is helpful to think of $I_{s, t}$ as the conditional mutual information of $X \mid H_{1 : s, t}$, $Y \mid H_{1 : s, t}$, and $Z \mid H_{1 : s, t}$.

Let $\Gamma_{s, t}$ and $\epsilon_{s, t}$ be potentially history-dependent non-negative random variables such that 
\begin{align}
  \mathbb{E}_{s, t}[\Delta_{s, t}]
  \leq \Gamma_{s, t} \sqrt{I_{s, t}(\theta_{s, *}; A_{s, t}, Y_{s, t})} + \epsilon_{s, t}
  \label{a:info}
\end{align} 
holds almost surely. We want to keep both $\Gamma_{s, t}$ and $\epsilon_{s, t}$ \say{small}. The following lemma provides a bound on the total regret over $n$ rounds in each of $m$ tasks in terms of $\Gamma_{s, t}$ and $\epsilon_{s, t}$.

\begin{restatable}[]{lemma}{decomp}
\label{lemm:decomp} Suppose that \eqref{a:info} holds for all $s \in [m]$ and $t \in [n]$, for some $\Gamma_{s, t}, \epsilon_{s, t} \geq 0$. In addition, let $(\Gamma_s)_{s \in [m]}$ and $\Gamma$ be non-negative constants such that $\Gamma_{s, t} \leq \Gamma_s \leq \Gamma$ holds for all $s \in [m]$ and $t \in [n]$ almost surely. Then
\begin{align*}
  R(m, n)
  \leq \Gamma \sqrt{m n I(\mu_*; H_{1 : m})} +
  \sum_{s = 1}^m \Gamma_s \sqrt{n I(\theta_{s, *}; H_s \mid \mu_*, H_{1 : s - 1})} +
  \sum_{s = 1}^m \sum_{t = 1}^n \E{\epsilon_{s, t}}\,.
\end{align*}
\end{restatable}

The first term above is the price for learning $\mu_*$, while the second is the price for learning all $\theta_{s, *}$ when $\mu_*$ is known. Accordingly, the price for learning $\mu_*$ is negligible when the mutual information terms grow slowly with $m$ and $n$. Specifically, we show shortly in linear bandits that $\Gamma_{s, t}$ and $\epsilon_{s, t}$ can be set so that the last term of the bound is comparable to the rest, while $\Gamma_{s, t}$ grows slowly with $m$ and $n$. At the same time, $I(\mu_*; H_{1 : m})$ and $I(\theta_{s, *}; H_s \mid \mu_*, H_{1 : s - 1})$ are only logarithmic in $m$ and $n$. Thus the price for learning $\mu_*$ is $\tilde{O}(\sqrt{m n})$ while that for learning all $\theta_{s, *}$ is $\tilde{O}(m \sqrt{n})$. Now we are ready to prove \cref{lemm:decomp}.

\begin{proof}
First, we use the chain rule of random conditional mutual information and derive
\begin{align*}
  I_{s, t}(\theta_{s, *}; A_{s, t}, Y_{s, t})
  \leq I_{s, t}(\theta_{s, *}, \mu_*; A_{s, t}, Y_{s, t})
  = I_{s, t}(\mu_*; A_{s, t}, Y_{s, t}) + I_{s, t}(\theta_{s, *}; A_{s, t}, Y_{s, t} \mid \mu_*)\,.
\end{align*} 
Now we take the square root of both sides, apply $\sqrt{a + b} \leq \sqrt{a} + \sqrt{b}$ to the right-hand side, and multiply both sides by $\Gamma_{s, t}$. This yields
\begin{align}
  \Gamma_{s, t} \sqrt{I_{s, t}(\theta_{s, *}; A_{s, t}, Y_{s, t})} 
  \leq \Gamma_{s, t} \sqrt{I_{s, t}(\mu_*; A_{s, t}, Y_{s, t})} +
  \Gamma_{s, t} \sqrt{I_{s, t}(\theta_{s, *}; A_{s, t}, Y_{s, t} \mid \mu_*)}\,.
  \label{eq:basicub}
\end{align}
We start with the second term in \eqref{eq:basicub}. Fix task $s$. From $\Gamma_{s, t} \le \Gamma_s$, followed by the Cauchy-Schwarz and Jensen's inequalities, we have
\begin{align*}
  \E{\sum_{t = 1}^n \Gamma_{s, t}
  \sqrt{I_{s, t}(\theta_{s, *}; A_{s, t}, Y_{s, t} \mid \mu_*)}}
  & \leq \Gamma_s 
  \sqrt{n \E{\sum_{t = 1}^n I_{s, t}(\theta_{s, *}; A_{s, t}, Y_{s, t} \mid \mu_*)}}\,.
\end{align*}
Thanks to $\E{I_{s, t}(\theta_{s, *}; A_{s, t}, Y_{s, t} \mid \mu_*)} = I(\theta_{s, *}; A_{s, t}, Y_{s, t} \mid \mu_*, H_{1 : s, t})$ and the chain rule of mutual information, we have $\E{\sum_{t = 1}^n I_{s, t}(\theta_{s, *}; A_{s, t}, Y_{s, t} \mid \mu_*)} = I(\theta_{s, *}; H_s \mid \mu_*, H_{1 : s - 1})$.

Now we consider the first term in \eqref{eq:basicub}. We bound $\Gamma_{s, t}$ using $\Gamma$, then apply the Cauchy-Schwarz and Jensen's inequalities, and obtain $\E{\sum_{s = 1}^m \sum_{t = 1}^n \Gamma_{s, t} \sqrt{I_{s, t}(\mu_*; A_{s, t}, Y_{s, t})}} \leq \Gamma \sqrt{m n I(\mu_*; H_{1 : m})}$; where we used the chain rule to get
\begin{align*}
  \E{\sum_{s = 1}^m \sum_{t = 1}^n I_{s, t}(\mu_*; A_{s, t}, Y_{s, t})}
  = \sum_{s = 1}^m \sum_{t = 1}^n I(\mu_*; A_{s, t}, Y_{s, t} \mid H_{1 : s, t})
  = I(\mu_*; H_{1 : m})\,.
\end{align*} 
This completes the proof.
\end{proof}

\subsection{Linear Bandit with Gaussian Rewards}
\label{sec:mutual linear}

Now we derive regret bounds for linear bandits (\cref{sec:linear bandit}). Without loss of generality, we make an assumption that the action set is bounded.

\begin{assumption}
\label{a:bounded} The arms are vectors in a unit ball, $\max_{a \in \cA} \|a\|_2 \leq 1$.
\end{assumption}

Our analysis is for \adats with a small amount of \emph{forced exploration} in each task. This guarantees that our estimate of $\mu_*$ improves uniformly in all directions after each task $s$. Therefore, we assume that the action set is diverse enough to explore in all directions.

\begin{assumption}
\label{a:forced explore} There exist arms $\set{a_i}_{i = 1}^d \subseteq \cA$ such that $\lambda_d(\sum_{i = 1}^d a_i a_i\T) \geq \eta$ for some $\eta > 0$.
\end{assumption}

This assumption is without loss of generality. In particular, if such a set does not exist, the action set $\cA$ can be projected into a lower dimensional space where the assumption holds. \adats is modified as follows. In each task, we initially pulls the arms $\set{a_i}_{i = 1}^d$ to explore all directions.

We start by showing that \eqref{a:info} holds for suitably \say{small} $\Gamma_{s, t}$ and $\epsilon_{s, t}$. In \adats, in round $t$ of task $s$, the posterior distribution of $\theta_{s, *}$ is $\cN(\hat{\mu}_{s, t}, \hat{\Sigma}_{s, t})$, where
\begin{align*}
  \hat{\mu}_{s, t}
  = \hat{\Sigma}_{s, t} \left((\Sigma_0 + \hat{\Sigma}_s)^{-1} \hat{\mu}_s +
  \sum_{\ell = 1}^{t - 1} \frac{A_{s, \ell} Y_{s, \ell}}{\sigma^2}\right)\,, \quad
  \hat{\Sigma}_{s, t}^{-1}
  = (\Sigma_0 + \hat{\Sigma}_s)^{-1} +
  \sum_{\ell = 1}^{t - 1} \frac{A_{s, \ell} A_{s, \ell}^\top}{\sigma^2}\,,
\end{align*} 
and $\hat{\mu}_s$ and $\hat{\Sigma}_s$ are defined in \cref{sec:linear bandit}. 
Then, from the properties of Gaussian distributions and that \adats samples from the posterior, we get a bound on $\Gamma_{s, t}$ and $\epsilon_{s, t}$ as a function of a tunable parameter $\delta\in (0, 1]$.

\begin{restatable}[]{lemma}{concentration}
\label{lemm:concentration} For all tasks $s \in [m]$, rounds $t \in [n]$, and any $\delta \in (0, 1]$, \eqref{a:info} holds almost surely for
\begin{align*}
  \Gamma_{s, t}
  = 4 \sqrt{\frac{\sigma^2_{\max}(\hat{\Sigma}_{s, t})}
  {\log(1 + \sigma^2_{\max}(\hat{\Sigma}_{s, t}) / \sigma^2)}
  \log(4 |\cA| / \delta)}\,, \quad
  \epsilon_{s, t}
  = \sqrt{2 \delta \sigma^2_{\max}(\hat{\Sigma}_{s, t})} +
  2 \cE_{s, t} \mathbb{E}_{s, t}[\|\theta_{s, *}\|_2]\,,
\end{align*}
where $\cE_{s, t}$ is the indicator of forced exploration in round $t$ of task $s$. Moreover, for each task $s$, the following history-independent bound holds almost surely,
\begin{align}
  \sigma^2_{\max}(\hat{\Sigma}_{s, t})
  \leq \lambda_1(\Sigma_0)
  \left(1 + \frac{\lambda_1(\Sigma_q)
  \left(1 + \frac{\sigma^2}{\eta \lambda_1(\Sigma_0)}\right)}
  {\lambda_1(\Sigma_0) + \sigma^2 / \eta + s \lambda_1(\Sigma_q)}\right)\,.
  \label{eq:smsqbound}
\end{align}
\end{restatable}

\cref{lemm:concentration} is proved in \cref{sec:proofoflemconc}.  By using the bound in \eqref{eq:smsqbound}, we get that $\Gamma_{s, t} = O(\sqrt{\log(1 / \delta)})$ and $\epsilon_{s, t} = O(\sqrt{\delta})$. \cref{lemm:concentration} differs from \citet{lu2019information} in two aspects. First, it considers uncertainty in the estimate of $\mu_*$ along with $\theta_{s, *}$. Second, it does not require that the rewards are bounded. Our next lemma bounds the mutual information terms in \cref{lemm:decomp}, by exploiting the hierarchical structure of our linear bandit model (\cref{fig:environment}).

\begin{restatable}[]{lemma}{mutual}
\label{lemm:mutual} For any $H_{1 : s, t}$-adapted action sequence and any $s \in [m]$, we have
\begin{align*}
  I(\theta_{s, *}; H_s \mid \mu_*, H_{1 : s - 1})
  \leq \tfrac{d}{2}
  \log\left(1 + \tfrac{\lambda_1(\Sigma_0) n}{\sigma^2}\right)\,, \quad
  I(\mu_*; H_{1 : m})
  \leq \tfrac{d}{2}
  \log\left(1 + \tfrac{\lambda_1(\Sigma_q) m}{\lambda_d(\Sigma_0) + \sigma^2 / n}\right)\,.
\end{align*} 
\end{restatable}

Now we are ready to prove our regret bound for the linear bandit. We take the mutual-information bounds from \cref{lemm:mutual}, and the bounds on $\Gamma_{s, t}$ and $\epsilon_{s, t}$ from \cref{lemm:concentration}, and plug them into \cref{lemm:decomp}. Specifically, $\sigma^2_{\max}(\hat{\Sigma}_{s, t}) \leq \lambda_1(\Sigma_q) + \lambda_1(\Sigma_0)$ holds for any $s$ and $t$ by \cref{lemm:concentration}, which yields $\Gamma$ in \cref{lemm:decomp}. On the other hand, $\Gamma_s$ is bounded using the upper bound in \eqref{eq:smsqbound}, which relies on forced exploration. Our regret bound is stated below. The terms $c_1$ to $c_4$ are at most polylogarithmic in $d$, $m$, and $n$; and thus small. The term $c_2$ arises due to summing up $\Gamma_s$ over all tasks $s$.

\begin{restatable}[\textbf{Linear bandit}]{theorem}{linearinfo}
\label{thm:linearinfo} The regret of \adats is bounded for any $\delta \in (0, 1]$ as
\begin{align*}
  R(m, n)
  \leq \underbrace{c_1 \sqrt{d m n}}_{\textrm{Learning of $\mu_*$}} {} +
  (m + c_2) \underbrace{R_\delta(n; \mu_*)}_{\textrm{Per-task regret}} {} +
  \underbrace{c_3 d m}_{\textrm{Forced exploration}}\,,
\end{align*}
where
\begin{align*}
  c_1
  = \sqrt{8 \tfrac{\lambda_1(\Sigma_q) + \lambda_1(\Sigma_0)}
  {\log\left(1 + \tfrac{\lambda_1(\Sigma_q) + \lambda_1(\Sigma_0)}{\sigma^2}\right)}
  \log(4 |\cA| / \delta)
  \log\left(1 + \tfrac{\lambda_1(\Sigma_q) m}{\lambda_d(\Sigma_0) + \sigma^2 / n}\right)}\,,
\end{align*}
$c_2 = \left(1 + \tfrac{\sigma^2}{\eta \lambda_1(\Sigma_0)}\right) \log m$, and $c_3 = 2 \sqrt{\|\mu_q\|_2^2 + \mathrm{tr}(\Sigma_q + \Sigma_0)}$. The \emph{per-task regret} is bounded as $R_\delta(n; \mu_*) \leq c_4 \sqrt{d n} + \sqrt{2 \delta \lambda_1(\Sigma_0)} n$, where
\begin{align*}
  c_4
  = \sqrt{8 \tfrac{\lambda_1(\Sigma_0)}
  {\log\left(1 + \tfrac{\lambda_1(\Sigma_0)}{\sigma^2}\right)}
  \log(4 |\cA| / \delta)
  \log\left(1 + \tfrac{\lambda_1(\Sigma_0) n}{\sigma^2}\right)}\,.
\end{align*}
\end{restatable}

The bound in \cref{thm:linearinfo} is sublinear in $n$ for $\delta = 1 / n^2$. It has three terms. The first term is the regret due to learning $\mu_*$ over all tasks; and it is $\tilde{O}(\sqrt{d m n})$. The second term is the regret for acting in $m$ tasks under the assumption that $\mu_*$ is known; and it is $\tilde{O}(m \sqrt{d n})$. The last term is the regret for forced exploration; and it is $\tilde{O}(d m)$. Overall, the extra regret due to unknown $\mu_*$ is $\tilde{O}(\sqrt{d m n} + d m)$ and is much lower than $\tilde{O}(m \sqrt{d n})$ when $d \ll n$. Therefore, we call \adats a \emph{no-regret algorithm} for linear bandits. Our bound also reflects the fact that the regret decreases as both priors become more informative, $\lambda_1(\Sigma_0) \to 0$ and $\lambda_1(\Sigma_q) \to 0$.

A frequentist regret bound for linear TS with finitely-many arms is $\tilde{O}(d \sqrt{n})$ \citep{agrawal13thompson}. When applied to $m$ tasks, it would be $\tilde{O}(d m \sqrt{n})$ and is worse by a factor of $\sqrt{d}$ than our regret bound. To show that our bound reflects the structure of our problem, we compare \adats to two variants of linear TS that are applied independently to each task. The first variant knows $\mu_*$ and thus has more information. Its regret can bounded by setting $c_1 = c_2 = c_3 = 0$ in \cref{thm:linearinfo} and is lower than that of \adats. The second variant knows that $\mu_* \sim \cN(\mu_q, \Sigma_q)$ but does not model that the tasks share $\mu_*$. This is analogous to assuming that $\theta_{s, *} \sim \cN(\mu_q, \Sigma_q + \Sigma_0)$. The regret of this approach can be bounded by setting $c_1 = c_2 = c_3 = 0$ in \cref{thm:linearinfo} and replacing $\lambda_1(\Sigma_0)$ in $c_4$ by $\lambda_1(\Sigma_q + \Sigma_0)$. Since the task regret increases linearly with $m$ and $\lambda_1(\Sigma_q + \Sigma_0) > \lambda_1(\Sigma_0)$, this approach would ultimately have a higher regret than \adats as the number of tasks $m$ increases.

\subsection{Semi-Bandit with Gaussian Rewards}
\label{sec:mutual semi-bandit}

In semi-bandits (\cref{sec:semi-bandit}), we use the independence of arms to decompose the per-round regret differently. Similarly to \cref{sec:mutual linear}, we analyze \adats with forced exploration, where each arm is initially pulled at least once. This is always possible in at most $K$ rounds, since there exists at least one $a \in \cA$ that contains any given arm.

Let $\Gamma_{s, t}(k)$ and $\epsilon_{s, t}(k)$ be non-negative history-dependent constants, for each arm $k \in [K]$, were we use $(k)$ to refer to arm-specific quantities. Then an analogous bound to \eqref{a:info} is
\begin{align*}
  \mathbb{E}_{s, t}[\Delta_{s, t}]
  \leq \sum_{k \in [K]} \mathbb{P}_{s, t}(k \in A_{s, t})
  \left(\Gamma_{s, t}(k) \sqrt{I_{s, t}(\theta_{s, *}(k); k, Y_{s, t}(k))} +
  \epsilon_{s, t}(k)\right)\,.
\end{align*}
The term $(k, Y_{s, t}(k))$ is a tuple of a pulled arm $k$ and its observation in round $t$ of task $s$. For any $k$, from the chain rule of mutual information, we have
\begin{align*}
  I_{s, t}(\theta_{s, *}(k); k, Y_{s, t}(k))
  \leq I_{s, t}(\mu_*(k); k, Y_{s, t}(k)) +
  I_{s, t}(\theta_{s, *}(k); k, Y_{s, t}(k) \mid \mu_*(k))\,.
\end{align*} 
Next we combine the mutual-information terms across all rounds and tasks, as in \cref{lemm:decomp}, and bound corresponding $\Gamma_{s, t}(k)$ and $\epsilon_{s, t}(k)$ independently of $m$ and $n$. Due to forced exploration, the estimate of $\mu_*(k)$ improves for all arms $k$ as more tasks are completed, and $\Gamma_{s, t}(k)$ decreases with $s$. This leads to \cref{thm:semibanditinfo}, which is proved in \cref{sec:appendix-semi-bandit}.

\begin{restatable}[\textbf{Semi-bandit}]{theorem}{semibanditinfo} \label{thm:semibanditinfo}
The regret of \adats is bounded for any $\delta \in (0, 1]$ as
\begin{align*}
  R(m, n)
  \leq \underbrace{c_1 \sqrt{K L m n}}_{\textrm{Learning of $\mu_*$}} {} +
  (m + c_2) \underbrace{R_\delta(n; \mu_*)}_{\textrm{Per-task regret}} {} +
  \underbrace{c_3 K^{3 / 2} m }_{\textrm{Forced exploration}} {} +
  c_4 \sigma \sqrt{2 \delta m} n\,,
\end{align*}
where
\begin{align*}
  c_1
  & = 4 \sqrt{\tfrac{1}{K} \sum_{k \in [K]}
  \tfrac{\sigma_{q, k}^2 + \sigma_{0, k}^2}
  {\log\left(1 + \tfrac{\sigma_{q, k}^2 + \sigma_{0, k}^2}{\sigma^2}\right)}
  \log(4 K / \delta)
  \log\left(1 + \tfrac{\sigma_{q, k}^2 m}{\sigma_{0, k}^2 + \sigma^2 / n}\right)}\,, \\
  c_2
  & = \left(1 + \max_{k \in [K]: \, \sigma_{0, k} > 0}
  \tfrac{\sigma^2}{\sigma_{0, k}^2}\right) \log m\,, \quad
  c_3
  = 2 \sqrt{\sum_{k \in [K]} (\mu_{q, k}^2 + \sigma_{q, k}^2 + \sigma_{0, k}^2)}\,, \\
  c_4
  & = \sqrt{\tfrac{1}{K} \sum_{k \in [K]: \, \sigma_{0, k} = 0}
  \log\left(1 + \tfrac{\sigma_{q, k}^2 m}{\sigma^2}\right)}\,.
\end{align*}
The \emph{per-task regret} is bounded as $R_\delta(n; \mu_*) \leq c_5 \sqrt{K L n} + \sqrt{2 \delta \tfrac{1}{K} \sum_{k \in [K]} \sigma_{0, k}^2} n$, where
\begin{align*}
  c_5
  = 4 \sqrt{\tfrac{1}{K} \sum_{\substack{k \in [K]: \, \sigma_{0, k} > 0}}
  \tfrac{\sigma_{0, k}^2}{\log\left(1 + \tfrac{\sigma_{0, k}^2}{\sigma^2}\right)}
  \log(4 K / \delta)
  \log\left(1 + \tfrac{\sigma_{0, k}^2 n}{\sigma^2}\right)}\,.
\end{align*}
The prior widths $\sigma_{q, k}$ and $\sigma_{0, k}$ are defined as in \cref{sec:gaussian bandit}.
\end{restatable}

The bound in \cref{thm:semibanditinfo} is sublinear in $n$ for $\delta = 1 / n^2$. Its form resembles \cref{thm:linearinfo}. Specifically, the regret for learning $\mu_*$ is $\tilde{O}(\sqrt{K L m n})$ and for forced exploration is $\tilde{O}(K^{3 / 2} m)$. Both of these are much lower than the regret for learning to act in $m$ tasks when $\mu_*$ is known, $\tilde{O}(m \sqrt{K L n})$, for $K \ll L n$. Therefore, \adats is also a \emph{no-regret algorithm} for semi-bandits.

\cref{thm:semibanditinfo} improves upon a naive application of \cref{thm:linearinfo} to semi-bandits. This is because all prior width constants are averages, as opposing to the maximum over arms in \cref{thm:linearinfo}. To the best of our knowledge, such per-arm prior dependence has not been captured in semi-bandits by any prior work. To illustrate the difference, consider a problem where $\sigma_{0, k} > 0$ for only $K' \ll K$ arms. This means that only $K'$ arms are uncertain in the tasks. Then the bound in \cref{thm:semibanditinfo} is $\tilde{O}(m \sqrt{K' L n})$, while the bound in \cref{thm:linearinfo} would be $\tilde{O}(m \sqrt{K L n})$. For the arms $k$ where $\sigma_{0, k} = 0$, the regret over all tasks is sublinear in $m$.

\section{Experiments}
\label{sec:experiments}

\begin{figure*}[t]
  \centering
  \includegraphics[width=5.4in]{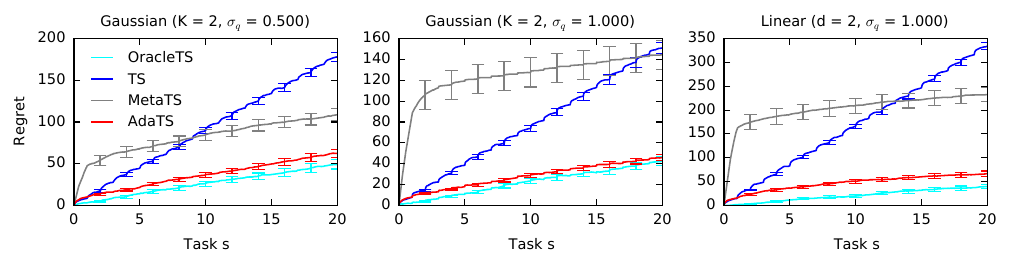}
  \vspace{-0.1in}
  \caption{Comparison of \adats to three baselines on three bandit problems.}
  \label{fig:synthetic}
\end{figure*}

We experiment with two synthetic problems. In both problems, the number of tasks is $m = 20$ and each task has $n = 200$ rounds. The first problem is a Gaussian bandit (\cref{sec:gaussian bandit}) with $K = 2$ arms. The meta-prior is $\cN(\mathbf{0}, \Sigma_q)$ with $\Sigma_q = \sigma_q^2 I_K$, the prior covariance is $\Sigma_0 = \sigma_0^2 I_K$, and the reward noise is $\sigma = 1$. We experiment with $\sigma_q \geq 0.5$ and $\sigma_0 = 0.1$. Since $\sigma_q \gg \sigma_0$, the entries of $\theta_{s, *}$ are likely to have the same order as in $\mu_*$. Therefore, a clever algorithm that learns $\mu_*$ could have very low regret. The second problem is a linear bandit (\cref{sec:linear bandit}) in $d = 2$ dimensions with $K = 5 d$ arms. The action set is sampled from a unit sphere. The meta-prior, prior, and noise are the same as in the Gaussian bandit. All results are averaged over $100$ runs.

\adats is compared to three baselines. The first is idealized TS with the true prior $\cN(\mu_*, \Sigma_0)$ and we call it \oraclets. \oraclets shows the minimum attainable regret. The second is agnostic TS, which ignores the structure of the problem. We call it \ts and implement it with prior $\cN(\mathbf{0}, \Sigma_q + \Sigma_0)$, since $\theta_{s, *}$ can be viewed as a sample from this prior when the structure is ignored (\cref{sec:mutual linear}). The third baseline is \metats of \citet{kveton21metathompson}. All methods are evaluated by their cumulative regret up to task $s$, which we plot as it accumulates round-by-round within each task (\cref{fig:synthetic}). The regret of the algorithms that do not learn $\mu_*$ (\oraclets and \ts) is obviously linear in $s$, as they solve $s$ similar tasks with the same policy (\cref{sec:setting}). A lower slope indicates a better policy. As no algorithm can outperform \oraclets, no regret can grow sublinearly in $s$.

\begin{wrapfigure}{r}{1.5in}
  \centering
  \vspace{-0.15in}
  \includegraphics[width=1.5in]{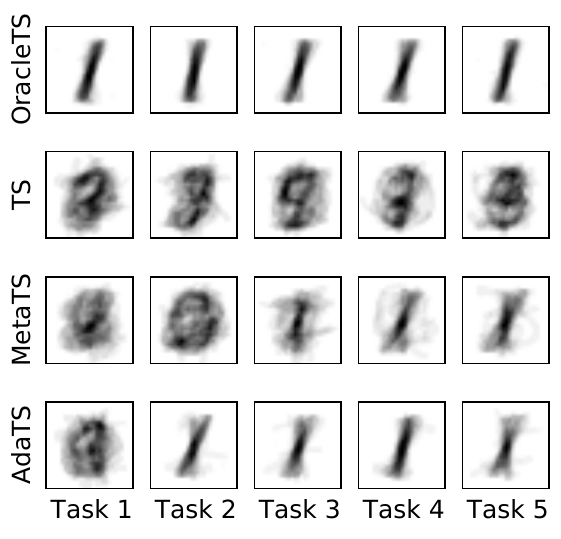}
  \caption{Meta-learning of a highly rewarding digit $1$.}
  \vspace{-0.15in}
  \label{fig:preview}
\end{wrapfigure}

Our results are reported in \cref{fig:synthetic}. We start with a Gaussian bandit with $\sigma_q = 0.5$. This setting is identical to Figure 1b of \citet{kveton21metathompson}. We observe that \adats outperforms \ts, which does not learn $\mu_*$, and is comparable to \oraclets, which knows $\mu_*$. Its regret is about $30\%$ lower than that of \metats. Now we increase the meta-prior width to $\sigma_q = 1$. In this setting, meta-parameter sampling in \metats leads to high biases in earlier tasks. This leads to a major increase in regret, while \adats performs comparably to \oraclets. We end with a linear bandit with $\sigma_q = 1$. In this experiment, \adats outperforms \metats again and has more than three times lower regret.

\cref{sec:supplementary experiments} contains more experiments. In \cref{sec:synthetic experiments}, we experiment with more values of $K$ and $d$, and show the robustness of \adats to missspecified meta-prior $Q$. In \cref{sec:classification experiments}, we apply \adats to bandit classification problems. In \cref{fig:preview}, we show results for one of these problems, meta-learning a highly rewarding digit $1$ in the bandit setting. For each method and task $s$, we show the average digit corresponding to the pulled arms in round $1$ of task $s$. \adats learns a good meta-parameter $\mu_*$ almost instantly, since its average digit in task $2$ already resembles digit $1$.

\section{Related Work}
\label{sec:related work}

Two closest related works are \citet{bastani19meta} and \citet{kveton21metathompson}. \citet{bastani19meta} proposed Thompson sampling that learns the prior from a sequence of pricing experiments. The algorithm is tailored to pricing and learns through forced exploration using conservative TS. Therefore, it is conservative. \citet{bastani19meta} also did not derive prior-dependent bounds.

Our studied setting is identical to \citet{kveton21metathompson}. However, the design of \adats is very different from \metats. \metats samples the meta-parameter $\mu_s$ at the beginning of each task $s$ and uses it to compute the posterior of the task parameter $\theta_{s, *}$. Since $\mu_s$ is fixed within the task, \metats does not have a correct posterior of $\theta_{s, *}$ given the history. \adats marginalizes out the uncertainty in the meta-parameter $\mu_*$ and thus has a correct posterior of $\theta_{s, *}$ within the task. This seemingly minor difference leads to an approach that is more principled, comparably general, has a fully-Bayesian analysis beyond multi-armed bandits, and may have several-fold lower regret in practice. While it is possible that the analysis of \metats could be extended to linear bandits, the price for meta-learning would likely remain $\tilde{O}(\sqrt{m} n^2)$. This cost arises due sampling the meta-parameter $\mu_s$ at the beginning of each task $s$. The price of meta-learning in our work is mere $\tilde{O}(\sqrt{m n})$, a huge improvement.

\adats is a meta-learning algorithm \citep{thrun96explanationbased,thrun98lifelong,baxter98theoretical,baxter00model,finn17modelagnostic,finn18probabilistic}. Meta-learning has a long history in multi-armed bandits. Some of the first works are \citet{azar13sequential} and \citet{gentile14online}, who proposed UCB algorithms for multi-task learning. \citet{deshmukh17multitask} studied multi-task learning in contextual bandits. \citet{cella20metalearning} proposed a UCB algorithm that meta-learns the mean parameter vector in a linear bandit, which is akin to learning $\mu_*$ in \cref{sec:linear bandit}. Another recent work is \citet{yang20provable}, who studied regret minimization with multiple parallel bandit instances, with the goal of learning their shared subspace. All of these works are frequentist, analyze a stronger notion of regret, and often lead to conservative algorithm designs. In contrast, we leverage the fundamentals of Bayesian reasoning to design a general-purpose algorithm that performs well when run as analyzed.

Several recent papers approached the problem of learning a bandit algorithm using policy gradients \citep{duan16rl2,boutilier20differentiable,kveton20differentiable,yang20differentiable,min20policy}, including learning Thompson sampling \citep{kveton20differentiable,min20policy}. These works focus on offline optimization against a known bandit-instance distribution and have no convergence guarantees in general \citep{boutilier20differentiable,kveton20differentiable}. Tuning of bandit algorithms is known to reduce regret \citep{vermorel05multiarmed,maes12metalearning,kuleshov14algorithms,hsu19empirical}. Typically it is ad-hoc and we believe that meta-learning is a proper way of framing this problem.

\section{Conclusions}
\label{sec:conlusions}

We propose \adats, a fully-Bayesian algorithm for meta-learning in bandits that adapts to a sequence of bandit tasks that it interacts with. \adats attains low regret by adapting the uncertainty in both the meta and per-task parameters. We analyze the Bayes regret of \adats using information-theory tools that isolate the effect of learning the meta-parameter from that of learning the per-task parameters. For linear bandits and semi-bandits, we derive novel prior-dependent regret bounds that show that the price for learning the meta-parameter is low. Our experiments underscore the generality of \adats, good out-of-the-box performance, and robustness to meta-prior misspecification.

We leave open several questions of interest. For instance, except for \cref{sec:exponential-family bandit}, our algorithms are for Gaussian rewards and priors, and so are their regret analyses. An extension beyond Gaussians would be of both practical and theoretical value. Our current work also relies heavily on a particular parameterization of tasks, where the mean $\theta_{s, *}$ is unknown but the covariance $\Sigma_0$ is known. It is not immediately obvious if a computationally-efficient extension to unknown $\Sigma_0$ exists.

\clearpage

\bibliographystyle{plainnat}
\bibliography{References}

\clearpage
\onecolumn
\appendix

\section{Algorithm Details}
\label{sec:appendix}

Our terminology is summarized below:

\begin{tabular}{ll}
  $\theta_{s, *}$ &
  Bandit instance parameter in task $s$, generated as $\theta_{s, *} \sim P(\cdot; \mu_*)$ \\
  $P(\cdot; \mu_*)$ &
  Task prior, a distribution over bandit instance parameter $\theta_{s, *}$ \\
  $\mu_*$ &
  Meta-parameter, a parameter of the task distribution \\
  $Q$ &
  Meta-prior, a distribution over the meta-parameter $\mu_*$ \\
  $P_s$ &
  Uncertainty-adjusted prior in task $s$, a distribution over $\theta_{s, *}$ conditioned on $H_{1 : s - 1}$ \\
  $Q_s$ &
  Meta-posterior in task $s$, a distribution over $\mu_*$ conditioned on $H_{1 : s - 1}$ \\
  $Y_{s, t}$ &
  Stochastic rewards of all arms in round $t$ of task $s$ \\
  $A_{s, t}$ &
  Pulled arm in round $t$ of task $s$ 
\end{tabular}

We continue with two lemmas, which are used in the algorithmic part of the paper (\cref{sec:algorithm}).

\posterior*
\begin{proof}
To simplify presentation, our proof is under the assumption that $\theta_{s, *}$ and $\mu_*$ take on countably-many values. A more general measure-theory treatment, where we would maintain measures over $\theta_{s, *}$ and $\mu_*$, would follow the same line of reasoning; and essentially replace all probabilities with densities. A good discussion of this topic is in Section 34 of \citet{lattimore19bandit}.

The following convention is used in the proof. The values of random variables that we marginalize out, such as $\theta_{s, *}$ and $\mu_*$, are explicitly assigned. For fixed variables, such as the history $H_{1 : s - 1}$, we also treat $H_{1 : s - 1}$ as the actual value assigned to $H_{1 : s - 1}$.

We start with the posterior distribution of $\theta_{s, *}$ in task $s$, which can be expressed as
\begin{align*}
  \condprob{\theta_{s, *} = \theta}{H_{1 : s - 1}}
  & = \sum_\mu \condprob{\theta_{s, *} = \theta, \mu_* = \mu}{H_{1 : s - 1}} \\
  & = \sum_\mu \condprob{\theta_{s, *} = \theta}{\mu_* = \mu}
  \condprob{\mu_* = \mu}{H_{1 : s - 1}}\,.
\end{align*}
The second equality holds because $\theta_{s, *}$ is independent of history $H_{1 : s - 1}$ given $\mu_*$. Now note that $\condprob{\mu_* = \mu}{H_{1 : s - 1}}$ is the meta-posterior in task $s$. It can be rewritten as
\begin{align*}
  \condprob{\mu_* = \mu}{H_{1 : s - 1}}
  & = \frac{\condprob{\mu_* = \mu}{H_{1 : s - 1}}}{\condprob{\mu_* = \mu}{H_{1 : s - 2}}}
  \condprob{\mu_* = \mu}{H_{1 : s - 2}} \\
  & = \frac{\condprob{H_{s - 1}}{H_{1 : s - 2}, \, \mu_* = \mu}}
  {\condprob{H_{s - 1}}{H_{1 : s - 2}}} \condprob{\mu_* = \mu}{H_{1 : s - 2}} \\
  & \propto \underbrace{\condprob{H_{s - 1}}{H_{1 : s - 2}, \, \mu_* = \mu}}_{f_1(\mu)}
  \condprob{\mu_* = \mu}{H_{1 : s - 2}}\,,
\end{align*}
where $\condprob{\mu_* = \mu}{H_{1 : s - 2}}$ is the meta-posterior in task $s - 1$. The last step follows from the fact that $\condprob{H_{s - 1}}{H_{1 : s - 2}}$ is constant in $\mu$. Now we focus on $f_1(\mu)$ above and rewrite it as
\begin{align*}
  f_1(\mu)
  & = \sum_\theta \condprob{H_{s - 1}, \, \theta_{s - 1, *} = \theta}
  {H_{1 : s - 2}, \, \mu_* = \mu} \\
  & = \sum_\theta
  \condprob{H_{s - 1}}{H_{1 : s - 2}, \, \theta_{s - 1, *} = \theta, \, \mu_* = \mu}
  \condprob{\theta_{s - 1, *} = \theta}{H_{1 : s - 2}, \, \mu_* = \mu} \\
  & = \sum_\theta
  \underbrace{\condprob{H_{s - 1}}{H_{1 : s - 2}, \, \theta_{s - 1, *} = \theta}}_{f_2(\theta)}
  \condprob{\theta_{s - 1, *} = \theta}{\mu_* = \mu}\,.
\end{align*}
In the last step, we use that the history $H_{s - 1}$ is independent of $\mu_*$ given $H_{1 : s - 2}$ and $\theta_{s - 1, *}$, and that the task parameter $\theta_{s - 1, *}$ is independent of $H_{1 : s - 2}$ given $\mu_*$.

Now we focus on $f_2(\theta)$ above. To simplify notation, it is useful to define $Y_t = Y_{s - 1, t}(A_{s - 1, t})$ and $A_t = A_{s - 1, t}$. Then we can rewrite $f_2(\theta)$ as
\begin{align*}
  f_2(\theta)
  & = \prod_{t = 1}^n \condprob{A_t, Y_t}{H_{1 : s - 1, t}, \, \theta_{s - 1, *} = \theta} \\
  & = \prod_{t = 1}^n \condprob{Y_t}{A_t, \, H_{1 : s - 1, t}, \, \theta_{s - 1, *} = \theta}
  \condprob{A_t}{H_{1 : s - 1, t}, \, \theta_{s - 1, *} = \theta} \\
  & = \prod_{t = 1}^n \condprob{Y_t}{A_t, \, \theta_{s - 1, *} = \theta}
  \condprob{A_t}{H_{1 : s - 1, t}}
  \propto \condprob{Y_{1 : n}}{A_{1 : n}, \, \theta_{s - 1, *} = \theta}\,.
\end{align*}
In the third equality, we use that the reward $Y_t$ is independent of history $H_{1 : s - 1, t}$ given the pulled arm $A_t$ and task parameter $\theta_{s - 1, *}$, and that $A_t$ is independent of $\theta_{s - 1, *}$ given $H_{1 : s - 1, t}$. In the last step, we use that $\condprob{A_t}{H_{1 : s - 1, t}}$ is constant in $\theta$.

Finally, we combine all above claims, note that
\begin{align*}
  \condprob{\theta_{s, *} = \theta}{\mu_* = \mu}
  = \condprob{\theta_{s - 1, *} = \theta}{\mu_* = \mu}
  = P(\theta; \mu)\,,
\end{align*}
and get
\begin{align*}
  \condprob{\theta_{s, *} = \theta}{H_{1 : s - 1}}
  & = \sum_\mu P(\theta; \mu) \,
  \condprob{\mu_* = \mu}{H_{1 : s - 1}}\,, \\
  \condprob{\mu_* = \mu}{H_{1 : s - 1}}
  & = \sum_\theta \condprob{Y_{1 : n}}{A_{1 : n}, \, \theta_{s - 1, *} = \theta}
  P(\theta; \mu) \,
  \condprob{\mu_* = \mu}{H_{1 : s - 2}}\,.
\end{align*}
These are the claims that we wanted to prove, since
\begin{align*}
  P_s(\theta)
  & = \condprob{\theta_{s, *} = \theta}{H_{1 : s - 1}}\,, \\
  Q_s(\mu)
  & = \condprob{\mu_* = \mu}{H_{1 : s - 1}}\,, \\
  \cL_{s - 1}(\theta)
  & = \condprob{Y_{1 : n}}{A_{1 : n}, \, \theta_{s - 1, *} = \theta}\,.
\end{align*}
This concludes the proof.
\end{proof}

\begin{lemma}
\label{lem:multi-task bayesian regression} Fix integers $s$ and $n$, features $(x_{\ell, t})_{\ell \in [s], t \in [n]}$, and consider a generative process
\begin{align*}
  \mu_*
  & \sim \cN(\mu_q, \Sigma_q)\,, \\
  \forall \ell \in [s]: \theta_{\ell, *} \mid \mu_*
  & \sim \cN(\mu_*, \Sigma_0)\,, \\
  \forall \ell \in [s], t \in [n]: Y_{\ell, t} \mid \mu_*
  & \sim \cN(x_{\ell, t}\T \theta_{\ell, *}, \sigma^2)\,,
\end{align*}
where all variables are drawn independently. Then $\mu_* \mid (Y_{\ell, t})_{\ell \in [s], \, t \in [n]} \sim \cN(\hat{\mu}, \hat{\Sigma})$ for
\begin{align*}
  \hat{\mu}
  & = \hat{\Sigma} \left(\Sigma_q^{-1} \mu_q +
  \sum_{\ell = 1}^s \frac{B_\ell}{\sigma^2} - \frac{G_\ell}{\sigma^2}
  \left(\Sigma_0^{-1} + \frac{G_\ell}{\sigma^2}\right)^{-1} \frac{B_\ell}{\sigma^2}\right)\,, \\
  \hat{\Sigma}^{-1}
  & = \Sigma_q^{-1} +
  \sum_{\ell = 1}^s \frac{G_\ell}{\sigma^2} - \frac{G_\ell}{\sigma^2}
  \left(\Sigma_0^{-1} + \frac{G_\ell}{\sigma^2}\right)^{-1} \frac{G_\ell}{\sigma^2}\,,
\end{align*}
where $G_\ell = \sum_{t = 1}^n x_{\ell, t} x_{\ell, t}\T$ is the outer product of the features in task $\ell$ and $B_\ell = \sum_{t = 1}^n x_{\ell, t} Y_{\ell, t}$ is their sum weighted by observations.
\end{lemma}
\begin{proof}
The claim is proved in Appendix D of \citet{kveton21metathompson}. We restate it for completeness.
\end{proof}

\clearpage

\section{Proofs for \cref{sec:generic}: Generic Regret Bound}

\subsection{Preliminaries and Omitted Definitions}

\paragraph{Notation for History:} 
Let us recall that $H_{s,t} = ((A_{s,1}, Y_{s,1}), \dots, (A_{s,t-1}, Y_{s,t-1}))$ denote the events in task $s$ upto and excluding round $t$ for all $t\geq 1$ ($H_{s,1}= \emptyset$). The events in task $s$ is denoted as $H_{s} = H_{s,n+1}$ and all the events upto and including stage $s$ is denoted as $H_{1:s} = \cup_{s'=1}^{s} H_{s'}$. Let us also define history upto and excluding round $t$ in task $s$ as $H_{1:s,t}=\{H_{1:s-1}\cup H_{s,t}\}$, with $H_{1:s}= H_{1:s, n+1}$. Given the history upto and excluding round $t$ in task $s$, the conditional probability is given as $\mathbb{P}_{s,t}(\cdot) = \mathbb{P}[\cdot\mid H_{1:s,t}]$, and   the conditional expectation is given as $\mathbb{E}_{s,t}(\cdot)= \mathbb{E}[\cdot\mid H_{1:s,t}]$. Note $\mathbb{P}[\cdot]$ and $\mathbb{E}[\cdot]$ denote the unconditional probability and expectation, respectively. 

\paragraph{History dependent Entropy and Mutual Information:}
We now define the entropy and mutual information terms as a function of history.

The mutual information between the parameter $\theta_{s,*}$, and the action ($A_{s,t}$) and reward ($Y_{s,t}$) at the beginning of  round $t$ in task $s$, for any $s\leq m$ and $t\leq n$, as a function of history is defined as 
\begin{align*}
&I_{s,t}(\theta_{s,*}; A_{s,t}, Y_{s,t}) = \mathbb{E}_{s,t}\left[\log\left(\frac{\mathbb{P}_{s,t}(\theta_{s,*}, Y_{s,t}, A_{s,t})}{\mathbb{P}_{s,t}(\theta_{s,*})\mathbb{P}_{s,t}(Y_{s,t}, A_{s,t})}\right)\right]
% &:= \sum_{a\in \mathcal{A}} \int_{\mu, \theta, y}  \mathbb{P}_{s,t}(\mu_{*}\mathtt{=}\mu, \theta_{s,*}\mathtt{=}\theta, Y_{s,t}\mathtt{=}y, A_{s,t}\mathtt{=}a) \log\left(\frac{\mathbb{P}_{s,t}(\theta_{s,*}=\theta, Y_{s,t}= y, A_{s,t}= a)}{\mathbb{P}_{s,t}(\theta_{s,*}=\theta)\mathbb{P}_{s,t}(Y_{s,t}= y, A_{s,t}= a)}\right) d\mu d\theta dy
\end{align*} 
We also define the mutual information between the parameter $\mu_{*}$, and the action ($A_{s,t}$) and reward ($Y_{s,t}$) at the beginning of  round $t$ in task $s$, for any $s\leq m$ and $t\leq n$ as
\begin{align*}
&I_{s,t}(\mu_{*}; A_{s,t}, Y_{s,t}) = \mathbb{E}_{s,t}\left[\log\left(\frac{\mathbb{P}_{s,t}(\mu_{*}, Y_{s,t}, A_{s,t})}{\mathbb{P}_{s,t}(\mu_{*})\mathbb{P}_{s,t}(Y_{s,t}, A_{s,t})}\right)\right]
%\\ &:= \sum_{a\in \mathcal{A}} \int_{\mu, y}  \mathbb{P}_{s,t}(\mu_{*}=\mu, Y_{s,t}= y, A_{s,t}= a) \log\left(\frac{\mathbb{P}_{s,t}(\mu_{*}=\mu, Y_{s,t}= y, A_{s,t}= a)}{\mathbb{P}_{s,t}(\mu_{*}=\mu)\mathbb{P}_{s,t}(Y_{s,t}= y, A_{s,t}= a)}\right) d\mu dy
\end{align*}

Further, the history dependent conditional mutual information between $(\mu_{*}, \theta_{s,*})$, and $A_{s,t}$ and $Y_{s,t}$, namely $I_{s,t}(\theta_{s,*}, \mu_{*}; A_{s,t}, Y_{s,t})$, is defined below. 
\begin{align*}
&I_{s,t}(\theta_{s,*}, \mu_{*}; A_{s,t}, Y_{s,t})  = \mathbb{E}_{s,t}\left[\log\left(\frac{\mathbb{P}_{s,t}(\theta_{s,*},\mu_{*}, Y_{s,t}, A_{s,t})}{\mathbb{P}_{s,t}(\theta_{s,*},\mu_{*})\mathbb{P}_{s,t}(Y_{s,t}, A_{s,t})}\right)\right]
% &\\:= \int_{\mu} \sum_{a\in \mathcal{A}} \int_{\theta, y}  \mathbb{P}_{s,t}(\mu_{*} = \mu, \theta_{s,*}=\theta, Y_{s,t}= y, A_{s,t}= a) \times \\
% &\hspace{5em}\times\log\left(\frac{\mathbb{P}_{s,t}(\mu_{*} = \mu, \theta_{s,*}=\theta, Y_{s,t}= y, A_{s,t}= a)}{\mathbb{P}_{s,t}(\mu_{*} = \mu, \theta_{s,*}=\theta)\mathbb{P}_{s,t}(Y_{s,t}= y, A_{s,t}= a)}\right) d\theta dy d\mu
\end{align*}

Finally, we define the  history dependent conditional mutual information between $\theta_{s,*}$, and $A_{s,t}$ and $Y_{s,t}$ given $\mu_{*}$ as $I_{s,t}(\theta_{s,*}; A_{s,t}, Y_{s,t}\mid\mu_{*})$.
\begin{align*}
&I_{s,t}(\theta_{s,*}; A_{s,t}, Y_{s,t}\mid\mu_{*}) = \mathbb{E}_{s,t}\left[\log\left(\frac{\mathbb{P}_{s,t}(\theta_{s,*}, Y_{s,t}, A_{s,t}\mid \mu_{*})}{\mathbb{P}_{s,t}(\theta_{s,*}\mid \mu_{*})\mathbb{P}_{s,t}(Y_{s,t}, A_{s,t} \mid\mu_{*})}\right)\right]
%\\ &:= \int_{\mu} \mathbb{P}_{s,t}(\mu_{*} = \mu) \sum_{a\in \mathcal{A}} \int_{\theta, y}  \mathbb{P}_{s,t}(\theta_{s,*}=\theta, Y_{s,t}= y, A_{s,t}= a\mid\mu_{*} = \mu) \times \\
% &\hspace{5em}\times\log\left(\frac{\mathbb{P}_{s,t}(\theta_{s,*}=\theta, Y_{s,t}= y, A_{s,t}= a\mid\mu_{*} = \mu)}{\mathbb{P}_{s,t}(\theta_{s,*}=\theta\mid\mu_{*} = \mu)\mathbb{P}_{s,t}(Y_{s,t}= y, A_{s,t}= a\mid\mu_{*} = \mu)}\right) d\theta dy d\mu
\end{align*}

The conditional entropy terms are defined as follows:
\begin{align*}
  h_{s,t}(\theta_{s,*})
  & = \mathbb{E}_{s,t}\left[-\log\left(\mathbb{P}_{s,t}(\theta_{s,*})\right)\right]\,, \\
  \quad h_{s,t}(\mu_{*})
  & = \mathbb{E}_{s,t}\left[-\log\left(\mathbb{P}_{s,t}(\mu_{*})\right)\right]\,, \\
  h_{s,t}(\theta_{s,*}\mid\mu_{*})
  & = \mathbb{E}_{s,t}\left[-\log\left(\mathbb{P}_{s,t}(\theta_{s,*}\mid\mu_{*})\right)\right]\,.
\end{align*}
Therefore, all the different mutual information terms $I_{s,t}(\cdot; A_{s,t}, Y_{s,t})$, and the entropy terms $h_{s,t}(\cdot)$ are random variables that depends on the history $H_{1:s,t}$. 

We next state some entropy and mutual information relationships which we will use later.
\begin{proposition}\label{prop:chain}
For all $s$, $t$, and any history $H_{1:s,t}$, the following hold
\begin{align*}
  I_{s,t}(\theta_{s,*}, \mu_{*}; A_{s,t}, Y_{s,t})
  & = I_{s,t}(\mu_{*}; A_{s,t}, Y_{s,t}) + I_{s,t}(\theta_{s,*}; A_{s,t}, Y_{s,t}\mid\mu_{*})\,, \\
  I_{s,t}(\theta_{s,*}; A_{s,t}, Y_{s,t})
  & = h_{s,t}(\theta_{s,*}) - h_{s,t+1}(\theta_{s,*})\,.
\end{align*}
\end{proposition}

\paragraph{History Independent Entropy and Mutual Information:} 
The  history independent conditional mutual information and entropy terms are then given by taking expectation over the possible histories 
\begin{gather*}
    I(\cdot; A_{s,t}, Y_{s,t}\mid H_{1:s, t}) = \mathbb{E}[I_{s,t}(\cdot; A_{s,t}, Y_{s,t})], \quad h(\cdot\mid H_{1:s, t}) = \mathbb{E}[h_{s,t}(\cdot)]\\
    I(\cdot; A_{s,t}, Y_{s,t}\mid \mu_{*}, H_{1:s, t}) = \mathbb{E}[I_{s,t}(\cdot; A_{s,t}, Y_{s,t}\mid \mu_{*})], \quad h(\cdot\mid \mu_{*}, H_{1:s, t}) = \mathbb{E}[h_{s,t}(\cdot\mid \mu_{*})]
\end{gather*} 

An important quantity that will play a pivotal role in our regret decomposition is the  conditional mutual information of the meta-parameter given the entire history, which is expressed as 
$$I(\mu_{*}; H_{1:m}) = \sum_{s=1}^{m}\sum_{t=1}^{n} I(\mu_{*}; A_{s,t}, Y_{s,t}\mid H_{1:s,t}) = \mathbb{E} \sum_{s=1}^{m}\sum_{t=1}^{n} I_{s,t}(\mu_{*}; A_{s,t}, Y_{s,t}).$$ 
The first equality is due to chain rule of mutual information, where at each round the new history  $H_{1:s,t+1} = H_{1:s,t} \cup (A_{s,t}, Y_{s,t})$. 

Similarly, in each stage $s$, the mutual information between parameter $\theta_{s,*}$  and the events in stage $s$, i.e. $H_s$, conditioned on  $\mu_{*}$ and history up to task $(s-1)$ is key in quantifying the local regret of task $s$. Which is again expressed as 
$$I(\theta_{s,*}; H_{s}\mid \mu_{*},H_{1:s-1}) = \sum_{t=1}^{n} I(\theta_{s,*}; A_{s,t}, Y_{s,t}\mid \mu_{*},H_{1:s, t})  = \mathbb{E} \sum_{t=1}^{n} I_{s,t}(\theta_{s,*}; A_{s,t}, Y_{s,t}\mid \mu_{*}).$$ The first inequality again follows chain rule of mutual information with new history being the combination of old history, and the action and the observed reward in the current round.

We further have the relation of mutual information and conditional entropy as 
\begin{align*}
  I(\theta_{s,*}; H_{s}\mid \mu_{*},H_{1:s-1})
  & = h(\theta_{s,*}\mid \mu_{*},H_{1:s-1}) - h(\theta_{s,*}\mid \mu_{*},H_{1:s})\,, \\
  I(\mu_{*}; H_{1:m})
  & = h(\mu_{*}) - h(\mu_{*}\mid H_{1:m})\,.
\end{align*}

\paragraph{Weyl's Inequalities:}
 In this paper, the matrices under consideration are all Positive Semi-definite (PSD) and symmetric. Thus, the eignevalues are non-negative and admits a total order.  We denote  the eigenvalues of a PSD matrix $A \in \mathbb{R}^d$, for any integer $d\geq 1$, as $\lambda_d(A) \leq \dots \leq \lambda_1(A)$; where  $\lambda_1(A)$ is the maximum  eigenvalue, and  $\lambda_d(A)$ is the minimum eigenvalue of the PSD matrix  $A$. 

Weyl's inequality states for two Hermitian matrices (PSD and Symmetric in reals) $A$ and $B$, 
$$
\lambda_{j}(A) + \lambda_{k}(B) \leq \lambda_{i}(A+B) \leq \lambda_{r}(A) + \lambda_{s}(B),\quad \forall\, j+k-d\geq i\geq r+s-1.
$$
The two important relations, derived from Weyl's inequality, that we frequently use in the proofs are given next. For PSD and symmetric matrices $\{A_i\}$ we have 
$$
\lambda_{1}(\sum_{i} A_i) \leq \sum_i\lambda_{1}(A_i),\quad \text{ and }\lambda_{d}(\sum_{i} A_i) \geq \sum_{i} \lambda_{d}(A_i).
$$

\clearpage

\decomp*
\begin{proof}
The proof follows through the series of inequalities below (explanation added).
\begin{align*}
  R(m, n) &= \mathbb{E}\sum_{s,t}[\Delta_{s,t}]\\
  [\text{Eq.}~\eqref{a:info}]&\leq \mathbb{E}\sum_{s,t}\Gamma_{s,t} \sqrt{I_{s,t}(\theta_{s,*}; A_{s,t}, Y_{s,t})} + \mathbb{E}\sum_{s,t}\epsilon_{s,t}\\
  [I(X;Z) \leq I(X,Y;Z)]&\leq \mathbb{E}\sum_{s,t}\Gamma_{s,t} \sqrt{I_{s,t}(\theta_{s,*}, \mu_{*}; A_{s,t}, Y_{s,t})} + \mathbb{E}\sum_{s,t}\epsilon_{s,t}\\
  [\text{Chain Rule}]&= \mathbb{E}\sum_{s,t}\Gamma_{s,t} \sqrt{I_{s,t}(\mu_{*}; A_{s,t}, Y_{s,t})+I_{s,t}(\theta_{s,*}; A_{s,t}, Y_{s,t}\mid \mu_{*})} + \mathbb{E}\sum_{s,t}\epsilon_{s,t}\\
  [\sqrt{a+b}\mathtt{\leq} \sqrt{a}\mathtt{+}\sqrt{b}]&\leq \mathbb{E}\sum_{s,t}\Gamma_{s,t} \sqrt{I_{s,t}(\mu_{*}; A_{s,t}, Y_{s,t})} + \mathbb{E}\sum_{s,t}\Gamma_{s,t}\sqrt{I_{s,t}(\theta_{s,*}; A_{s,t}, Y_{s,t}\mid \mu_{*})}\\ 
  & + \mathbb{E}\sum_{s,t}\epsilon_{s,t}\\
  [\Gamma_{s,t}\leq \Gamma_s \leq \Gamma, \forall s,t, \text{ w.p. } 1]&\leq \Gamma\, \mathbb{E}\sum_{s,t}\sqrt{ I_{s,t}(\mu_{*}; A_{s,t}, Y_{s,t})} +  \sum_{s}\Gamma_{s}\left[\mathbb{E} \sum_{t}\sqrt{I_{s,t}(\theta_{s,*}; A_{s,t}, Y_{s,t}\mid \mu_{*})}\right]\\ 
  &+ \mathbb{E}\sum_{s,t}\epsilon_{s,t}\\
  [\text{Jensen's Inequality}]&\leq \Gamma\sum_{s,t} \sqrt{\mathbb{E} I_{s,t}(\mu_{*}; A_{s,t}, Y_{s,t})} +  \sum_{s}\Gamma_{s}\sum_{t}\sqrt{\mathbb{E} I_{s,t}(\theta_{s,*}; A_{s,t}, Y_{s,t}\mid \mu_{*})} \\
  &+ \mathbb{E}\sum_{s,t}\epsilon_{s,t}\\
  [\text{Cauchy-Schwarz}]&\leq \Gamma\sqrt{mn\sum_{s,t}\mathbb{E} I_{s,t}(\mu_{*}; A_{s,t}, Y_{s,t})} +  \sum_{s}\Gamma_{s}\sqrt{n\sum_{t}\mathbb{E} I_{s,t}(\theta_{s,*}; A_{s,t}, Y_{s,t}\mid \mu_{*})}\\ 
  &+ \mathbb{E}\sum_{s,t}\epsilon_{s,t}\\
  [\text{Chain Rule}]&= \Gamma \sqrt{ mnI(\mu_{*}; H_{1:m})} +  \sum_{s}\Gamma_{s}\sqrt{n I(\theta_{s,*}; H_{s}\mid \mu_{*}, H_{1:s-1})} + \mathbb{E}\sum_{s,t}\epsilon_{s,t}
\end{align*}
\begin{itemize}
    \item[-] The first inequality follows due to Eq.~\eqref{a:info}. 
    \item[-] The second inequality uses the fact that $I(X;Z) \leq I(X,Y;Z)$ for any random variables $X$, $Y$, and $Z$. Here $X = \theta_{s,*}$, $Y = \mu_{*}$, and $Z = (A_{s,t}, Y_{s,t})$.
    \item[-] The second equality uses the chain rule $I(X,Y;Z) = I(X;Z) + I(X;Z\mid Y)$, as stated in Proposition~\ref{prop:chain}, with the same random variables $X$, $Y$, and $Z$. 
    \item[-] The Jensen's inequality uses concavity of $\sqrt{\cdot}$. 
\end{itemize}
\end{proof}

\section{Proofs for \cref{sec:mutual linear}: Linear Bandit}

\subsection{Marginalization of the Variables}

\paragraph{Notation in Marginalization:}Let $\mathcal{N}(x; \mu,\Sigma)$ denote a (possibly multivariate) Gaussian p.d.f. with mean $\mu$ and covariance matrix $\Sigma$ for variable $x$. We now recall the notations of posterior distributions at different time of our algorithm
\begin{align*}
  &P(\theta; \mu)
  = \condprob{\theta_{s, *} = \theta}{\mu_* = \mu} = \mathcal{N}(\theta; \mu, \Sigma_0),\quad
  Q(\mu) = \prob{\mu_* = \mu} = \mathcal{N}(\mu; \mu_0, \Sigma_q)\\
  &P_s(\theta)
  = \condprob{\theta_{s, *} = \theta}{H_{1 : s - 1}} = \int_\mu P(\theta; \mu) Q_s(\mu) \dif \mu,\\ 
  &P_{s, t}(\theta) = \condprob{\theta_{s, *} = \theta}{H_{1 : s, t}} 
  \propto \condprob{H_{s, t}}{\theta_{s, *} = \theta} P_s(\theta),\\
  Q_s(\mu)
  & = \condprob{\mu_* = \mu}{H_{1 : s - 1}} = \int_\theta \condprob{H_{s - 1}}{\theta_{s - 1, *} = \theta} P(\theta; \mu) \dif \theta  Q_{s - 1}(\mu)
\end{align*}

The marginalization is proved in an inductive manner due to the dependence of the action matrix $A$ on the history.
We recall the expression of the rewards,
\begin{align*}
    Y_{s,t} = A_{s,t}^T\theta_{s,*} + w_{s,t}
\end{align*}

In each round $t$ and task $s$, given the parameter $\theta_{s,*}$ and the action $A_{s,t}$, the reward $Y_{s,t}$ has the p.d.f. $\mathbb{P}(Y_{s,t}\mid \theta_{s,*}, A_{s,t}) = \mathcal{N}(Y_{s,t}; A_{s,t}^T \theta_{s,*}, \sigma^2)$. Let $\propto_{X}$ denote that the proportionality constant is independent of $X$ (possibly a set). 

We obtain the posterior probability of the true parameter in task $s$ in round $t$, given the true parameter $\mu_{*}$. Let us define for all $s \leq m$, and $t\leq n$.
\begin{align*}
    &P_{s,t, \mu_{*}}(\theta)= \mathbb{P}(\theta_{s,*}= \theta\mid \mu_{*}, H_{1:s, t})
    \propto \prod_{t'=1}^{t-1} \mathbb{P}(Y_{s,t'}\mid \theta_{s,*}= \theta, A_{s,t'}) P(\theta, \mu_{*})\\
    & \propto_{\theta}  \prod_{t'=1}^{t-1} \exp\left(-\frac{(Y_{s,t'}- A_{s,t'}^T\theta)^2}{2\sigma^2}\right) \mathcal{N}(\theta; \mu_{*}, \Sigma_{0})\\
    & \propto_{\theta} \exp\left(-\sum_{t'=1}^{t-1} (\theta - A_{s,t'} Y_{s,t'})^T \tfrac{A_{s,t'}A^T_{s,t'}}{2\sigma^2}(\theta - A_{s,t'} Y_{s,t'})\right) \mathcal{N}(\theta; \mu_{*}, \Sigma_{0})\\
    & \propto_{\theta}  \exp\left(- \left(\theta - \bar{\theta}\right)^T 
   \sum_{t'=1}^{t-1}\tfrac{A_{s,t'}A^T_{s,t'}}{2\sigma^2}
   \left(\theta - \bar{\theta}\right)\right)\mathcal{N}(\theta; \mu_{*}, \Sigma_{0}) \quad\left[\bar{\theta} = (\sum_{t'=1}^{t-1}\tfrac{A_{s,t'}A^T_{s,t'}}{\sigma^2})^{-1}\sum_{t'=1}^{t-1} A_{s,t'} Y_{s,t'}\right]\\
   & \propto_{\theta}   \mathcal{N}\left(\theta; \hat{\Sigma}_{s,t, \mu_{*}}\left(\Sigma_{0}^{-1}\mu_{*} + \sum_{t'=1}^{t-1} A_{s,t'} Y_{s,t'}\right), \hat{\Sigma}_{s,t, \mu_{*}} \right) \quad\left[\hat{\Sigma}_{s,t, \mu_{*}}^{-1} = \Sigma_{0}^{-1}+ \sum_{t'=1}^{t-1}\tfrac{A_{s,t'}A^T_{s,t'}}{\sigma^2}\right]
\end{align*}

We now obtain the posterior probability of the true parameter in task $s$ in round $t$ as by taking integral over the prior of the parameter $\mu_{*}$.
\begingroup
\allowdisplaybreaks
\begin{align*}
   &P_{s,t}(\theta)=\mathbb{P}(\theta_{s,*}= \theta\mid H_{1:s,t}) \propto  \prod_{t'=1}^{t-1} \mathbb{P}(Y_{s,t'}\mid \theta_{s,*}= \theta, A_{s,t'}) \int_{\mu}P(\theta, \mu)Q_{s}(\mu) d\mu\\
   & \propto_{\theta}  \prod_{t'=1}^{t-1} \exp\left(-\frac{(Y_{s,t'}- A_{s,t'}^T\theta)^2}{2\sigma^2}\right) \int_{\mu}\mathcal{N}(\theta; \mu, \Sigma_{0}) \mathcal{N}(\mu; \hat{\mu}_{s}, \hat{\Sigma}_{s}) d\mu\\
   & \propto_{\theta}  \exp\left(-\sum_{t'=1}^{t-1} (\theta - A_{s,t'} Y_{s,t'})^T \tfrac{A_{s,t'}A^T_{s,t'}}{2\sigma^2}(\theta - A_{s,t'} Y_{s,t'})\right) \mathcal{N}(\theta; \hat{\mu}_{s}, \Sigma_{0} + \hat{\Sigma}_{s})\\
   & \propto_{\theta}  \exp\left(-\sum_{t'=1}^{t-1} (\theta - A_{s,t'} Y_{s,t'})^T \tfrac{A_{s,t'}A^T_{s,t'}}{2\sigma^2}(\theta - A_{s,t'} Y_{s,t'})\right) \mathcal{N}(\theta; \hat{\mu}_{s}, \Sigma_{0} +\hat{\Sigma}_{s})\\
   & \propto_{\theta}  \exp\left(- \left(\theta - \bar{\theta}\right)^T 
   \sum_{t'=1}^{t-1}\tfrac{A_{s,t'}A^T_{s,t'}}{2\sigma^2}
   \left(\theta - \bar{\theta}\right)\right) \mathcal{N}(\theta; \hat{\mu}_{s}, \Sigma_{0} +\hat{\Sigma}_{s}) \quad\left[\bar{\theta} = (\sum_{t'=1}^{t-1}\tfrac{A_{s,t'}A^T_{s,t'}}{\sigma^2})^{-1}\sum_{t'=1}^{t-1} A_{s,t'} Y_{s,t'}\right]\\
   & \propto_{\theta}  \mathcal{N}\left(\theta; (\sum_{t'=1}^{t-1}\tfrac{A_{s,t'}A^T_{s,t'}}{\sigma^2})^{-1}\sum_{t'=1}^{t-1} A_{s,t'} Y_{s,t'}, (\sum_{t'=1}^{t-1}\tfrac{A_{s,t'}A^T_{s,t'}}{\sigma^2})^{-1}\right)
   \mathcal{N}(\theta; \hat{\mu}_{s}, \Sigma_{0} +\hat{\Sigma}_{s})\\
   & \propto_{\theta}   \mathcal{N}\left(\theta; \hat{\Sigma}_{s,t}\left((\Sigma_{0}+\hat{\Sigma}_{s})^{-1}\hat{\mu}_{s} + \sum_{t'=1}^{t-1} A_{s,t'} Y_{s,t'}\right), \hat{\Sigma}_{s,t} \right) \quad\left[\hat{\Sigma}_{s,t}^{-1} = (\Sigma_{0}+\hat{\Sigma}_{s})^{-1} + \sum_{t'=1}^{t-1}\tfrac{A_{s,t'}A^T_{s,t'}}{\sigma^2}\right]
\end{align*}
\endgroup

Thus, for  $\hat{\mu}_{s,t} = \hat{\Sigma}_{s,t}\left((\Sigma_{0}+\hat{\Sigma}_{s})^{-1}\hat{\mu}_{s} + \sum_{t'=1}^{t-1} A_{s,t'} Y_{s,t'}\right)$, the parameter conditioned on the history is distributed as  $\theta_{s,*} \mid H_{1:s,t} \sim \mathcal{N}(\hat{\mu}_{s,t}, \hat{\Sigma}_{s,t})$.

We now compute the posterior of the meta-parameter $\mu_{*}$ in a similar way, but some of the computation can be avoided by using \cref{lem:multi-task bayesian regression}. 
\begin{align*}
    Q_{s+1}(\mu) &= \int_\theta \mathbb{P}(H_{s}\mid \theta_{s,*}=\theta)P(\theta;\mu) d\theta Q_{s}(\mu)\\
    &\propto_{\theta,\mu} \int_\theta \prod_{t=1}^{n} \mathbb{P}(Y_{s,t}\mid \theta_{s,*}= \theta, A_{s,t}) P(\theta, \mu) d\theta Q_{s}(\mu)\\
    &\propto_{\theta,\mu} \prod_{\ell=1}^{s}\int_{\theta_{\ell}} \prod_{t=1}^{n} \mathbb{P}(Y_{\ell,t}\mid \theta_{\ell,*}= \theta_\ell, A_{\ell,t}) P(\theta_{\ell}, \mu) d\theta_s Q_{0}(\mu)\\
    &= \cN(\hat{\mu}_{s+1}, \hat{\Sigma}_{s+1}) 
\end{align*}
The second equality is obtained by expanding out the $Q_{s}(\mu)$ expressions iteratively, and using the fact that $Q_{0}(\mu)$ is the prior distribution of $\mu$ at the beginning. The final equality follows from the application of \cref{lem:multi-task bayesian regression}, by observing that the expression describes a setting identical to the setting therein, with actions $x_{\ell,t} = A_{\ell,t}$ for all $\ell\in [s]$ and $t\in [n]$. The probability of playing the actions $A_{\ell,t}$ (as oppossed to fixed $x_{\ell,t}$ in \cref{lem:multi-task bayesian regression}) are absorbed by the proportionality constant. 

Recall that we have due to \cref{lem:multi-task bayesian regression}, for $G_{\ell} = \sum_{t=1}^{n}A_{\ell,t}A^T_{\ell,t},\, \forall \ell \in [m]$ and for any $s\in [m]$, 
$$
\hat{\Sigma}_{s}^{-1}
= \Sigma_q^{-1} +
  \sum_{\ell = 1}^{s-1} \tfrac{G_\ell}{\sigma^2} - \tfrac{G_\ell}{\sigma^2}
  \left(\Sigma_0^{-1} + \tfrac{G_\ell}{\sigma^2}\right)^{-1} \tfrac{G_\ell}{\sigma^2}
 = \Sigma_q^{-1} +
  \sum_{\ell = 1}^{s-1} \tfrac{G_\ell}{\sigma^2}\left(\Sigma_0^{-1} + \tfrac{G_\ell}{\sigma^2}\right)^{-1}\Sigma_0^{-1}.
$$
Further, if in task $\ell$ if forced exploration is used, then $G_{\ell}$ is invertible, and  using Woodbury matrix identity we have 
$$
 \tfrac{G_\ell}{\sigma^2}\left(\Sigma_0^{-1} + \tfrac{G_\ell}{\sigma^2}\right)^{-1}\Sigma_0^{-1} = \left(\Sigma_0 + (\tfrac{G_\ell}{\sigma^2})^{-1}\right)^{-1}.
$$

\clearpage

\subsection{Proof of Lemma~\ref{lemm:mutual}}
\label{sec:proofoflemmut}

\mutual*
\begin{proof}
We obtain the conditional mutual entropy of $\theta_{s,*}$ given the history upto $(s-1)$-th task and $\theta$ (similar to Lu et al.\cite{lu2019information})
\begin{align*}
    I(\theta_{s,*}; H_{s}\mid \mu_{*}, H_{1:s-1})
    &= h(\theta_{s,*}\mid \mu_{*}, H_{1:s-1}) - h(\theta_{s,*}\mid \mu_{*},H_{1:s})\\
    &= \mathbb{E}[h_{s-1,n+1}(\theta_{s,*}\mid \mu_{*})] - \mathbb{E}[h_{s,n+1}(\theta_{s,*}\mid \mu_{*})]\\
    &= \tfrac{1}{2}\log(\det(2\pi e \Sigma_0)) - \mathbb{E}[ \tfrac{1}{2}\log(\det(2\pi e \hat{\Sigma}_{s,n, \mu_{*}}))]\\
    &=\tfrac{1}{2}\mathbb{E}[\log(\det(\Sigma_0) \det(\hat{\Sigma}_{s,n, \mu_{*}}^{-1}))]\\
    &= \tfrac{1}{2}\mathbb{E}\left[\prod_{i=1}^{d} \lambda_i(\Sigma_0)\lambda_i(\hat{\Sigma}_{s,n, \mu_{*}}^{-1})))\right]\\
    &\leq \tfrac{1}{2}\log\left(\prod_{i=1}^{d} \lambda_i(\Sigma_0) \left(\frac{1}{\lambda_i(\Sigma_0)} + \frac{n}{\sigma^2}\right)\right)\\
    &\leq \tfrac{d}{2}\log\left(1 + n\frac{\lambda_1(\Sigma_0)}{\sigma^2}\right)
\end{align*} 
The first inequality follows from the definition of conditional mutual information (here we have outer expectation). Using the relation between the history-independent and history-dependent entropy terms we obtain the second inequality. Note that $h_{s-1,n+1}(\theta_{s,*}\mid \mu_{*})$ is independent of history, as the $\theta_{s,*}$ given $\mu_{*}$ does not depend on old tasks.

For the first inequality, we derive the following history independent bound. 
\begin{align*}
\lambda_i(\hat{\Sigma}_{s,n, \mu_{*}}^{-1}) &= \lambda_i\left(\Sigma_0^{-1} + \tfrac{1}{\sigma^2} \sum_{t'=1}^{n} A_{s,t'}A_{s,t'}^T\right) \\
&\leq \lambda_i\left(\Sigma_0^{-1}\right) + \lambda_1\left(\sum_{t'=1}^{n}\frac{A_{s,t'}A_{s,t'}^T}{\sigma^2}\right)\\
&\leq \frac{1}{\lambda_i(\Sigma_0)} + tr\left(\sum_{t'=1}^{n}\frac{A_{s,t'}A_{s,t'}^T}{\sigma^2}\right)\\
&\leq \frac{1}{\lambda_i(\Sigma_0)} + \frac{n}{\sigma^2}
\end{align*}

The matrices $\tfrac{1}{\sigma^2} \sum_{t'=1}^{n} A_{s,t'}A_{s,t'}^T$, and $\Sigma_0^{-1}$ are Hermitian matrices, giving us the first inequality by applicaiton of Weyl's inequality.
The last inequality first uses linearity of trace, and $tr(A_{s,t'}A_{s,t'}^T) = tr(A_{s,t'}^TA_{s,t'})\leq 1$, by Assumption~\ref{a:bounded}.

Similarly, we derive the mutual information of the meta-parameter of $\theta$ given the history as follows
\begin{align*}
    I(\mu_{*}; H_{1:m}) 
    &= h(\mu_{*}) - h(\mu_{*}\mid H_{1:m})\\
    &= h(\mu_{*}) - \mathbb{E}[h_{m,n+1}(\mu_{*})]\\
    &= \tfrac{1}{2}\log(\det(2\pi e \Sigma_q)) - \mathbb{E}[ \tfrac{1}{2}\log(\det(2\pi e \hat{\Sigma}_{m+1}))]\\
    &= \tfrac{1}{2}\mathbb{E}[ \log(\det(\Sigma_q)\det(\hat{\Sigma}_{m+1}^{-1}))]\\
    &\leq \tfrac{d}{2}\log\left(1+ \frac{mn \lambda_1(\Sigma_q)}{n\lambda_{d}(\Sigma_0) + \sigma^2} \right)
\end{align*} 
For the final inequality above, we derive a history independent bounds in a similar manner.
\begingroup
\allowdisplaybreaks
\begin{align*}
    &\lambda_i(\hat{\Sigma}^{-1}_{m+1})
    \leq\lambda_i(\Sigma_q^{-1}) + \lambda_1\left(\sum_{s'=1}^{m}\left(\Sigma_0 + (\sum_{t'=1}^{n}\tfrac{A_{s',t'}A^T_{s',t'}}{\sigma^2})^{-1}\right)^{-1}\right)\\
    &
    \leq\lambda_i(\Sigma_q^{-1}) + \sum_{s'=1}^{m}\lambda_1\left(\left(\Sigma_0 + (\sum_{t'=1}^{n}\tfrac{A_{s',t'}A^T_{s',t'}}{\sigma^2})^{-1}\right)^{-1}\right)\\
    &\leq \frac{1}{\lambda_i(\Sigma_q)} +  \sum_{s'=1}^{m} \lambda_{d}^{-1}\left(\Sigma_0 + (\sum_{t'=1}^{n}\tfrac{A_{s',t'}A^T_{s',t'}}{\sigma^2})^{-1}\right)\\
     &\leq \frac{1}{\lambda_i(\Sigma_q)} +  \sum_{s'=1}^{m} \left(\lambda_{d}(\Sigma_0) + \lambda_{d}\left((\sum_{t'=1}^{n}\tfrac{A_{s',t'}A^T_{s',t'}}{\sigma^2})^{-1}\right)\right)^{-1}\\
    &\leq \frac{1}{\lambda_i(\Sigma_q)} +  \sum_{s'=1}^{m} \left(\lambda_{d}(\Sigma_0) + \lambda_{1}^{-1}(\sum_{t'=1}^{n}\tfrac{A_{s',t'}A^T_{s',t'}}{\sigma^2})\right)^{-1}\\
    &\leq \frac{1}{\lambda_i(\Sigma_q)} +  \sum_{s'=1}^{m} \left(\lambda_{d}(\Sigma_0) + \tfrac{\sigma^2}{n}\right)^{-1} 
    = \frac{1}{\lambda_i(\Sigma_q)} +  \frac{mn}{n\lambda_{d}(\Sigma_0) + \sigma^2}
\end{align*}
\endgroup
\end{proof}

\subsection{Proof of Lemma~\ref{lemm:concentration}}
\label{sec:proofoflemconc}
\concentration*
\begin{proof}
We next derive the confidence interval bounds, similar to Lu et al.~\cite{lu2019information}, for the reward $Y_{s,t}$ around it's mean conditioned on the history $H_{s-1}\cup H_{s,t-1}$. Let $\hat{\theta}_{s,t}$ be the parameter sampled by TS in task $s$ and round $t$, when we do not have forced exploration.
\begin{align*}
    \mathbb{E}_{s,t}[\Delta_{s,t}] &= \mathbb{E}_{s,t}[A_{s,*}^T\theta_{s,*} - A_{s,t}^T\theta_{s,*}]
    = \mathbb{E}_{s,t}[A_{s,t}^T\hat{\theta}_{s,t} - A_{s,t}^T\theta_{s,*}]
\end{align*}
The last equality holds as for Thompson sampling  ($\stackrel{d}{=}$ denotes equal distribution)
$$
A_{s,*}^T\theta_{s,*}\mid H_{1:s,t} \stackrel{d}{=} A_{s,t}^T\hat{\theta}_{s,t}\mid H_{1:s,t}.
$$

When for task $s$ and round $t$ we have  forced exploration the bound is given as 
\begin{align*}
    \mathbb{E}_{s,t}[\Delta_{s,t}] &= \mathbb{E}_{s,t}[A_{s,t}^T\hat{\theta}_{s,t} - A_{s,t}^T\theta_{s,*}] + \mathbb{E}_{s,t}[A_{s,*}^T\theta_{s,*} - A_{s,t}^T\hat{\theta}_{s,t}]\\
    &\leq \mathbb{E}_{s,t}[A_{s,t}^T\hat{\theta}_{s,t} - A_{s,t}^T\theta_{s,*}] + 2 \mathbb{E}_{s,t}[ \max_{a\in \mathcal{A}} |a^T \theta_{s,*}|]\\
     &\leq \mathbb{E}_{s,t}[A_{s,t}^T\hat{\theta}_{s,t} - A_{s,t}^T\theta_{s,*}] + 2 \mathbb{E}_{s,t}[  \|\theta_{s,*}\|_2].
\end{align*}
In the second last inequality we use the fact that $\theta_{s,*}\mid H_{1:s,t} \stackrel{d}{=} \hat{\theta}_{s,t}\mid H_{1:s,t}$.

Recall $Y_{s,t}(a)$ denote the reward obtained by taking action $a$ in task $s$ and round $t$.  Also recall that $\hat{\theta}_{s,t}\mid H_{1:s,t}\sim \mathcal{N}(\hat{\mu}_{s,t}, \hat{\Sigma}_{s,t})$.
Let us consider the set 
$$\Theta_{s,t} = \{\theta: \mid a^T\theta - a^T\hat{\theta}_{s,t}\mid \leq \tfrac{\Gamma_{s,t}}{2}\sqrt{I_{s,t}(\theta_{s,*}; a, Y_{s,t}(a))}, \forall a\in \mathcal{A}\}.$$ 

The history dependent conditional mutual entropy of $\theta_{s,*}$ given the history $H_{s-1}\cup H_{s,t}$   (not $\mu_*$) (which will be useful in deriving concentration bounds) as 
\begin{align*}
    I_{s,t}(\theta_{s,*};A_{s,t},Y_{s,t})
    &= h_{s,t}(\theta_{s,*}) - h_{s,t+1}(\theta_{s,*})\\
    &= \tfrac{1}{2}\log(\det(2\pi e (\hat{\Sigma}_{s,t-1}))) - \tfrac{1}{2}\log(\det(2\pi e \hat{\Sigma}_{s,t}))\\
    &= \tfrac{1}{2}\log(\det(\hat{\Sigma}_{s,t-1}\hat{\Sigma}^{-1}_{s,t}))\\
    &= \tfrac{1}{2}\log\left(\det\left(I+\hat{\Sigma}_{s,t-1}\tfrac{A_{s,t}A_{s,t}^T}{\sigma^2}\right)\right)\\
    &= \tfrac{1}{2}\log\left(\det\left(1+\tfrac{A_{s,t}^T\hat{\Sigma}_{s,t-1}A_{s,t}}{\sigma^2}\right)\right)
\end{align*} The last step above uses Matrix determinant lemma.\footnote{Matrix determinant lemma states that for an invertible square matrix $A$,  and vectors $u$ and $v$ 
$\det \left(A + uv^T\right)=\left(1+ v^T A^{-1} u \right)\,\det \left(A\right).$ We use $A = I$, $u = \hat{\Sigma}_{s,t-1}A_{s,t}$, and $v = A_{s,t}/\sigma^2$.}
Recall that $\sigma^2_{\max}(\hat{\Sigma}_{s,t})  =\max_{a\in \mathcal{A}} a^T \hat{\Sigma}_{s,t} a$ for all $s\leq m$ and $t\leq n$. For $\delta\in (0,1]$, let 
$$\Gamma_{s,t} = 4\sqrt{\frac{\sigma^2_{\max}(\hat{\Sigma}_{s,t-1})}{\log(1+\sigma^2_{\max}(\hat{\Sigma}_{s,t-1})/\sigma^2 )}\log(\tfrac{4|\mathcal{A}|}{\delta}}).$$ 
Now  it follows from Lu et al.~\cite{lu2019information} Lemma 5 that for the $\Gamma_{s,t}$ defined as above we have 
$$\mathbbm{P}_{s,t}(\hat{\theta}_{s,t} \in \Theta_{s,t}) \geq 1- \delta/2.$$

We continue with the regret decomposition as 
\begin{align*}
    &\mathbb{E}_{s,t}[\Delta_{s,t}]\\
    &= \mathbb{E}_{s,t}\left[\mathbbm{1}(\hat{\theta}_{s,t},\theta_{s,*}\in \Theta_{s,t})\left(A_{s,t}^T\hat{\theta}_{s,t} - A_{s,t}^T\theta_{s,*}\right)\right] 
    + \mathbb{E}_{s,t}\left[\mathbbm{1}^c(\hat{\theta}_{s,t},\theta_{s,*}\in \Theta_{s,t})\left(A_{s,t}^T\hat{\theta}_{s,t} - A_{s,t}^T\theta_{s,*}\right)\right]\\
    &\leq \mathbb{E}_{s,t}\left[\sum_{a\in \mathcal{A}}\mathbbm{1}(A_{s,t}=a)\Gamma_{s,t}\sqrt{I_{s,t}(\theta_{s,*}; a, Y_{s,t}(a))}\right] \\
    &+ \sqrt{\mathbbm{P}_{s,t}(\hat{\theta}_{s,t} \text{ or } \theta_{s,*}\notin \Theta_{s,t})\mathbb{E}_{s,t}\left[\left(A_{s,t}^T\hat{\theta}_{s,t} - A_{s,t}^T\theta_{s,*}\right)^2\right]}\\
    &\leq \Gamma_{s,t}\sqrt{I_{s,t}(\theta_{s,*}; A_{s,t}Y_{s,t})} + \sqrt{\mathbbm{P}_{s,t}(\hat{\theta}_{s,t} \text{ or } \theta_{s,*}\notin \Theta_{s,t})} \max\limits_{a\in \mathcal{A}} \sqrt{\mathbb{E}_{s,t}\left[\left(a^T\hat{\theta}_{s,t} - a^T\theta_{s,*}\right)^2\right]} \\
    &\leq \Gamma_{s,t}\sqrt{I_{s,t}(\theta_{s,*}; A_{s,t}Y_{s,t})} + \underbrace{ \sqrt{2\delta \sigma^2_{\max}(\hat{\Sigma}_{s,t-1})}}_{\epsilon_{s,t}}
\end{align*}
\begin{itemize}
    \item[-] The left side term in the first inequality uses the definition of $\Theta_{s,t}$. The right side term in the first inequality holds due to Cauchy–Schwarz. In particular, we use $\mathbb{E}[XY] \leq \sqrt{\mathbb{E}[X^2]\mathbb{E}[Y^2]}$ with $X=\mathbbm{1}^c(\hat{\theta}_{s,t},\theta_{s,*}\in \Theta_{s,t})$ and $Y=\left(A_{s,t}^T\hat{\theta}_{s,t} - A_{s,t}^T\theta_{s,*}\right)$. 
    \item[-] The left side term in the second inequality follows steps similar to proof of Lemma 3 in Lu et al.~\cite{lu2019information}.
The right side term in the second inequality maximizes over the possible actions (we can take the max out of the expectation as action $A_{s,t}$ is a function of history upto task $s$, and round $t-1$).  The last inequality follows from the following derivation  
\begin{align*}
  &\mathbb{E}_{s,t}\left[\left(a^T\hat{\theta}_{s,t} - a^T\theta_{s,*}\right)^2\right]\\
  &\leq \mathbb{E}_{s,t}\left[a^T\left((\hat{\theta}_{s,t}- \mu_{s,t-1}) - (\theta_{s,*}- \mu_{s,t-1})\right)^2\right]\\
  &\leq a^T\left(\mathbb{E}_{s,t}\left[(\hat{\theta}_{s,t}- \mu_{s,t-1})(\hat{\theta}_{s,t}- \mu_{s,t-1})^T \right] + \mathbb{E}_{s,t}\left[(\theta_{s,*}- \mu_{s,t-1})(\theta_{s,*}- \mu_{s,t-1})^T \right]\right)a\\
  & \leq 2 a^T \hat{\Sigma}_{s,t-1} a \leq 2 \sigma^2_{\max}(\hat{\Sigma}_{s,t-1})
\end{align*}
\end{itemize}

This conclude the proof of the first part.

We first claim that $\sigma^2_{\max}(\hat{\Sigma}_{s,t})\leq \lambda_1(\hat{\Sigma}_{s,t})$. Indeed,  as $\|a\|_2\leq 1$, we have $$\sigma^2_{\max}(\hat{\Sigma}_{s,t}) = \max_{a\in \mathcal{A}} a^T \hat{\Sigma}_{s,t} a \leq \max_{a\in \mathcal{A}} a^T\lambda_1(\hat{\Sigma}_{s,t}) a\leq \lambda_1(\hat{\Sigma}_{s,t}).$$  

Furthermore, $\lambda_1(\hat{\Sigma}_{s,t})$ decreases with $s$ and $t$ (precisely with $n(s-1) + t$).
To show this we use 
\begin{align*}
\lambda_1(\hat{\Sigma}_{s,t}) &= \lambda_{d}^{-1}(\hat{\Sigma}_{s,t}^{-1}) \\
&= \lambda_{d}^{-1}\left( (\Sigma_0+\hat{\Sigma}_s)^{-1} + \sum_{t'=1}^{t}\tfrac{A_{s,t'}A^T_{s,t'}}{\sigma^2}\right)\\
&\leq \lambda_{d}^{-1}\left((\Sigma_0+\hat{\Sigma}_s)^{-1} + \sum_{t'=1}^{t-1}\tfrac{A_{s,t'}A^T_{s,t'}}{\sigma^2}\right) \\
&= \lambda_{d}^{-1}(\hat{\Sigma}_{s,t-1}^{-1}) 
= \lambda_1(\hat{\Sigma}_{s,t-1}) 
\end{align*}
The inequality holds due to Weyl's inequality and $\tfrac{A_{s,t}A^T_{s,t}}{\sigma^2}$ being a PSD matrix. In particular, we have $\lambda_{d}(A+B) \geq \lambda_{d}(A) + \lambda_{d}(B)$, given $A$ and $B$ are Hermitian. Thus 
$$\lambda^{-1}_{d}(A+B) \leq (\lambda_{d}(A) + \lambda_{d}(B))^{-1} \leq \lambda^{-1}_{d}(A).$$

Recall in each task $s$, due to forced exploration, we have $\lambda_d(\sum_{t'=1}^{n}\tfrac{A_{s,t'}A^T_{s,t'}}{\sigma^2}) \geq \tfrac{\eta}{\sigma^2}$, where $\eta$ is the forced exploration constant. 
We now prove an upper bound for the term $\lambda_1(\hat{\Sigma}_{s})$ independent of action sequences.
\begin{align*}
&\lambda_1(\Sigma_0+\hat{\Sigma}_{s}) - \lambda_1(\Sigma_0) \leq \lambda_1(\hat{\Sigma}_{s}) = \lambda_{d}^{-1}(\hat{\Sigma}_{s}^{-1}) \\
&=\lambda_{d}^{-1}\left(\Sigma_q^{-1} + \sum_{s'=1}^{s-1} \left(\sum_{t'=1}^{n}\tfrac{A_{s',t'}A^T_{s',t'}}{\sigma^2}\right) \left(\Sigma_0^{-1} + \sum_{t'=1}^{n}\tfrac{A_{s',t'}A^T_{s',t'}}{\sigma^2}\right)^{-1} \Sigma_0^{-1} \right)\\
&\leq\left(\lambda_{d}(\Sigma_q^{-1}) + \sum_{s'=1}^{s-1}\lambda_{d}\left(\left(\Sigma_0 + (\sum_{t'=1}^{n}\tfrac{A_{s',t'}A^T_{s',t'}}{\sigma^2})^{-1}\right)^{-1}\right)\right)^{-1}\\
&\leq\left(\lambda_{d}(\Sigma_q^{-1}) + \sum_{s'=1}^{s-1}\left(\lambda_{1}(\Sigma_0) + \lambda_{1}\left((\sum_{t'=1}^{n}\tfrac{A_{s',t'}A^T_{s',t'}}{\sigma^2})^{-1}\right)\right)^{-1}\right)^{-1}\\
&\leq \left(\lambda_{1}^{-1}(\Sigma_q) + s (\lambda_{1}(\Sigma_0) + \sigma^2/\eta)^{-1} \right)^{-1}
\end{align*}
In the above derivation, we use the Weyl's inequalities multiple times. Note the direction of inequality should be $\leq$ if there are even number of inverses, whereas it should be $\geq$ if there are an odd number of inverses associated. The first inequality uses the inequality $\lambda_{d}(\sum_i A_i) \geq \sum_i\lambda_{d}(A_i)$ given all the matrices $A_i$-s are Hermitian. The second inequality similarly uses $\lambda_{1}(\sum_i A_i) \leq \sum_i\lambda_{1}(A_i)$ given all the matrices $A_i$-s are Hermitian. The final inequality uses the minimum eigenvalue bound when forced exploration is used.

This concludes the second part of the proof, in particular 
$$
\sigma^2_{\max}(\hat{\Sigma}_{s,t}) \leq \lambda_1(\Sigma_0 + \hat{\Sigma}_s) \leq \lambda_1(\Sigma_{0})\left(1+\tfrac{\lambda_{1}(\Sigma_q)(1+  \tfrac{\sigma^2/\eta }{\lambda_{1}(\Sigma_0)})}{\lambda_{1}(\Sigma_0) + \sigma^2/ \eta  +  s\lambda_{1}(\Sigma_q)}\right).
$$

\end{proof}

\subsection{Proof of Theorem~\ref{thm:linearinfo}}
\label{sec:proofofthmlininfo}
\linearinfo*
\begin{proof}
We note that, for each $s$, we can bound w.p. $1$
$$\Gamma_{s,t} \leq 4\sqrt{\frac{\lambda_1(\Sigma_{0})\left(1+\tfrac{\lambda_{1}(\Sigma_q)(1+  \tfrac{\sigma^2 / \eta}{\lambda_{1}(\Sigma_0)})}{\lambda_{1}(\Sigma_0) + \sigma^2 / \eta +  s\lambda_{1}(\Sigma_q)}\right)}{\log\left(1+\tfrac{\lambda_1(\Sigma_0)}{\sigma^2}\left(1+\tfrac{\lambda_{1}(\Sigma_q)(1+  \tfrac{\sigma^2 / \eta}{\lambda_{1}(\Sigma_0)})}{\lambda_{1}(\Sigma_0) + \sigma^2 / \eta +  s\lambda_{1}(\Sigma_q)}\right)\right)}\log(4|\mathcal{A}|/\delta)}.$$
This is true by using the upper bounds on $\sigma_{max}^2(\hat{\Sigma}_{s,t})$ in Lemma~\ref{lemm:concentration}, and  because the function $\sqrt{x/\log(1+ax)}$ for $a>0$ increases with $x$. Similarly, we have  
$$\epsilon_{s,t} \leq \sqrt{\delta\lambda_1(\Sigma_{0})\left(1+\tfrac{\lambda_{1}(\Sigma_q)(1+  \tfrac{\sigma^2 / \eta}{\lambda_{1}(\Sigma_0)})}{\lambda_{1}(\Sigma_0) + \sigma^2 / \eta +  s\lambda_{1}(\Sigma_q)}\right)}.$$
Therefore, we have the bounds $\Gamma_{s,t} \leq \Gamma_s$ w.p. $1$ for all $s$ and $t$ by using appropriate $s$, and by setting $s=0$ we obtain $\Gamma$.

We are now at a position to provide the final regret bound. For any  $\delta > 0$
\begingroup
\allowdisplaybreaks
\begin{align*}
  &R(m, n) \leq  \Gamma \sqrt{ mnI(\mu_{*}; H_{m})} +  \mathbb{E}\sum_{s}\Gamma_{s}\sqrt{n I(\theta_{s,*}; H_{s}\mid \mu_{*}, H_{s-1})} + \mathbb{E}\sum_{s,t}\epsilon_{s,t}\\
  &\leq 4\underbrace{\sqrt{\frac{\lambda_1(\Sigma_{q}) + \lambda_1(\Sigma_{0})}{\log(1+(\lambda_1(\Sigma_q)+\lambda_1(\Sigma_{0}))/\sigma^2 )}\log(4\mid \mathcal{A}\mid /\delta)} 
  \sqrt{mn \tfrac{d}{2}\log\left(1+  \frac{mn \lambda_1(\Sigma_q)}{n\lambda_{d}(\Sigma_0) + \sigma^2}\right)}}_{\text{regret for learning $\mu$}}\\
  &+ \sum_{s=1}^{m} 4\sqrt{\frac{\lambda_1(\Sigma_{0}){\color{blue}\left(1+\tfrac{\lambda_{1}(\Sigma_q)(1+  \tfrac{\sigma^2 / \eta}{\lambda_{1}(\Sigma_0)})}{\lambda_{1}(\Sigma_0) + \sigma^2 / \eta +  s\lambda_{1}(\Sigma_q)}\right)}}{\log\left(1+\tfrac{\lambda_1(\Sigma_0)}{\sigma^2}{\color{blue}\left(1+\tfrac{\lambda_{1}(\Sigma_q)(1+  \tfrac{\sigma^2 / \eta}{\lambda_{1}(\Sigma_0)})}{\lambda_{1}(\Sigma_0) + \sigma^2 / \eta +  s\lambda_{1}(\Sigma_q)}\right)}\right)}\log(4|\mathcal{A}|/\delta)} \sqrt{n \tfrac{d}{2}\log\left(1 + n\frac{\lambda_1(\Sigma_0)}{\sigma^2}\right) } \\
  &+ \sum_{s=1}^{m} n \sqrt{2\delta\lambda_1(\Sigma_{0}){\color{blue} \left(1+\tfrac{\lambda_{1}(\Sigma_q)(1+  \tfrac{\sigma^2 / \eta}{\lambda_{1}(\Sigma_0)})}{\lambda_{1}(\Sigma_0) + \sigma^2 / \eta +  s\lambda_{1}(\Sigma_q)}\right)}}  \\
  & + 2 \left(\sum_{s=1}^{m}\sum_{t=1}^{n} \mathcal{E}_{s,t}\right)\mathbb{E}[\|\theta_{s,*}\|_2]\\
  &\leq 4\underbrace{\sqrt{\frac{\lambda_1(\Sigma_{q}) + \lambda_1(\Sigma_{0})}{\log(1+(\lambda_1(\Sigma_q)+\lambda_1(\Sigma_{0}))/\sigma^2 )}\log(4|\mathcal{A}|/\delta)} 
  \sqrt{mn \tfrac{d}{2}\log\left(1+  \frac{mn \lambda_1(\Sigma_q)}{n\lambda_{d}(\Sigma_0) + \sigma^2}\right)}}_{\text{regret for learning $\mu$}}\\
  &+  \left(m+ {\color{blue} \tfrac{1}{2 \lambda_{1}(\Sigma_0)} \sum_{s=1}^{m} \tfrac{\lambda_{1}(\Sigma_q)( \lambda_{1}(\Sigma_0)+  \sigma^2/ \eta )}{\lambda_{1}(\Sigma_0) + \sigma^2/\eta  +  s\lambda_{1}(\Sigma_q)}}\right) \times \\  
  &\left(4 \sqrt{\frac{\lambda_1(\Sigma_{0})}{\log\left(1+\tfrac{\lambda_1(\Sigma_0)}{\sigma^2}\right)}\log(4|\mathcal{A}|/\delta)} \sqrt{n \tfrac{d}{2}\log\left(1 + n\frac{\lambda_1(\Sigma_0)}{\sigma^2}\right) } 
  +  n \sqrt{2\lambda_1(\Sigma_0) \delta} \right) \\
  & + 2 md\sqrt{\|\mu_{q}\|^2_2 + \mathrm{tr}(\Sigma_q + \Sigma_0)}\\
  &\leq 4\underbrace{\sqrt{\frac{\lambda_1(\Sigma_{q}) + \lambda_1(\Sigma_{0})}{\log(1+(\lambda_1(\Sigma_q)+\lambda_1(\Sigma_{0}))/\sigma^2 )}\log(4|\mathcal{A}|/\delta)} 
  \sqrt{mn \tfrac{d}{2}\log\left(1+  \frac{mn \lambda_1(\Sigma_q)}{n\lambda_{d}(\Sigma_0) + \sigma^2}\right)}}_{\text{regret for learning $\mu$}}\\
  &+  \left(m+ {\color{blue}  (1+  \tfrac{\sigma^2/\eta }{\lambda_{1}(\Sigma_0)})\log(m)}\right)\times \\
  & \left( 4\sqrt{\frac{\lambda_1(\Sigma_{0})}{\log\left(1+\tfrac{\lambda_1(\Sigma_0)}{\sigma^2}\right)}\log(4|\mathcal{A}|/\delta)} \sqrt{n \tfrac{d}{2}\log\left(1 + n\frac{\lambda_1(\Sigma_0)}{\sigma^2}\right) } 
  +  n \sqrt{2\lambda_1(\Sigma_0) \delta} \right) \\
  & + 2 md \sqrt{\|\mu_{q}\|^2_2 + \mathrm{tr}(\Sigma_q + \Sigma_0)} 
\end{align*}
\endgroup
The first inequality follows by substituting the appropriate bounds. The second inequality first removes the part highlighted in blue (which is positive) inside the logarithm, and then uses the fact that $\sqrt{1+x}\leq 1+ x/2$ for all $x\geq 1$. We also use $\mathbb{E}[\|\theta_{s,*}\|_2] = \sqrt{\|\mu_{q}\|^2_2 + \mathrm{tr}(\Sigma_q + \Sigma_0)}$ and the fact that \adats explores for $d$ rounds in each task. The final inequality replaces the summation by an integral over $s$ and derives the closed form.
\end{proof}

\clearpage

\section{Proofs for \cref{sec:mutual semi-bandit}: Semi-Bandit}
\label{sec:appendix-semi-bandit}

In this section, we expand the linear bandit analysis to handle multiple inputs as is common in semi-bandit feedback in combinatorial optimizations. Furthermore, as the rewards for each base-arm are  independent for each arm we can improve our analysis providing tighter prior dependent bounds. The center piece of the proof is again the mutual information separation between the meta-parameter and the parameter in each stage. However, in the regret decomposition we sum the confidence intervals of different arms separately. 

 \paragraph{Notations:}
We recall the necessary notations for the proof of regret upper bound in the semi-bandit setting. For each arm $k$ the meta-parameter $\mu_{*,k} \sim \mathcal{N}(\mu_{q,k}, \sigma^2_{q,k})$. The mean reward at the beginning for each task $s$, for an arm $k$ is sampled from $\mathcal{N}(\mu_{*,k}, \sigma^2_{0,k})$.
 The reward realization of arm $k$ in round $t$ and task $s$ is denoted by $Y_{s,t}(k) = \theta_{s,*}(k) + w_{s,t}(k)$ where  be the reward of the arm $k$ at time $t$ (arm $k$ need not be played during time $t$). Then the reward obtained for the action $a$ (a subset of $[K]$ with size at most $L$) is given as $Y_{s,t}(a) = \sum_{k\in a} Y_{s,t}(k)$.
Let, for each task $s$ and round $t$,  the action (a subset of $[K]$) be $A_{s,t}$, and  the observed reward vector be $Y_{s,t} = (Y_{s,t}(k): k \in A_{s,t})$. 

The linear bandits notations for history, conditional probability, and conditional expectation carry forward to semi-bandits. Additionally, let us denote the number of pulls for arm $k$, in phase $s$, upto and excluding round $t$ as $N_{s,t}(k)$. The total number of pulls for arm $k$ in task $s$ is denoted as $N_{s}(k) = N_{s,n+1}(k)$, and up to and including task $s$ is denoted by $N_{1:s}(t)$.

\paragraph{Mutual Information in Semi-bandits:}
The history dependent and independent mutual information terms are defined analogously, but we are now interested in the terms for each arms separately. 
For any arm $k\in [K]$ and action $a\subseteq [K]$, the history dependent mutual information terms of interest are  
\begin{align*}
    & I_{s,t}(\theta_{s,*}(k); A_{s,t}, Y_{s,t}\mid \mu_{*,k}) = \mathbb{E}_{s,t}\left[\frac{\mathbb{P}_{s,t}(\theta_{s,*}(k), A_{s,t}, Y_{s,t}\mid \mu_{*,k})}{\mathbb{P}_{s,t}(\theta_{s,*}(k)\mid \mu_{*,k}) \mathbb{P}_{s,t}(A_{s,t}, Y_{s,t}\mid \mu_{*,k})}\right]\\
    & I_{s,t}(\theta_{s,*}(k); a, Y_{s,t}(a)\mid \mu_{*,k}) = \mathbb{E}_{s,t}\left[\frac{\mathbb{P}_{s,t}(\theta_{s,*}(k),  Y_{s,t}(k)\mid \mu_{*,k}, A_{s,t}=a)}{\mathbb{P}_{s,t}(\theta_{s,*}(k)\mid \mu_{*,k}, A_{s,t}=a) \mathbb{P}_{s,t}( Y_{s,t}(a)\mid \mu_{*,k}, A_{s,t}=a)}\right]\\
    & I_{s,t}(\mu_{*,k}; A_{s,t}, Y_{s,t}) = \mathbb{E}_{s,t}\left[\frac{\mathbb{P}_{s,t}(\mu_{*,k}, A_{s,t}, Y_{s,t})}{\mathbb{P}_{s,t}(\mu_{*,k}) \mathbb{P}_{s,t}(A_{s,t}, Y_{s,t})}\right]\\
    & I_{s,t}(\mu_{*,k}; a, Y_{s,t}(a)) = \mathbb{E}_{s,t}\left[\frac{\mathbb{P}_{s,t}(\mu_{*,k},  Y_{s,t}(k) \mid A_{s,t}=a)}{\mathbb{P}_{s,t}(\theta_{s,*}(k) \mid  A_{s,t}\mathtt{=}a) \mathbb{P}_{s,t}( Y_{s,t}(a) \mid A_{s,t}\mathtt{=}a)}\right]
\end{align*}
The history mutual information independent terms of interest are 
\begin{gather*}
  I(\theta_{s,*}(k); A_{s,t}, Y_{s,t}\mid \mu_{*,k}, H_{1:s, t})  =  \mathbb{E}[I_{s,t}(\theta_{s,*}(k); A_{s,t}, Y_{s,t}\mid \mu_{*,k})], \\
  I(\mu_{*,k}; A_{s,t}, Y_{s,t}\mid H_{1:s, t}) = \mathbb{E}[I_{s,t}(\mu_{*,k}; A_{s,t}, Y_{s,t})].
\end{gather*}

We now derive the mutual information of $\theta_{s,*}(k)$ and events in task $s$, i.e. $H_s$, given $\mu_{*,k}$, and history upto and excluding task $s$, i.e. $H_{1:s-1}$.
\begin{lemma}\label{lemm:chainrule semibandit}
For any $k\in [K]$, $s\in [m]$, and $H_{1:s,t}$ adapted sequence of actions $((A_{s,t})_{t=1}^{n})_{s=1}^{m}$, the following statements hold for a $(K,L)$-Semi-bandit
\begin{gather*}
    I(\theta_{s,*}(k); H_s\mid \mu_{*,k}, H_{1:s-1}) = \mathbb{E}\sum_{t}\mathbb{P}_{s,t}(k\in A_{s,t}) I_{s,t}(\theta_{s,*}(k); k, Y_{s,t}(k)\mid\mu_{*,k})\, \\
    I(\mu_{*,k} \mid H_{1:m}) = \mathbb{E}\sum_{s}\sum_{t}\mathbb{P}_{s,t}(k\in A_{s,t}) I_{s,t}(\mu_{*,k}; k, Y_{s,t}(k)).
\end{gather*}
\end{lemma}
\begin{proof}
The proof follows by the application of the chain rule of mutual information, and noticing that the rounds when an arm $k$ was not played the mutual information $I_{s,t}(\theta_{s,*}(k); k, Y_{s,t}(k)\mid \mu_{*,k})$ and $I_{s,t}(\mu_{*,k}; k, Y_{s,t}(k))$ both are zero. This is true because no information is gained about the parameters $\theta_{s,*}(k)$ and $\mu_{*,k}$ in those rounds.

\begin{align*}
 &I(\theta_{s,*}(k); H_s\mid \mu_{*,k}, H_{1:s-1})\\
&=  \mathbb{E}\sum_{t} I(\theta_{s,*}(k); A_{s,t}, Y_{s,t}\mid\mu_{*,k}, H_{1:s-1}, H_{s,t-1} )\\
&= \mathbb{E}\sum_{t} I_{s,t}(\theta_{s,*}(k); A_{s,t}, Y_{s,t}\mid\mu_{*,k})\\
&= \mathbb{E}\sum_{t} \sum_{a \in \mathcal{A}} \mathbb{P}_{s,t}(A_{s,t} = a)  I_{s,t}(\theta_{s,*}(k); a, Y_{s,t}(a)\mid\mu_{*,k})\\
&= \mathbb{E}\sum_{t} \sum_{a \in \mathcal{A}} \mathbb{P}_{s,t}(A_{s,t} = a) \mathbbm{1}(k\in a) I_{s,t}(\theta_{s,*}(k); k, Y_{s,t}(k)\mid\mu_{*,k})  \\
&+\mathbb{E}\sum_{t} \sum_{a \in \mathcal{A}} \mathbb{P}_{s,t}(A_{s,t} = a)  I_{s,t}(\theta_{s,*}(k); a\setminus k, Y_{s,t}(a\setminus k)\mid\mu_{*,k}, (k,Y_{s,t}(k)))\\
&= \mathbb{E}\sum_{t}\mathbb{P}_{s,t}(k\in A_{s,t}) I_{s,t}(\theta_{s,*}(k); k, Y_{s,t}(k)\mid\mu_{*,k}) 
\end{align*}
Here, $a\setminus k$ implies the action with arm $k$ removed from subset $a$. Due to the independence of the reward of each arm, for any fixed action $a$, $\theta_{s,*}(k) \perp (a\setminus k, Y_{s,t}(a\setminus k))$ conditioned on $\mu_{*,k}$, $(k,Y_{s,t}(k))$, and history $H_{1:s,t}$. Therefore,  we have 
$$
I_{s,t}(\theta_{s,*}(k); a\setminus k, Y_{s,t}(a\setminus k)\mid\mu_{*,k}, (k,Y_{s,t}(k))) = 0.
$$

A similar sequence of steps lead to 
$$
I(\mu_{*,k} \mid H_{1:m}) = \mathbb{E}\sum_{s}\sum_{t}\mathbb{P}_{s,t}(k\in A_{s,t}) I_{s,t}(\mu_{*,k}; k, Y_{s,t}(k)). 
$$
The above equalities develop the chain rules of mutual information for each of the arms separately, by leveraging the independence of the rewards per arms. 
\end{proof}

\paragraph{Per Task Regret Bound:} We derive the per task regret using the information theoretical confidence intervals while accounting for each arm separately. Let the posterior distribution of $\theta_{s,*}(k)$ at the beginning of round $t$ of task $s$ be $\cN(\hat{\mu}_{s,t}(k), \hat{\sigma}^2_{s,t}(k))$  for appropriate $\hat{\mu}_{s,t}(k)$ and $\hat{\sigma}^2_{s,t}(k)$ that depends on the history $H_{1:s,t}$, for all $k\in [K]$, $s\in [m]$, and $t\in [n]$. We will derive these terms or bounds on these terms later.   

\begin{lemma}\label{lemm:concentration semibandit} 
For an $H_{1:s,t}$ adapted sequence of actions $((A_{s,t})_{t=1}^{n})_{s=1}^{m}$, and any $\delta \in (0,1]$, the expected regret in round $t$ of stage $s$ in a $(K,L)$-Semi-bandit is bounded as  
\begin{gather}\label{a:info-semibandit}
    \mathbb{E}_{s,t}[\Delta_{s,t}] = \sum_{k\in [K]}\mathbb{P}_{s,t}(k \in A_{s,t})\left(\Gamma_{s,t}(k)\sqrt{I_{s,t}(\theta_{s,*}(k); k, Y_{s,t}(k))} + \sqrt{2\delta\tfrac{1}{K}\sigma_{s,t}^2(k)}\right),
\end{gather}
where $$\Gamma_{s,t}(k) =  4\sqrt{\frac{\hat{\sigma}^2_{s,t-1}(k)}{\log(1+\hat{\sigma}^2_{s,t-1}(k)/\sigma^2)}\log(\tfrac{4K}{\delta}}). $$
\end{lemma}
\begin{proof}
Similar to linear bandits we have without forced exploration
\begin{align*}
    \mathbb{E}_{s,t}[\Delta_{s,t}] &= \mathbb{E}_{s,t}[\sum_{k\in A_{s,*}}\theta_{s,*}(k) - \sum_{k\in A_{s,t}}\theta_{s,*}(k)]\\
    &= \mathbb{E}_{s,t}[\sum_{k\in A_{s,t}}\hat{\theta}_{s,t}(k) - \sum_{k\in A_{s,t}}\theta_{s,*}(k)]\\
    &= \mathbb{E}_{s,t}[\sum_{a\in \mathcal{A}}\mathbbm{1}(A_{s,t} = a)\sum_{k\in a}(\hat{\theta}_{s,t}(k) - \theta_{s,*}(k))]\\
    &= \mathbb{E}_{s,t}[\sum_{k\in [K]}\mathbbm{1}(k \in A_{s,t})(\hat{\theta}_{s,t}(k) - \theta_{s,*}(k))]\\
    &= \sum_{k\in [K]}\mathbb{P}_{s,t}(k \in A_{s,t})\mathbb{E}_{s,t}[\hat{\theta}_{s,t}(k) - \theta_{s,*}(k)].
\end{align*}
The second equality is due to Thompson sampling  ($\stackrel{d}{=}$ denotes equal distribution)
$$
\sum_{k\in A_{s,*}}\theta_{s,*}(k)\mid H_{1:s,t} \stackrel{d}{=} \sum_{k\in A_{s,t}}\hat{\theta}_{s,t}(k)\mid H_{1:s,t}.
$$

When forced exploration is used in some task $s$ and round $t$ we have, $\theta_{s,*}(k)\mid H_{1:s,t} \stackrel{d}{=} \hat{\theta}_{s,t}(k)\mid H_{1:s,t}$
\begin{align*}
    \mathbb{E}_{s,t}[\Delta_{s,t}] &= \mathbb{E}_{s,t}[\sum_{k\in A_{s,*}}\theta_{s,*}(k) - \sum_{k\in A_{s,t}}\theta_{s,*}(k)]\\
    &= \mathbb{E}_{s,t}[\sum_{k\in A_{s,t}}\hat{\theta}_{s,t}(k) - \sum_{k\in A_{s,t}}\theta_{s,*}(k)]
    + \mathbb{E}_{s,t}[\sum_{k\in A_{s,*}}\theta_{s,*}(k) - \sum_{k\in A_{s,t}}\hat{\theta}_{s,t}(k)]\\
    &= \mathbb{E}_{s,t}[\sum_{k\in A_{s,t}}\hat{\theta}_{s,t}(k) - \sum_{k\in A_{s,t}}\theta_{s,*}(k)]
    + \mathbb{E}_{s,t}[\sum_{k\in A_{s,*}}\theta_{s,*}(k) - \sum_{k\in A_{s,t}}\hat{\theta}_{s,t}(k)]\\
    &\leq \sum_{k\in [K]}\mathbb{P}_{s,t}(k \in A_{s,t})\mathbb{E}_{s,t}[\hat{\theta}_{s,t}(k) - \theta_{s,*}(k)] + 
   2\sqrt{K} \mathbb{E}_{s,t}[\sqrt{\sum_{k\in K} \theta^2_{s,*}(k)}]
\end{align*}

For each $k\in [K]$, for appropriate $\hat{\mu}_{s,t}(k)$ and $\hat{\sigma}^2_{s,t}(k)$ we know that $\hat{\theta}_{s,t}(k)\mid H_{1:s,t} \sim \mathcal{N}(\hat{\mu}_{s,t}(k), \hat{\sigma}^2_{s,t}(k))$. 
We define the confidence set for each arm $k$ at round $t$ of task $s$, for some $\Gamma_{s,t}(k)$, which can be a function of $H_{1:s,t}$,  to be specified late, as
$$
\Theta_{s,t}(k) = \{\theta: \mid\theta - \hat{\mu}_{s,t}(k)\mid \leq \tfrac{\Gamma_{s,t}(k)}{2}\sqrt{I_{s,t}(\theta_{s,*}(k); k, Y_{s,t}(k))}\}.
$$
A derivation equivalent to linear bandits, gives us 
$$I_{s,t}(\theta_{s,*}(k); k, Y_{s,t}(k)) = \tfrac{1}{2}\log\left(1+ \tfrac{\hat{\sigma}^2_{s,t-1}(k)}{\sigma^2}\right).$$
Because, we only consider arm $k$ we obtain as a corollary of Lemma $5$ in Lu et al.~\cite{lu2019information} that for any $k$, and any $\delta\tfrac{1}{K} >0$ for 
$$\Gamma_{s,t}(k) =  4\sqrt{\frac{\hat{\sigma}^2_{s,t-1}(k)}{\log(1+\hat{\sigma}^2_{s,t-1}(k)/\sigma^2)}\log(\tfrac{4K}{\delta}}). $$ 
we have 
$\mathbb{P}_{s,t}(\hat{\theta}_{s,t}(k) \in \Theta_{s,t}(k)) \geq 1 - \delta/2K$. 

We proceed with the regret bound as
\begin{align*}
    &\mathbb{E}_{s,t}[\hat{\theta}_{s,t}(k) - \theta_{s,*}(k)] \\
    &\leq \mathbb{E}_{s,t}[\mathbbm{1}(\hat{\theta}_{s,t}(k), \theta_{s,*}(k) \in \Theta_{s,t}(k)) (\hat{\theta}_{s,t}(k) - \theta_{s,*}(k))] \\
    & +  \mathbb{E}_{s,t}[\mathbbm{1}^c(\hat{\theta}_{s,t}(k), \theta_{s,*}(k) \in \Theta_{s,t}(k)) (\hat{\theta}_{s,t}(k) - \theta_{s,*}(k))]\\
    &\leq \Gamma_{s,t}(k)\sqrt{I_{s,t}(\theta_{s,*}(k); k, Y_{s,t}(k))} + \sqrt{\mathbb{P}(\hat{\theta}_{s,t}(k) \text{ or } \theta_{s,*}(k) \notin \Theta_{s,t}(k))\mathbb{E}_{s,t}[(\hat{\theta}_{s,t}(k) - \theta_{s,*}(k))^2]}\\
    &\leq \Gamma_{s,t}(k)\sqrt{I_{s,t}(\theta_{s,*}(k); k, Y_{s,t}(k))} + \sqrt{\delta\tfrac{1}{K}\mathbb{E}_{s,t}[(\hat{\theta}_{s,t}(k) - \hat{\mu}_{s,t}(k))^2 + (\theta_{s,*}(k)- \hat{\mu}_{s,t}(k))^2]}\\
    &\leq \Gamma_{s,t}(k)\sqrt{I_{s,t}(\theta_{s,*}(k); k, Y_{s,t}(k))} + \sqrt{2\delta\tfrac{1}{K}\sigma_{s,t}^2(k)}
\end{align*}
This concludes the proof.
\end{proof}

\paragraph{Regret Decomposition:}
We now develop the regret decomposition for the $(K,L)$-Semi-bandit based on the per step regret characterization in \cref{lemm:concentration semibandit}.   
\begin{lemma}
\label{lemm:decomp semibandit} Let, for each $k \in [K]$, $(\Gamma_s(k))_{s\in [m]}$ and $\Gamma(k)$ be non-negative constants such that $\Gamma_{s,t}(k) \leq \Gamma_s(k) \leq \Gamma(k)$   holds for all $s\in [m]$ and $t\in [n]$ almost surely. Then for any $\delta \in (0,1]$ the regret of \adats admits the upper bound
\begin{align*}
  R(m, n)
  \leq & \sqrt{m n K L}\sqrt{\tfrac{1}{K}\sum_{k\in [K]}\Gamma^2(k)I(\mu_{*,k}; H_{1:m})} + 
  2 \sqrt{m K \sum_{k\in K}\left(\mu^2_{q}(k) + \sigma^2_{0,k}+\sigma^2_{q,k}\right)} \\
  & + \sum_{s=1}^{m}\left( \sqrt{n K L} \sqrt{\tfrac{1}{K}\sum_{k\in [K]}\Gamma^2_{s}(k)I(\theta_{s,*}(k); H_s\mid \mu_{*,k}, H_{1:s-1})} + n\sqrt{2\delta\tfrac{1}{K} \sum_{k\in [K]} \hat{\sigma}_{s}^2(k)}\right)\,.
\end{align*}
\end{lemma}
\begin{proof}
The regret decomposition is computed in the following steps. Recall that $\mathcal{E}_{s,t}$ is the indicator if in round $t$ of task $s$ we use exploration.  
\begingroup
\allowdisplaybreaks
\begin{align*}
  &R(m, n) = \mathbb{E}\left[\sum_{s,t}\Delta_{s,t}\right]\\
  &\leq  \mathbb{E}\left[\sum_{s,t}\sum_{k\in [K]}\mathbb{P}_{s,t}(k\in A_{s,t}) \left(\Gamma_{s,t}(k)\sqrt{I_{s,t}(\theta_{s,*}(k); k, Y_{s,t}(k))}\right)\right]\\
  &+ \mathbb{E}\left[\sum_{s,t}\sum_{k\in [K]}\mathbb{P}_{s,t}(k\in A_{s,t})\epsilon_{s,t}(k)\right]
  + 2\mathbb{E}\left[\sum_{s,t} \mathcal{E}_{s,t}\sqrt{K} \mathbb{E}_{s,t}[\sqrt{\sum_{k\in K} \theta^2_{s,*}(k)}] \right]\\
  &\leq  \mathbb{E}\left[\sum_{s,t}\sum_{k\in [K]}\mathbb{P}_{s,t}(k\in A_{s,t}) \left(\Gamma_{s,t}(k)\sqrt{I_{s,t}(\theta_{s,*}(k), \mu_{*,k}; k, Y_{s,t}(k))}\right)\right]\\
  &+ \mathbb{E}\left[\sum_{s,t}\sum_{k\in [K]}\mathbb{P}_{s,t}(k\in A_{s,t}) \sqrt{2\delta\tfrac{1}{K}\sigma_{s,t}^2(k)}\right]
  + 2 mK^{3/2} \mathbb{E}[\sqrt{\sum_{k\in K} \theta^2_{s,*}(k)}]\\
  &= \mathbb{E}\left[\sum_{s,t}\sum_{k\in [K]}\mathbb{P}_{s,t}(k\in A_{s,t}) \left(\Gamma_{s,t}(k)\sqrt{I_{s,t}(\theta_{s,*}(k); k, Y_{s,t}(k)\mid \mu_{*,k}) + I_{s,t}(\mu_{*,k}; k, Y_{s,t}(k)) }\right)\right]\\
  &+ \mathbb{E}\left[\sum_{s,t}\sum_{k\in [K]}\mathbb{P}_{s,t}(k\in A_{s,t}) \sqrt{2\delta\tfrac{1}{K}\sigma_{s,t}^2(k)}\right]
  + 2 mK^{3/2}\sqrt{ \sum_{k\in K}\left(\mu^2_{q}(k) + \sigma^2_{0,k}+\sigma^2_{q,k}\right)}\\
  &\leq \mathbb{E}\left[\sum_{s,t}\sum_{k\in [K]}\mathbb{P}_{s,t}(k\in A_{s,t}) \Gamma_{s,t}(k)\left(\sqrt{I_{s,t}(\theta_{s,*}(k); k, Y_{s,t}(k)\mid \mu_{*,k})} + \sqrt{I_{s,t}(\mu_{*,k}; k, Y_{s,t}(k)) }\right)\right]\\
  &+ \mathbb{E}\left[\sum_{s,t}\sum_{k\in [K]}\mathbb{P}_{s,t}(k\in A_{s,t}) \sqrt{2\delta\tfrac{1}{K}\sigma_{s,t}^2(k)}\right]
  + 2 mK^{3/2}\sqrt{\sum_{k\in K}\left(\mu^2_{q}(k) + \sigma^2_{0,k}+\sigma^2_{q,k}\right)}\\
  &\leq \Gamma_{s}(k)\sum_{s}\sum_{k\in [K]}\mathbb{E}\left[ \sum_{t}\mathbb{P}_{s,t}(k\in A_{s,t}) \sqrt{I_{s,t}(\theta_{s,*}(k); k, Y_{s,t}(k)\mid \mu_{*,k})}\right]\\
  &+ \sum_{k\in [K]}\Gamma(k)\mathbb{E}\left[\sum_{s,t}\mathbb{P}_{s,t}(k\in A_{s,t})\sqrt{I_{s,t}(\mu_{*,k}; k, Y_{s,t}(k)) }\right]\\
  &+ \sum_{s}\sum_{k\in [K]}\sqrt{2\delta\tfrac{1}{K}\hat{\sigma}_{s}^2(k)} \mathbb{E}[\sum_{t} \mathbb{P}_{s,t}(k\in A_{s,t})]
  + 2 mK^{3/2}\sqrt{\sum_{k\in K}\left(\mu^2_{q}(k) + \sigma^2_{0,k}+\sigma^2_{q,k}\right)}
\end{align*}
\endgroup
The first inequality follows from the expression for the reward gaps in Equation~\ref{a:info-semibandit}. The next two equations follow due to the chain rule of mutual information, similar to the linear bandit case. The only difference in this case we use the parameters for each arm ($\theta_{s,*}(k)$ and $\mu_{*,k}$) separately. Also, we use the fact that there are at most $m K$ rounds where forced exploration is used for the $m$ tasks.  The next inequality is due to $\sqrt{a+b}\leq \sqrt{a}+ \sqrt{b}$. The final inequality follows as $\Gamma_{s,t}(k)\leq \Gamma_s(k)\leq \Gamma(k)$, and $\sigma_{s,t}(k) \leq \sigma_{s}(k)$ w.p. $1$ for all $k\in [K]$, and $s\leq m$ and $t\leq n$.

We now derive the bounds for the sum of the mutual information terms for the per task parameters given the knowledge of the meta-parameter.
\begingroup
\allowdisplaybreaks
\begin{align*}
    &\sum_{s}\sum_{k\in [K]}\Gamma_{s}(k)\mathbb{E}\left[\sum_{t}\mathbb{P}_{s,t}(k\in A_{s,t})\sqrt{I_{s,t}(\theta_{s,*}(k); k, Y_{s,t}(k)\mid \mu_{*,k})}\right]\\
 &=\sum_{s}\sum_{k\in [K]}\Gamma_{s}(k)\mathbb{E}\left[\sum_{t}\sqrt{\mathbb{P}_{s,t}(k\in A_{s,t})} \sqrt{\mathbb{P}_{s,t}(k\in A_{s,t}) I_{s,t}(\theta_{s,*}(k); k, Y_{s,t}(k)\mid \mu_{*,k})}\right]\\
  &\leq \sum_{s}\sum_{k\in [K]}\Gamma_{s}(k)\mathbb{E}\left[\sqrt{\sum_{t}\mathbb{P}_{s,t}(k\in A_{s,t})} \sqrt{\sum_{t} \mathbb{P}_{s,t}(k\in A_{s,t}) I_{s,t}(\theta_{s,*}(k); k, Y_{s,t}(k)\mid \mu_{*,k})}\right]\\
  &\leq \sum_{s}\sum_{k\in [K]}\Gamma_{s}(k)\sqrt{\mathbb{E}\sum_{t}\mathbb{P}_{s,t}(k\in A_{s,t})} \sqrt{\mathbb{E}\sum_{t} \mathbb{P}_{s,t}(k\in A_{s,t}) I_{s,t}(\theta_{s,*}(k); k, Y_{s,t}(k)\mid \mu_{*,k})}\\
  &= \sum_{s}\sum_{k\in [K]}\Gamma_{s}(k)\sqrt{\mathbb{E}N_s(k)} \sqrt{I(\theta_{s,*}(k); H_s\mid \mu_{*,k}, H_{1:s-1})}\\
  &\leq \sum_{s}\sqrt{K\sum_{k\in [K]}\mathbb{E}N_s(k)} \sqrt{\tfrac{1}{K}\sum_{k\in [K]}\Gamma^2_{s}(k)I(\theta_{s,*}(k); H_s\mid \mu_{*,k}, H_{1:s-1})}\\
  & = \sum_{s}\sqrt{n K L} \sqrt{\tfrac{1}{K}\sum_{k\in [K]}\Gamma^2_{s}(k)I(\theta_{s,*}(k); H_s\mid \mu_{*,k}, H_{1:s-1})}
\end{align*}
\endgroup
The first equality is easy to see. Next sequence of inequalities follow mainly by repeated application of Cauchy-Schwarz in different forms, and application of chain rule of mutual information. We now describe the other ones.
\begin{itemize}
    \item[-] The second equation follow as $\sum_i a_ib_i \leq \sqrt{\sum_i a^2_i \sum_i b^2_i}$ for $a_i, b_i\geq 0$, with $a_i = \sqrt{\mathbb{P}_{s,t}(k\in A_{s,t})}$ and $b_i =  \sqrt{\mathbb{P}_{s,t}(k\in A_{s,t}) I_{s,t}(\theta_{s,*}(k); k, Y_{s,t}(k)\mid \mu_{*,k})}$.
    \item[-] The third equation uses $\mathbb{E}[XY] \leq \sqrt{\mathbb{E}[X^2]\mathbb{E}[Y^2]}$ for  $X,Y > 0$ w.p. $1$ (positive random variables).
    \item[-]The next equality first uses the relation $\mathbb{E}\sum_{t} \mathbb{P}_{s,t}(k\in A_{s,t}) = \mathbb{E}[N_s(k)]$ where $N_s(k)$ is the number of time arm $k$ is played in the task $s$. Then it also use the chain rule for $I(\theta_{s,*}(k); H_s\mid \mu_{*,k}, H_{1:s-1})$.
    \item[-] For the next  inequality, we apply Cauchy-Schwarz ($\sum_i a_ib_i \leq \sqrt{\sum_i a^2_i \sum_i b^2_i}$) again as $a_i = \sqrt{\mathbb{E}[N_s(k)]}$ and $b_i = I(\theta_{s,*}(k); H_s\mid \mu_{*,k}, H_{1:s-1})$.  Also note that $K$ and $\frac{1}{K}$ cancels out. 
    \item[-] The final inequality is attained by noticing $\mathbb{E}[\sum_{k}N_s(k)] \leq nL$, as at most $L$ arms can be played in each round. 
\end{itemize}

The sum of the mutual information terms pertaining to the meta-parameter can be derived equivalently. 
\begingroup
\allowdisplaybreaks
\begin{align*}
  &\sum_{k\in [K]}\Gamma(k)\mathbb{E}\left[\sum_{s,t}\mathbb{P}_{s,t}(k\in A_{s,t})\sqrt{I_{s,t}(\mu_{*,k}; k, Y_{s,t}(k)) }\right]\\
  &= \sum_{k\in [K]}\Gamma(k)\mathbb{E}\left[\sum_{s,t}\sqrt{\mathbb{P}_{s,t}(k\in A_{s,t})}\sqrt{\mathbb{P}_{s,t}(k\in A_{s,t}) I_{s,t}(\mu_{*,k}; k, Y_{s,t}(k)) }\right]\\
  &\leq  \sum_{k\in [K]}\Gamma(k)\mathbb{E}\left[\sqrt{\sum_{s,t}\mathbb{P}_{s,t}(k\in A_{s,t})}\sqrt{\sum_{s,t}\mathbb{P}_{s,t}(k\in A_{s,t}) I_{s,t}(\mu_{*,k}; k, Y_{s,t}(k)) }\right]\\
  &\leq  \sum_{k\in [K]}\Gamma(k)\sqrt{\mathbb{E}\sum_{s,t}\mathbb{P}_{s,t}(k\in A_{s,t})}\sqrt{\mathbb{E}\sum_{s,t}\mathbb{P}_{s,t}(k\in A_{s,t}) I_{s,t}(\mu_{*,k}; k, Y_{s,t}(k)) }\\
  &=  \sum_{k\in [K]}\Gamma(k)\sqrt{\mathbb{E}N_{1:m}(k)}\sqrt{I(\mu_{*,k}; H_{1:m})}\\
  &\leq  \sqrt{K\sum_{k\in [K]}\mathbb{E}N_{1:m}(k)}\sqrt{\tfrac{1}{K}\sum_{k\in [K]}\Gamma^2(k)I(\mu_{*,k}; H_{1:m})}\\
  & =  \sqrt{m n K L}\sqrt{\tfrac{1}{K}\sum_{k\in [K]}\Gamma^2(k)I(\mu_{*,k}; H_{1:m})}
\end{align*}

For the third term we have 
\begin{align*}
&\sum_{k\in [K]}\sqrt{2\delta\tfrac{1}{K}\hat{\sigma}_{s}^2(k)} \mathbb{E}[\sum_{t} \mathbb{P}_{s,t}(k\in A_{s,t})]\\ 
&\leq \sum_{k\in [K]} \sqrt{2\delta\tfrac{1}{K}  \hat{\sigma}_{s}^2(k)}\mathbb{E}[N_s(k)]\\
&\leq \sqrt{2\delta\tfrac{1}{K} \sum_{k\in [K]} \hat{\sigma}_{s}^2(k)} \sqrt{\sum_{k\in [K]} (\mathbb{E}[N_s(k)])^2}\\
&\leq \sqrt{2\delta\tfrac{1}{K} \sum_{k\in [K]} \hat{\sigma}_{s}^2(k)} \sum_{k\in [K]} \mathbb{E}[N_s(k)] \leq n\sqrt{2\delta\tfrac{1}{K} \sum_{k\in [K]} \hat{\sigma}_{s}^2(k)}
\end{align*}
\endgroup
This provides us with the bound stated in the lemma.

Finally, we have $\mathbb{E}[\sqrt{\sum_{k\in K} \theta^2_{s,*}(k)}] \leq \sqrt{\sum_{k\in K}\left(\mu^2_{q}(k) + \sigma^2_{0,k}+\sigma^2_{q,k}\right)}$ 

\end{proof}

\paragraph{Bounding Mutual Information:} The derivation of the mutual information can be done similar to the linear bandits while using the diagonal nature of the covariance matrices. We present a different argument here.  
\begin{lemma}\label{lemm:mutual semibandit}
For any $H_{1:s, t}$-adapted action-sequence and any $s \in [m]$ and $k\in [K]$, we have
\begin{gather*}
  I(\theta_{s,*}(k); H_{s}\mid\mu_{*,k}, H_{1:s-1})
  \leq \tfrac{1}{2}\log\left(1+ n \tfrac{\sigma^2_{0,k}}{\sigma^2}\right)\,,\\
  I(\mu_{*, k}; H_{1:m})
  \leq \tfrac{1}{2}\log\left(1+ m \tfrac{\sigma^2_{q,k}}{\sigma^2_{0,k} + \sigma^2/n} \right)\,.
\end{gather*} 
\end{lemma}
\begin{proof}
We have the following form for the conditional mutual information $\theta_{s,*}(k)$ with the events $H_s(k)$ conditioned on the meta-parameter $\mu_{*,k}$, and the history of arm $k$ pulls upto stage $s$ (for each $s$)  $H_{1:s-1}$, as a function of $H_{s}$, given as  
\begin{align*}
    I(\theta_{s,*}(k); H_{s}\mid\mu_{*,k}, H_{1:s-1}) = \tfrac{1}{2}\mathbb{E}\log\left(1+N_s(k) \tfrac{\sigma^2_{0,k}}{\sigma^2}\right) \leq \tfrac{1}{2}\log\left(1+ n \tfrac{\sigma^2_{0,k}}{\sigma^2}\right).
\end{align*}
Another way to see this is, in each stage if $\mu_{*,k}$ was known then the variance of the estimate of $\theta_{s,*}(k)$, or equivalently of $(\theta_{s,*}(k) - \mu_{*,k})$, after $N_s(k)$ samples and with an initial variance $\sigma^2_{0,k}$ will be $\frac{1}{\sigma_0^{-2}(k)+ N_s(k)\sigma^{-2}(k)}$. We note that only when $N_s(k) \geq 1$ the mutual information is non-zero. Thus we have the multiplication with $\mathbb{P}(N_s(k) \geq 1)$. The mutual information is then derived easily.   

Similarly, the  mutual information of $\theta_{s,*}(k)$ and the entire history of arm $k$ pulls, i.e.  $H_{1:m}(k)$, is stated as follows. 
\begin{align*}
    I(\mu_{*,k}; H_{1:m}) =  \tfrac{1}{2}\mathbb{E}\log\left(1+\sum_{s=1}^{m} \frac{\sigma^2_{q,k}}{\sigma^2_{0,k} + \sigma^2/N_s(k)} \right) \leq  \tfrac{1}{2}\log\left(1+ m \frac{\sigma^2_{q,k}}{\sigma^2_{0,k} + \sigma^2/n} \right) .
\end{align*}
We claim (proven shortly) that at the end of task $m$ the variance of estimate of $\mu_{*,k}$ is $(\hat{\sigma}^{-2}_{q}(k) + \sum_{s'=1}^{m}(\sigma^2_{0,k} + \sigma^2/N_{s'}(k))^{-1})^{-1}$. This gives the first equality. The final inequality holds by noting that minimizing the terms $\sigma^2/N_s(k)$ with $N_s(k)=n$, for all $s$, (as any arm can be pulled at most $n$ times in any task) maximizes the mutual information.

We now derive the variance of $\mu_{*,k}$. Let the distribution of $\mu_{*,k}$ at the beginning of stage $s$ is  $\mathcal{N}(\hat{\mu}_{s}(k), \hat{\sigma}^2_{s}(k))$. From the $N_s(k)$ samples of arm $k$, we know $\theta_{s,*}(k) \sim \mathcal{N}(\hat{\theta}_s(k), \sigma^2/N_s(k))$ where $\hat{\theta}_s(k)$ is the empirical mean of arm $k$ in task $s$. Further, $\theta_{s,*}(k) - \mu_{*,k}\sim \mathcal{N}(0, \sigma^2_{0,k})$ by our reward model. Thus, we have 
from the two above relation $$\mu_{*,k} \sim \mathcal{N}\left(\frac{\sigma^2/N_s(k)}{\sigma^2_{0,k} + \sigma^2/N_s(k)}\hat{\theta}_{s}(k), \sigma^2_{0,k} + \sigma^2/N_s(k)\right).$$

However, we also know independently that $\mu_{*,k} \sim \mathcal{N}(\hat{\mu}_{s}(k), \hat{\sigma}^2_{s}(k))$. Therefore, a similar combination gives us $\mu_{*,k} \sim \mathcal{N}(\hat{\mu}_{s+1}(k), \hat{\sigma}^2_{s+1}(k))$ where 
\begin{gather*}
    \hat{\mu}_{s+1}(k) = \hat{\sigma}^{-2}_{s+1}(k)\left(\hat{\mu}_s(k) \hat{\sigma}^2_s(k) + \hat{\theta}_s(k) \sigma^2/N_s(k)\right)\\
    \hat{\sigma}^{-2}_{s+1}(k) = \hat{\sigma}^{-2}_{s}(k) + (\sigma^2_{0,k} + \sigma^2/N_s(k))^{-1}\\
    \hat{\sigma}^{-2}_{s+1}(k) = \sigma^{-2}_{q, k} + \sum_{s'=1}^{s}(\sigma^2_{0,k} + \sigma^2/N_{s'}(k))^{-1}.
\end{gather*}
The last equality follows from induction with the base case $\hat{\sigma}_{0}^2(k) = \sigma^2_{q,k}$.
\end{proof}

\paragraph{Bounding $\Gamma_{s}(k)$:} We finally provide the bound on the $\Gamma_{s}(k)$ and $\Gamma(k)$ terms used in the regret decomposition \cref{lemm:decomp semibandit}.
\begin{lemma}\label{lemm:gamma semibandit}
For all $s\in [m]$, and $(\Gamma_s(k))_{s\in [m]}$ and $\Gamma(k)$ as defined in \cref{lemm:decomp semibandit} admit the following bounds, for any $\delta \in (0,1]$, almost surely 
\begin{gather*}
    \Gamma_{s}(k) \leq  4\sqrt{\frac{\sigma^2_{0,k}\left(1 + \tfrac{(1+\sigma^2/\sigma^2_{0,k})\sigma^2_{q,k}}{(\sigma^2_{0,k} +\sigma^2) + s \sigma^2_{q,k}}\right) }{\tfrac{1}{2}\log\left(1+ \tfrac{\sigma^2_{0,k}}{\sigma^2}\left(1 + \tfrac{(1+\sigma^2/\sigma^2_{0,k})\sigma^2_{q,k}}{(\sigma^2_{0,k} +\sigma^2) + s \sigma^2_{q,k}}\right)\right)} \log(\tfrac{4K}{\delta})},\\
    \Gamma(k) \leq 4\sqrt{\tfrac{\sigma^2_{0,k} +\sigma^2_{q,k}}{\log\left(1+ \tfrac{\sigma^2_{0,k} +\sigma^2_{q,k}}{\sigma^2}\right)} \log(\tfrac{4K}{\delta})}.
\end{gather*}
\end{lemma}
\begin{proof}
At the beginning of task $s$ we know that $\theta_{s,*}(k) \sim \mathcal{N}(\hat{\mu}_s(k), \sigma^2_{0,k} + \hat{\sigma}^2_s(k))$. And as the variance of $\theta_{s,*}(k)$ decreases during task $s$ with new samples from arm $k$, we have the variance 
$$\hat{\sigma}^2_{s,t}(k) \leq \sigma^2_{0,k} + \hat{\sigma}^2_s(k) \leq \sigma^2_{0,k} + \tfrac{(\sigma^2_{0,k}+\sigma^2)\sigma^2_{q,k}}{(\sigma^2_{0,k} +\sigma^2) + s \sigma^2_{q,k}}.$$  
The inequality holds by taking $N_{s'}(k) = 1$ in the expression of $\hat{\sigma}^2_s(k)$ for all tasks as arm $k$ has been played using forced exploration. 

Therefore, we can bound $\Gamma_{s}(k)$, for any $s$,  as 
$$
\Gamma_{s}(k) \leq  4\sqrt{\frac{\sigma^2_{0,k}\left(1 + \tfrac{(1+\sigma^2/\sigma^2_{0,k})\sigma^2_{q,k}}{(\sigma^2_{0,k} +\sigma^2) + s \sigma^2_{q,k}}\right) }{\tfrac{1}{2}\log\left(1+ \tfrac{\sigma^2_{0,k}}{\sigma^2}\left(1 + \tfrac{(1+\sigma^2/\sigma^2_{0,k})\sigma^2_{q,k}}{(\sigma^2_{0,k} +\sigma^2) + s \sigma^2_{q,k}}\right)\right)} \log(\tfrac{4K}{\delta})} 
$$
This implies $\Gamma(k) \leq 4\sqrt{\frac{\sigma^2_{0,k} +\sigma^2_{q,k}}{\log\left(1+ \tfrac{\sigma^2_{0,k} +\sigma^2_{q,k}}{\sigma^2}\right)} \log(\tfrac{4K}{\delta})}$ by setting $s=0$.
\end{proof}

\paragraph{Deriving Final Regret Bound:} We proceed with our final regret bound as 
\semibanditinfo*
\begin{proof}
We now use the regret decomposition in \cref{lemm:decomp semibandit}, the bounds on terms $\Gamma_{s}(k)$ and $\Gamma(k)$ in \cref{lemm:gamma semibandit}, and the mutual information in \cref{lemm:mutual semibandit} bounds derived earlier to obtain our final regret bound for the semi-bandits. The regret bound follows from the following chain of inequalities. 
\begingroup
\allowdisplaybreaks
\begin{align*}
    &R(m, n) \\
    &\leq \sqrt{m n K L}  \sqrt{\tfrac{1}{K}\sum_{k\in [K]}\Gamma^2(k)I(\mu_{*,k}; H_{1:m})}
    + 2 \sqrt{m K \sum_{k\in K}\left(\mu^2_{q}(k) + \sigma^2_{0,k}+\sigma^2_{q,k}\right)} \\
    & + \sum_{s}\left(\sqrt{n K L} \sqrt{\tfrac{1}{K}\sum_{k\in [K]}\Gamma^2_{s}(k)I(\theta_{s,*}(k); H_s\mid \mu_{*,k}, H_{1:s-1})}  +  n\sqrt{2\delta\tfrac{1}{K}\sum_{k\in [K]} \hat{\sigma}_{s}^2(k)}\right)\\
    & \leq 4\sqrt{ \tfrac{1}{K} \sum_{k\in [K]} \frac{\sigma^2_{0,k}+\sigma^2_{q,k}}{\log\left(1+\tfrac{\sigma^2_{0,k}+\sigma^2_{q,k})}{\sigma^2} \right)}\log\left(1+  m \frac{\sigma^2_{q,k}}{\sigma^2_{0,k} + \sigma^2/n}\right)\log(\tfrac{4K}{\delta})}\sqrt{mn K L} \\
    & + 2 m K^{3/2}\sqrt{ \sum_{k\in K}\left(\mu^2_{q}(k) + \sigma^2_{0,k}+\sigma^2_{q,k}\right)}\\
    & + \sum_{s}\left(\sqrt{n K L} \sqrt{\tfrac{1}{K}\sum_{k\in [K]: \sigma^2_{0,k} = 0}\Gamma^2_{s}(k)I(\theta_{s,*}(k); H_s\mid \mu_{*,k}, H_{1:s-1})}  +  n\sqrt{2\delta\tfrac{1}{K}\sum_{k\in [K]: \sigma^2_{0,k} = 0} \hat{\sigma}_{s}^2(k)}\right)\\
    & + \sum_{s}\left(\sqrt{n K L} \sqrt{\tfrac{1}{K}\sum_{k\in [K]: \sigma^2_{0,k} > 0}\Gamma^2_{s}(k)I(\theta_{s,*}(k); H_s\mid \mu_{*,k}, H_{1:s-1})}  +  n\sqrt{2\delta\tfrac{1}{K}\sum_{k\in [K]: \sigma^2_{0,k} > 0} \hat{\sigma}_{s}^2(k)}\right)\\
    &\leq 4\sqrt{ \tfrac{1}{K} \sum_{k\in [K]} \frac{\sigma^2_{0,k}+\sigma^2_{q,k}}{\log\left(1+\tfrac{\sigma^2_{0,k}+\sigma^2_{q,k})}{\sigma^2} \right)}\log\left(1+  m \frac{\sigma^2_{q,k}}{\sigma^2_{0,k} + \sigma^2/n}\right)\log(\tfrac{4K}{\delta})}\sqrt{mn K L} \\
    & + 2m K^{3/2} \sqrt{ \sum_{k\in K}\left(\mu^2_{q}(k) + \sigma^2_{0,k}+\sigma^2_{q,k}\right)} +   n\sqrt{2 m \delta\tfrac{1}{K}\sum_{k\in [K]: \sigma^2_{0,k}= 0}\sum_{s=1}^{m} \tfrac{\sigma^2 \sigma^2_{q,k}}{ \sigma^2  + s \sigma^2_{q,k}}}\\
    %-------------
    & + \left(m+ {\color{blue} \sum_{s=1}^{m} \max\limits_{k\in [K]: \sigma^2_{0,k}>0}  \tfrac{1}{2\sigma^2_{0,k}}\tfrac{(\sigma^2_{0,k}+\sigma^2)\sigma^2_{q,k}}{(\sigma^2_{0,k} +\sigma^2) + s \sigma^2_{q,k}} }\right)\\
    &\times \left(4\sqrt{ \tfrac{1}{K} \sum_{k\in [K]: \sigma^2_{0,k}>0} \frac{\sigma^2_{0,k}}{\log\left(1+\tfrac{\sigma^2_{0,k}}{\sigma^2} \right)}\log\left(1+   n\frac{\sigma^2_{0,k}}{\sigma^2}\right)\log(\tfrac{4K}{\delta})}\sqrt{n K L} 
    +  n\sqrt{2\delta\tfrac{1}{K}\sum_{k\in [K]: \sigma^2_{0,k}>0} \sigma^2_{0,k}} \right)\\
    &\leq 4\sqrt{ \tfrac{1}{K} \sum_{k\in [K]} \frac{\sigma^2_{0,k}+\sigma^2_{q,k}}{\log\left(1+\tfrac{\sigma^2_{0,k}+\sigma^2_{q,k})}{\sigma^2} \right)}\log\left(1+  m \frac{\sigma^2_{q,k}}{\sigma^2_{0,k} + \sigma^2/n}\right)\log(\tfrac{4K}{\delta})}\sqrt{mn K L} \\
    & + 2 m K^{3/2} \sqrt{\sum_{k\in K}\left(\mu^2_{q}(k) + \sigma^2_{0,k}+\sigma^2_{q,k}\right)} +   n \sqrt{ 2\delta \sigma^2 m\tfrac{1}{K}\sum_{k\in [K]: \sigma^2_{0,k}= 0} \log(1 + m \tfrac{\sigma^2_{q,k}}{\sigma^2})}\\
    %-------------
    & + \left(m+ {\color{blue} \left(1+ \max\limits_{k\in [K]: \sigma^2_{0,k}>0} \tfrac{\sigma^2}{\sigma^2_{0,k}}\right) \log(m) }\right)\\
    &\times \left(4\sqrt{ \tfrac{1}{K} \sum_{k\in [K]: \sigma^2_{0,k}>0} \frac{\sigma^2_{0,k}}{\log\left(1+\tfrac{\sigma^2_{0,k}}{\sigma^2} \right)}\log\left(1+   n\frac{\sigma^2_{0,k}}{\sigma^2}\right)\log(\tfrac{4K}{\delta})}\sqrt{n K L} 
    +  n\sqrt{2\delta\tfrac{1}{K}\sum_{k\in [K]: \sigma^2_{0,k}>0} \sigma^2_{0,k}} \right)
\end{align*}
\endgroup
The derivation follows through steps similar to the corresponding derivations for the linear bandits. In the second inequality we differentiate the arms which has $\sigma^2_{0,k} = 0$ against the rest. Any arm $k$ with $\sigma^2_{0,k}=0$ has no mutual information once $\mu_{*,k}$ is known, i.e. $I(\theta_{s,*}(k); H_s\mid \mu_{*,k}, H_{1:s-1})=0$ for all such $k$. 
\end{proof}

\clearpage

\section{Supplementary Experiments}
\label{sec:supplementary experiments}

We conduct two additional experiments. In \cref{sec:synthetic experiments}, we extend synthetic experiments from \cref{sec:experiments}. In \cref{sec:classification experiments}, we experiment with two real-world classification problems: MNIST \cite{lecun10mnist} and Omniglot \cite{lake15human}.

\subsection{Synthetic Experiments}
\label{sec:synthetic experiments}

\begin{figure*}[t]
  \centering
  \includegraphics[width=5.4in]{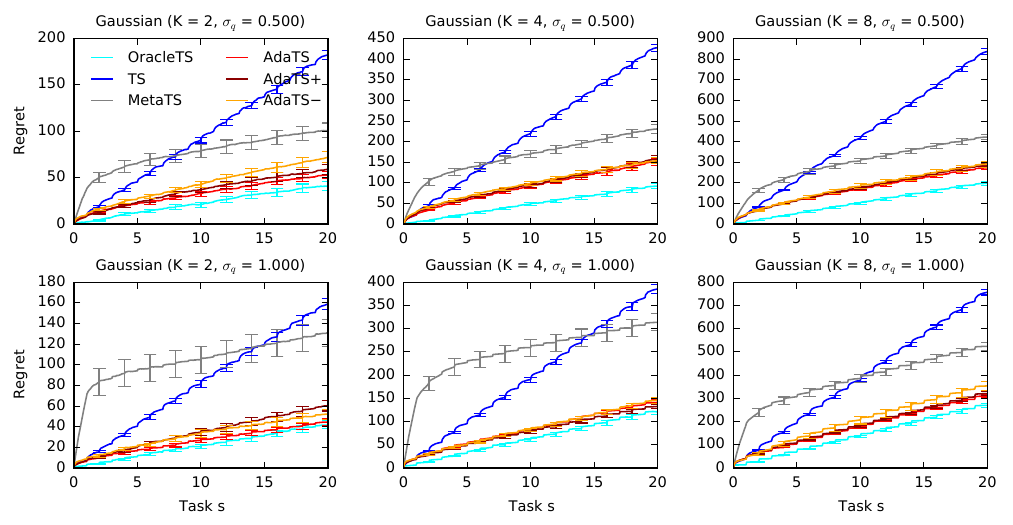}
  \vspace{-0.1in}
  \caption{\adats in a $K$-armed Gaussian bandit. We vary both $K$ and meta-prior width $\sigma_q$.}
  \label{fig:gaussian}
\end{figure*}

\begin{figure*}[t]
  \centering
  \includegraphics[width=5.4in]{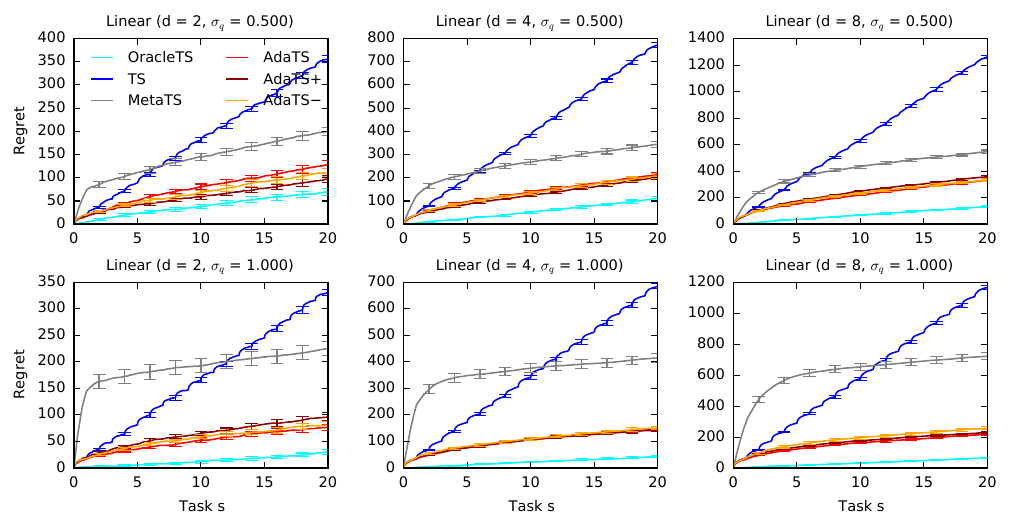}
  \vspace{-0.1in}
  \caption{\adats in a $d$-dimensional linear bandit with $K = 5 d$ arms. We vary both $d$ and meta-prior width $\sigma_q$.}
  \label{fig:linear}
\end{figure*}

This section extends experiments in \cref{sec:experiments} in three aspects. First, we show the Gaussian bandit with $K \in \set{2, 4, 8}$ arms. Second, we show the linear bandit with $d \in \set{2, 4, 8}$ dimensions. Third, we implement \adats with a misspecified meta-prior.

Our results are reported in \cref{fig:gaussian,fig:linear}. The setup of this experiment is the same as in \cref{fig:synthetic}, and it confirms all earlier findings. We also experiment with two variants of misspecified \adats. In \adatsplus, the meta-prior width is widened to $3 \sigma_q$. This represents an overoptimistic agent. In \adatsminus, the meta-prior width is reduced to $\sigma_q / 3$. This represents a conservative agent. We observe that this misspecification has no major impact on the regret of \adats, which attests to its robustness.

\subsection{Online One-Versus-All Classification Experiments}
\label{sec:classification experiments}

\begin{figure}[t]
    \centering
    \centering
	\begin{subfigure}[t]{0.45\textwidth}
    \includegraphics[width=\textwidth]{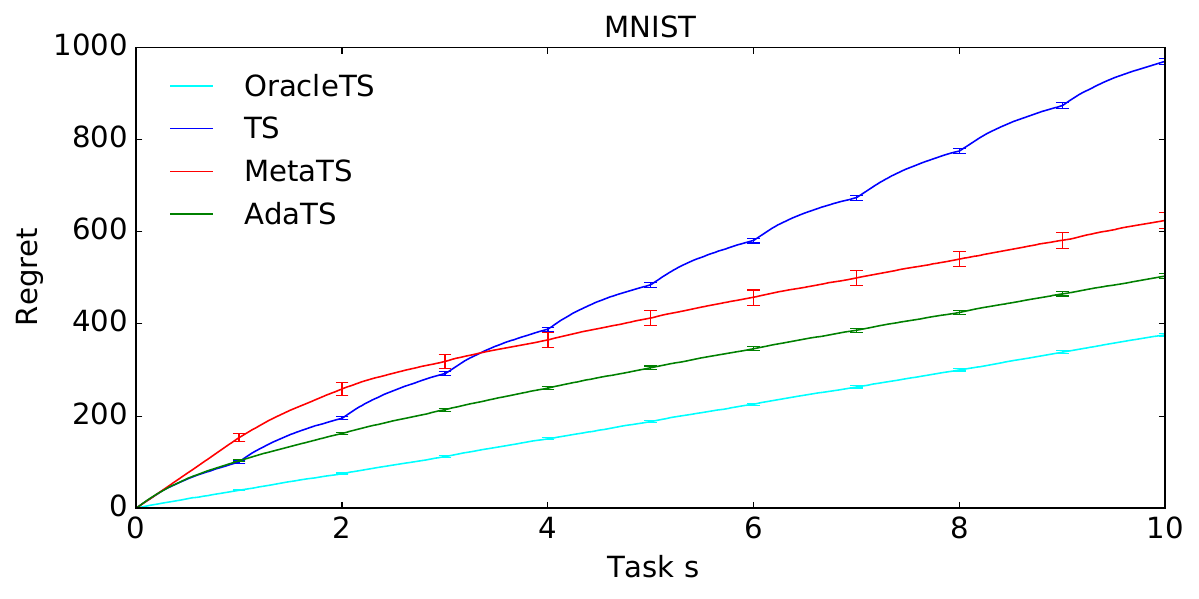}
    \includegraphics[width=0.985\textwidth]{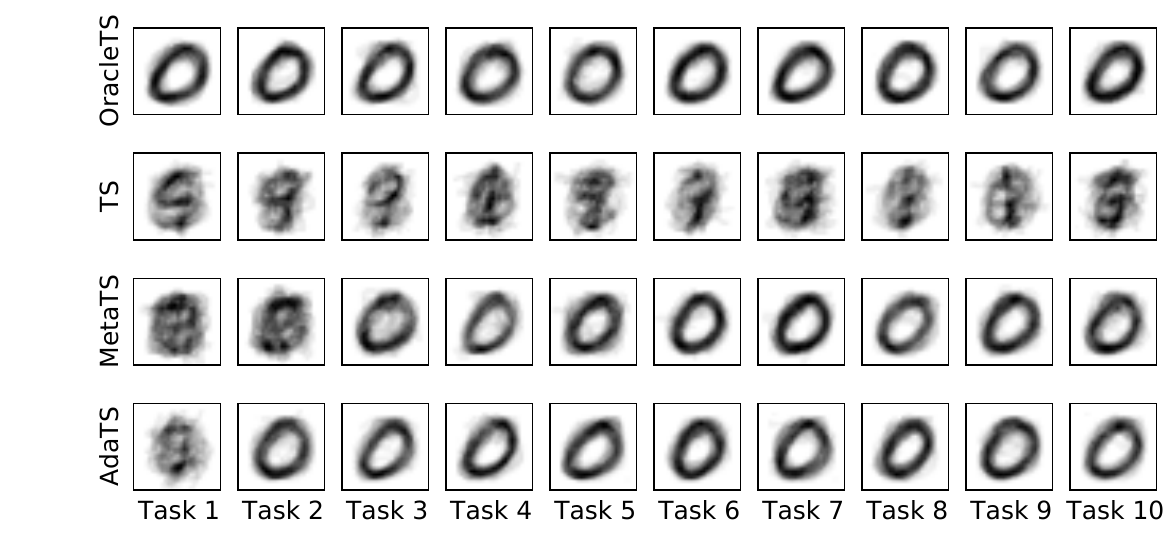}
    \end{subfigure}
    \begin{subfigure}[t]{0.45\textwidth}
    \includegraphics[width=\textwidth]{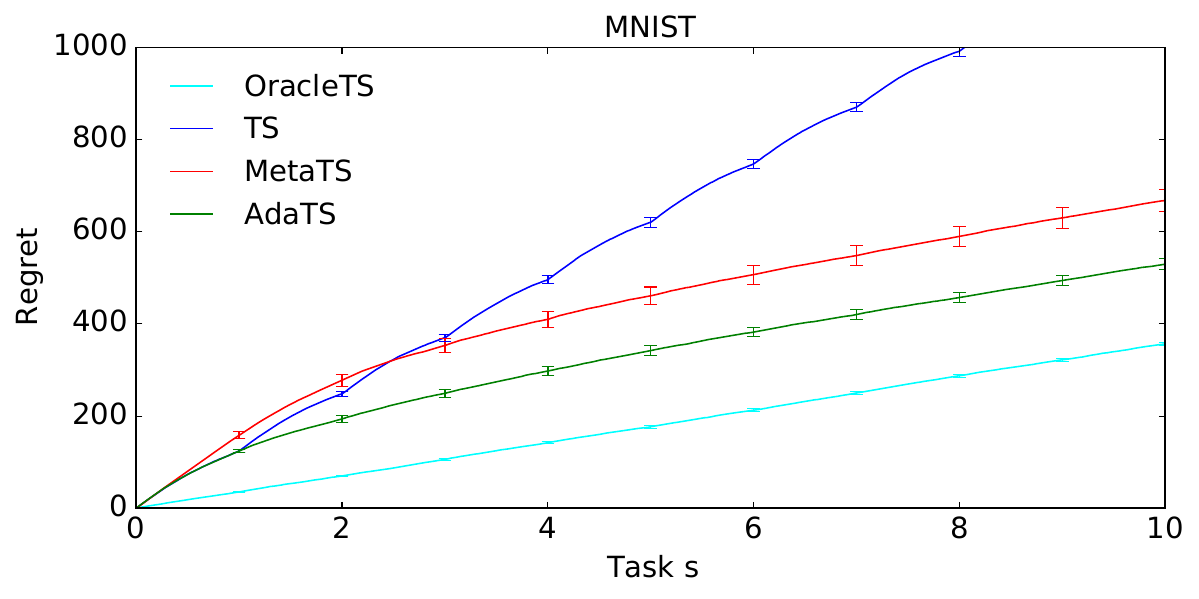}
    \includegraphics[width=0.985\textwidth]{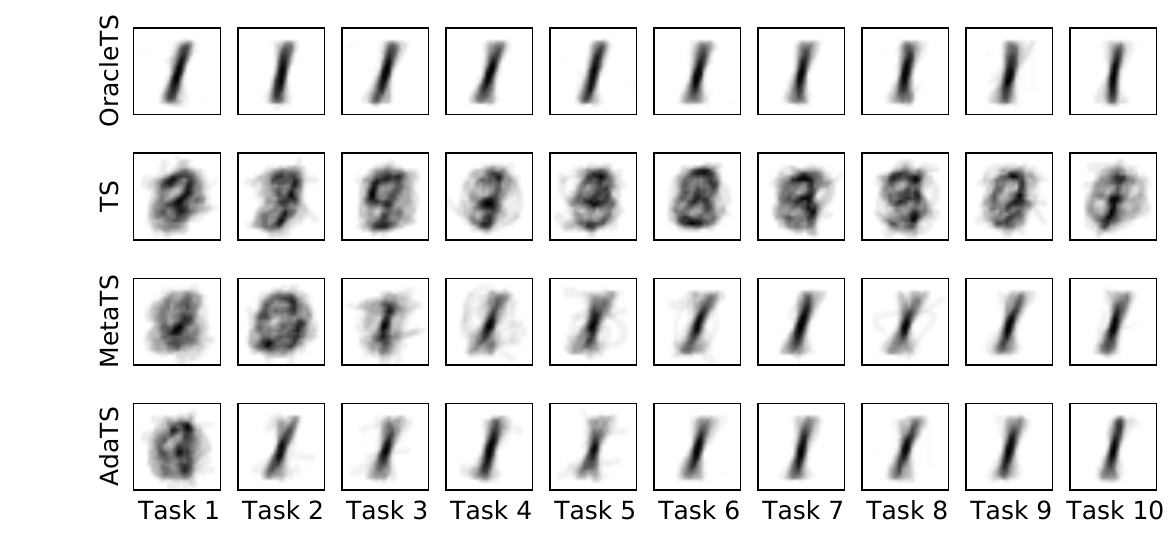}
    \end{subfigure}
    \caption{\adats in two meta-learning problems of digit classification from MNIST. On the top, we plot the cumulative regret as it accumulates over rounds within each task. Below we visualize the average digit, corresponding to the pulled arms in round $1$ of the tasks.}
    \label{fig:mnist}
\end{figure}

\begin{figure}[t]
    \centering
	\begin{subfigure}[t]{0.45\textwidth}
    \includegraphics[width=\textwidth]{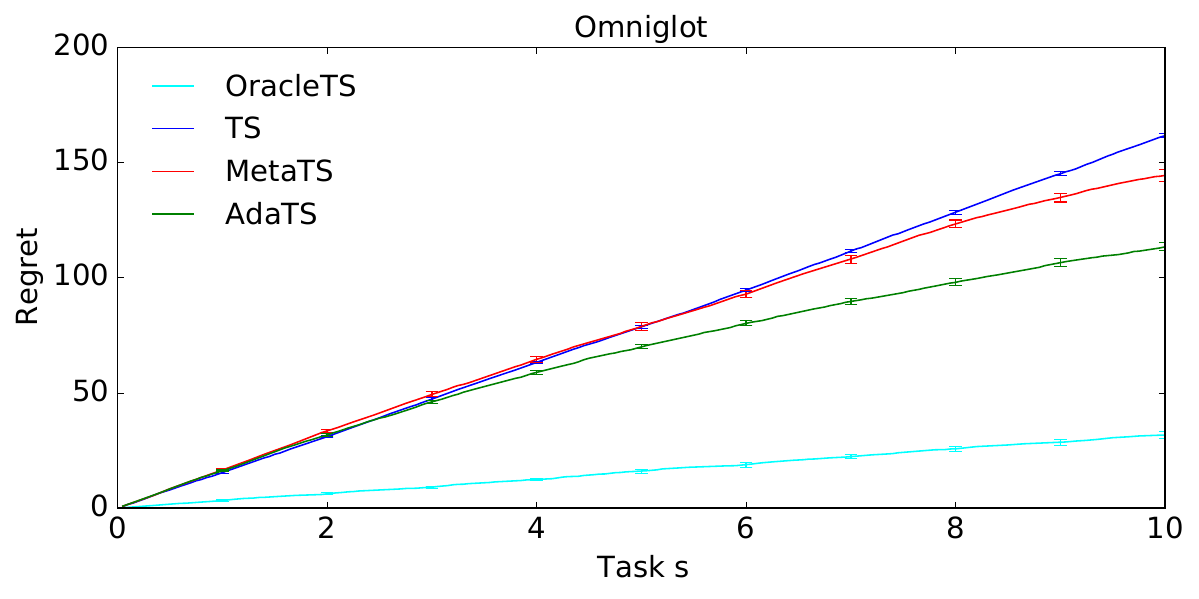}
    \includegraphics[width=0.985\textwidth]{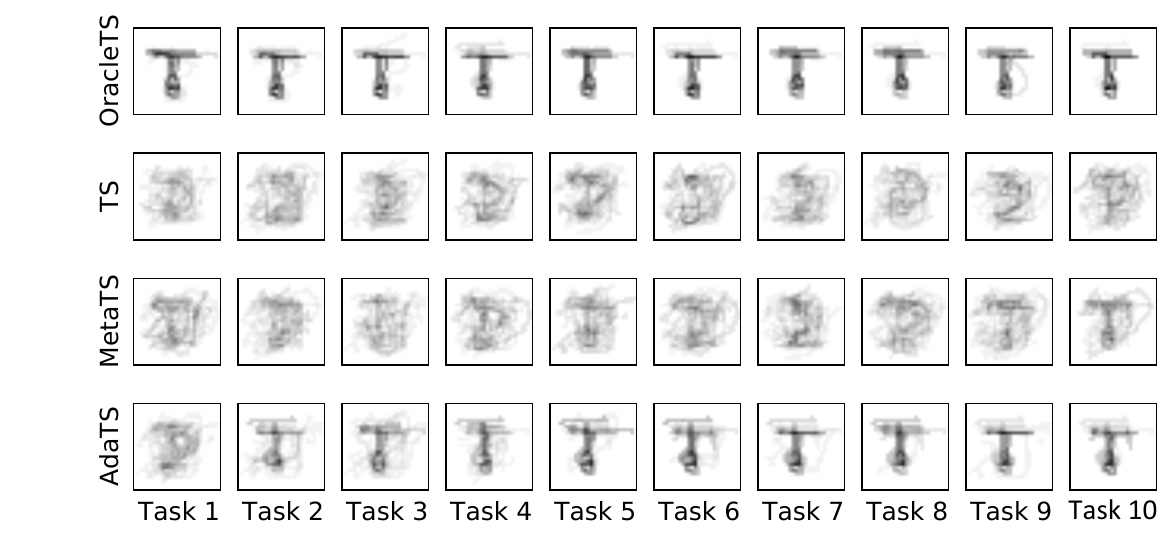}
    \end{subfigure}
    \begin{subfigure}[t]{0.45\textwidth}
    \includegraphics[width=\textwidth]{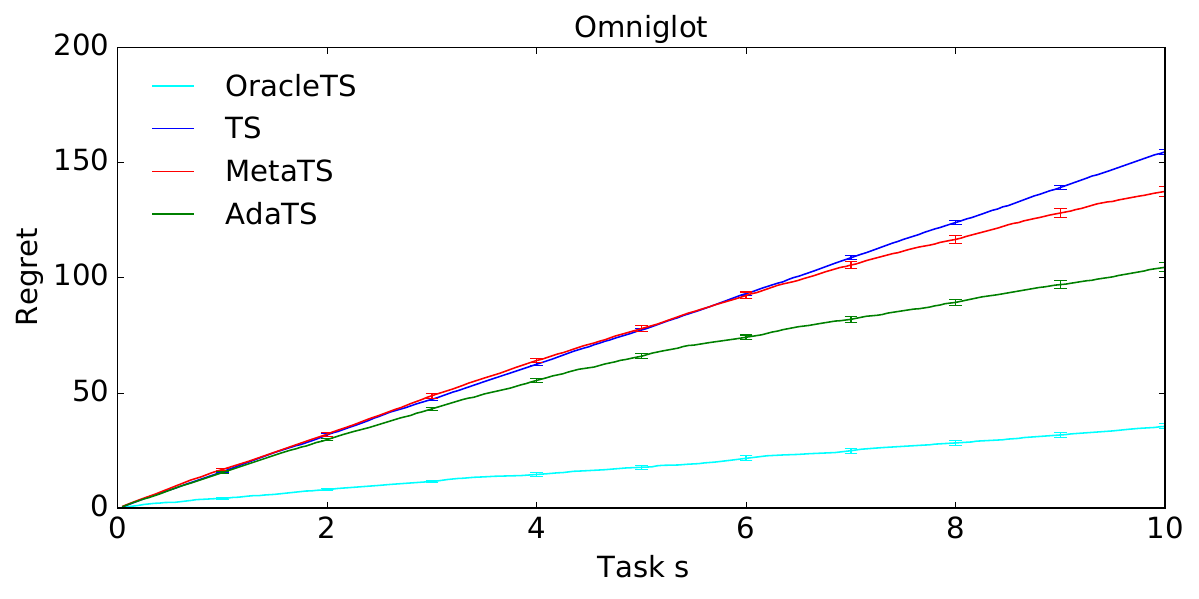}
    \includegraphics[width=0.985\textwidth]{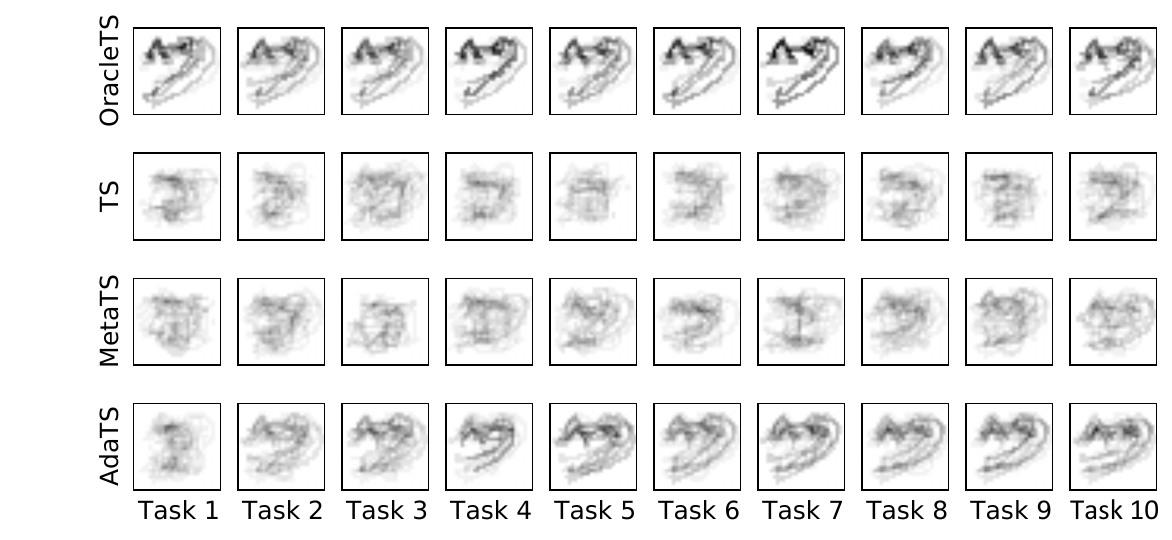}
    \end{subfigure}
    \caption{\adats in two meta-learning problems of character classification from Omniglot. On the top, we plot the cumulative regret as it accumulates over rounds within each task. Below we visualize the average character, corresponding to the pulled arms in round $1$ of the tasks.}
    \label{fig:omni}
\end{figure}

We consider online classification on two real-world datasets, which are commonly used in meta-learning. The problem is cast as a multi-task linear bandit with Bernoulli rewards. Specifically, we have a sequence of image classification tasks where one class is selected randomly to be positive. In each task, at every round, $K$ random images are selected as the arms and the goal is to pull the arm corresponding to an image from the positive class. The reward of an image from the positive class is $\mathrm{Ber}(0.9)$ and for all other classes is $\mathrm{Ber}(0.1)$. Dataset-specific settings are as follows:

\begin{enumerate}
\item \textbf{MNIST} \cite{lecun10mnist}: The dataset contains $60\,000$ images of handwritten digits, which we split into equal-size training and test sets. We down-sample each image to $d = 49$ features and then use these as arm features. The training set is used to estimate $\mu_0$ and $\Sigma_0$ for each digit. The bandit algorithms are evaluated on the test set. In each simulation, we have $m = 10$ tasks with horizon $n = 200$ and $K = 30$ arms.

\item \textbf{Omniglot} \cite{lake15human}: The dataset contains $1\,623$ different handwritten characters from $50$ different alphabets. This is an extremely challenging dataset because we have only $20$ human-drawn images per character. Therefore, it is important to adapt quickly. We train a $4$-layer CNN to extract $d = 64$ features using characters from $30$ alphabets. The remaining $20$ alphabets are split into equal-size training and test sets, with $10$ images per character in each. The training set is used to estimate $\mu_0$ and $\Sigma_0$ for each character. The bandit algorithms are evaluated on the test set. In each simulation, we have $m = 10$ tasks with horizon $n = 10$ and $K = 10$ arms. We guarantee that at least one character from the positive class is among the $K$ arms.
\end{enumerate}

In all problems, the meta-prior is $\cN(\mathbf{0}, I_d)$ and the reward noise is $\sigma = 0.1$. We compare \adats with the same three baselines as in \cref{sec:experiments}, repeat all experiments $20$ times, and report the results in \cref{fig:mnist,fig:omni}. Along with the cumulative regret, we also visualize the average digit / character corresponding to the pulled arms in round $1$ of each task. We observe that \adats learns a very good meta-parameter $\mu_*$ almost instantly, since its average digit / character in task $2$ already resembles the unknown highly-rewarding digit / character. This happens even in Omniglot, where the horizon of each task is only $n = 10$ rounds. Note that the meta-prior was not selected in any dataset-specific way. The fact that \adats still works well attests to the robustness of our method.

\end{document}